\documentclass{article}

\PassOptionsToPackage{numbers,sort&compress}{natbib}
\usepackage[preprint]{neurips_2026}

\usepackage[utf8]{inputenc}
\usepackage[T1]{fontenc}
\usepackage{hyperref}
\usepackage{url}
\usepackage{booktabs}
\usepackage{amsmath,amssymb,amsfonts,mathtools}
\usepackage{amsthm}
\usepackage{microtype}
\usepackage{xcolor}
\usepackage{graphicx}
\usepackage{enumitem}
\usepackage{nicefrac}
\usepackage{multirow}
\usepackage{wrapfig}

\usepackage[most]{tcolorbox}
\usepackage[table]{xcolor}
\definecolor{ourscolor}{RGB}{230,243,255}  % light blue
\definecolor{mydarkblue}{RGB}{0,80,160}

  \newtcolorbox{summarybox}{
    colback=ourscolor,
    colframe=mydarkblue,
    boxrule=1pt,
    arc=3pt,
    left=6pt,
    right=6pt,
    top=6pt,
    bottom=6pt,
  }

\newtheorem{proposition}{Proposition}

\newcommand{\R}{\mathbb{R}}

\newcommand{\Mem}{\mathsf{Mem}}

\newcommand{\ValAcc}{\mathrm{ValAcc}}
\newcommand{\TrainAcc}{\mathrm{TrainAcc}}
\newcommand{\Loss}{\mathcal{L}}

\newcommand{\sigw}{\sigma_w}
\newcommand{\taubest}{\tau_{\mathrm{best}}}
\newcommand{\tauinterp}{\tau_{\mathrm{interp}}}

\title{Do Deep Networks Forget Initialization?\\
A Forgetting-Time View of Practical Inductive Bias}

\author{Mohua Das\thanks{Equal contribution, correspondence at \texttt{mohuadas@mit.edu} and \texttt{pierb@mit.edu}.}\\
  MIT\\
  \And
  Pierfrancesco Beneventano$^*$\\
  MIT\\
  \And
  Shibshankar Dey\\
  Northwestern University\\
  \And
  Gareth H. McKinley \\
  MIT\\
  \And
  Tomaso Poggio \\
  MIT\\
  }

\begin{document}

\maketitle

\begin{abstract}
Randomly initialized neural networks induce a prior over functions, but the predictor used in practice is produced only after training.  We ask how much of this initial bias survives the training pipeline.  To make the question measurable, we introduce \emph{initialization memory}: the dependence of the validation-selected predictor on the scale of the random initialization.  We perform controlled CIFAR-10 experiments on ResNets where initialization memory already sharply separates training regimes.  Low-learning-rate SGD can interpolate while still remembering its initialization: on ResNet-9 with batch size $b=128$, test accuracy varies by $26.5$ percentage points across initialization scales despite $\ge99.5\%$ training accuracy.  This is not undertraining: extending the same low-learning-rate regime to $5{,}000$ epochs leaves the spread essentially unchanged. In contrast, Adam-family methods largely erase the dependence. SGD can also be made to forget when larger learning rates are paired with explicit $L_2$ norm control.  We interpret these findings in terms of the time scale of forgetting: gradient-flow-like dynamics can preserve initialization memory, whereas stochastic finite-step effects, explicit norm decay, and adaptive preconditioning erase it on scales governed by the size of explicit or implicit regularization. The practical inductive bias of a trained network is therefore not the architectural prior alone, but the architectural prior after being filtered by the forgetting dynamics of the training pipeline; and the same regularizers that improve generalization are precisely those that erase memory of initialization. 
%pier{I changed it a bit, one should add a final sentence to close up. Please change it in the submission}
\end{abstract}

\section{Introduction}
\label{sec:intro}
\paragraph{The literature on initialization.}
Modern neural networks are too expressive a priori for performance to be explained by capacity alone. Thus, the relevant object is not merely the hypothesis class, but the training pipeline that selects one function from it \citep{zhang2021understanding}.  In practice, this
pipeline includes data preprocessing, architecture, initialization, optimizer,
batching, explicit regularization, and training time.  Understanding performance, therefore, requires understanding not only the prior induced by the architecture, but also how training transforms that prior. This article investigates a precise subquestion:
\vspace{-0.3em}
\begin{center}
    \textit{What is the role of initialization in explaining performance?}
\end{center}
\vspace{-0.3em}
Initialization is a natural place to look for such bias. From a dynamical
systems perspective, the initial condition, in the absence of regularization, determines which region of parameter space the trajectory explores (basin of attraction) and which solution is ultimately selected~\citep{strogatz2015nonlinear,hirsch2013differential}.
In simplified linear and homogeneous networks, initialization is also known to
control the implicit bias of gradient-based training~\citep{min2021explicit,gruber2024role}.
A line of work studies the function prior of random networks: before seeing labels, architectures, and initialization schemes assign much higher probability to some functions than to others, often favoring simple ones \citep{valle2019deep,mingard2025occam, fink2026deep}.
Mingard et al. \citep{mingard2025occam} distinguish the first-order question of why overparameterized DNNs generalize at all from the second-order question of how to further improve the performance of already-generalizing models. Our focus is on one concrete part of that bridge: whether the simplicity bias present at initialization survives training strongly enough to remain visible in the final predictor and in practical performance. Answering this requires studying (i) what happens after optimization begins and (ii) whether practical performance is affected by these geometric biases. The simplicity bias present at initialization may be
preserved, distorted, or erased by the subsequent training dynamics. 
Thus, the practical question is not only whether random networks have a simplicity bias, but whether that bias remains visible in the predictor selected by a modern training pipeline.
The question, given these works, becomes:
\vspace{-0.3em}
\begin{center}
    \textit{When does training remember initialization's bias, when does it forget it, and on what time scale?}\\
    \textit{How does this geometric simplicity bias resolve in practical performance?}
\end{center}
\vspace{-0.3em}
In apparent contrast with the literature above, a different intuition is common in large-scale model training.  There, initialization is often treated less as a source of final inductive bias and more as a mechanism for stabilizing optimization: preventing exploding or vanishing signals, enabling depth, and making training predictable at scale. This perspective underlies variance-preserving initialization schemes, random-matrix and dynamical-mean-field analyses of signal propagation \citep{glorot2010understanding,he2015delving,
poole2016exponential,schoenholz2017deep,pennington2017resurrecting,
hanin2018start,hanin2018gradients,xiao2018dynamical,chen2018dynamical}, and modern parameterization theory such as $\mu$P \citep{yang2021feature,yang2021tensorv,bordelon2023self,
bordelon2024finite,bordelon2025deep,lauditi2025adaptive}.   From this viewpoint, improvements attributed to initialization may come primarily from making training possible or stable, rather than from a persistent preference among functions.

Recent evidence that random seeds and initialization can affect
language-model training---both during fine-tuning
\citep{dodge2020finetuning} and during pretraining
\citep{vanderwal2025polypythias,fehlauer2025convergence,tong2026seedprints,li2026transformersborn}---further sharpens the timeline of the issue.
Related controlled studies of language-model training pipelines and
architecture choices reinforce that large-scale behavior is shaped by
more than the architecture alone~\citep{allenzhu2024physics31,allenzhu2025physics41}.

% Not only have these questions not been answered, to the knowledge of the authors, but recent evidence that different initialization seed may lead to different performance for pretrained LLMs \citep{allenzhu_language_models_things}, prompts the need for understanding the fundamental question of what the effect of initialization is.

\begin{figure}[t]
\centering
\includegraphics[width=0.9\textwidth]{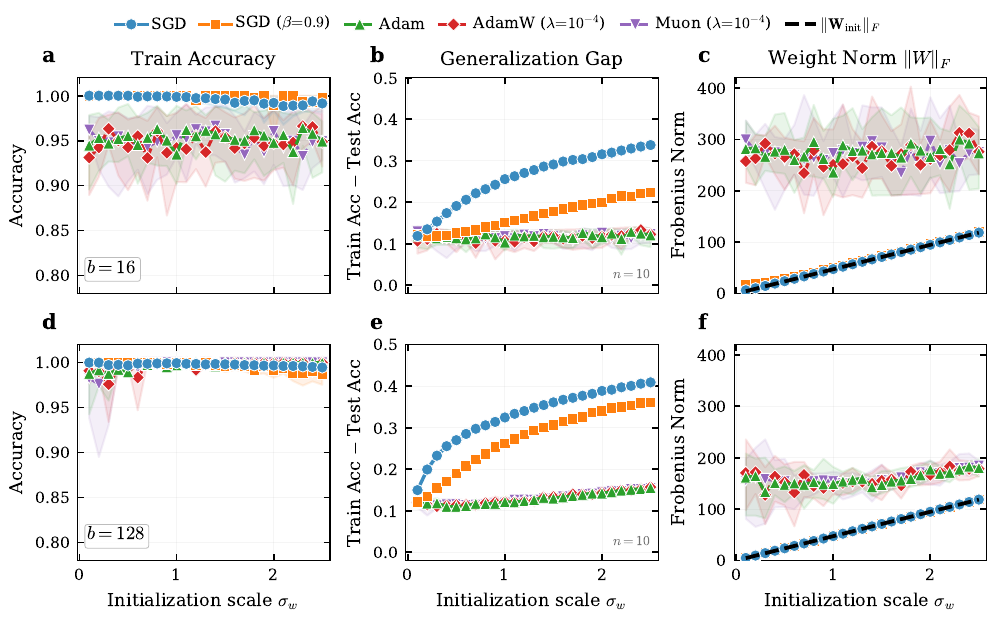}
\caption{
\textbf{SGD remembers initialization; Adam-family methods forget.}
ResNet-9 under a shared low-learning-rate training procedure. Each curve shows the mean over $n=10$ seeds; shaded bands indicate the $10^{\mathrm{th}}-90^{\mathrm{th}}$ percentile range. SGD interpolates, but its generalization gap grows with $\sigw$. Adam, AdamW, and Muon show substantially weaker dependence on $\sigw$. The norm panels show radial memory: SGD retains sensitivity to the initial norm (dashed line), whereas adaptive methods converge toward a common final norm scale. Top row: $b = 16$; bottom row: $b=128.$
}
\label{fig:memory}
\vspace{-1.0em}
\end{figure}

\paragraph{Our contribution.}

As we argue above, there is substantial work showing that random networks and SGD-trained networks are biased toward simple functions, and there is separate work showing that normalization, SGD noise, regularization, and finite-step discretization strongly alter optimization trajectories \citep{li2020reconciling,barrett2021implicit,smith2021origin,beneventano2023trajectories,beneventano2024support} (further related work in Appendix \ref{app:further-related-work}). What is still comparatively under-articulated is when initialization-induced simplicity survives training and when it is forgotten. That is the precise gap our paper occupies.

This article is an empirical study in a deliberately controlled setting: ResNets on CIFAR-10. We perform extensive ablations across initialization scales, optimizers, batch sizes, depths, training horizons, and explicit regularization. We introduce initialization memory as our diagnostic: the extent to which the predictor returned by a training pipeline still depends on the initialization scale~$\sigma_w$.

\begin{enumerate}[leftmargin=1.4em,itemsep=0.15em]
    \item \textbf{A controlled phase diagram of initialization-scale memory.}
    We show that training procedures differ sharply in whether the returned
    predictor still depends on initialization scale. 
    \begin{itemize}
    \item Low-learning-rate SGD can interpolate while retaining large
    initialization-scale memory: for ResNet-9 at \(b=128\), test accuracy varies
    by \(26.5\) percentage points across \(\sigw\) despite \(\ge99.5\%\) training
    accuracy.
    \item In the same diagnostic grid, Adam, AdamW, and Muon instead largely erase this dependence.
    \end{itemize}
    Hyperparameter ablations and depth stress tests show that the failure mode changes across regimes: poor forgetting can appear both as interpolation with poor generalization in shallower networks, or as degraded trainability
    in deeper ones.

    \item \textbf{A separation between interpolation, training horizon, and initialization-scale forgetting.}
    We show that erasing initialization-scale dependence is not implied by
    fitting the labels or by extending the same low-learning-rate dynamics.
    A \(5{,}000\)-epoch low-LR SGD control leaves the initialization-scale
    spread essentially unchanged. SGD can nevertheless be made to forget when
    the recipe supplies larger effective movement (larger implicit regularization) or explicit regularization as weight decay, showing that forgetting is a property of the full training recipe rather than the optimizer name.
    In particular, we show that the training procedure forgets more with a larger learning rate and regularization, or a smaller batch size.

    \item \textbf{A forgetting-timescale mechanism.}
    We organize the results by cumulative optimizer clocks:
    \[
        \mathcal T_{\mathrm{SGD}}
        =
        \frac{1}{b}\sum_{k<K}\eta_k^2,
        \qquad
        \mathcal T_{L_2}
        =
        \lambda\sum_{k<K}\eta_k,
        \qquad
        \mathcal T_{\mathrm{adapt}}
        =
        \sum_{k<K}\eta_k .
    \]
    These time scales track stochastic finite-step effects, explicit norm decay, and adaptive preconditioning. They explain why epoch count and interpolation time are not reliable proxies for erasing initialization-scale dependence,
    and are supported by the repair map, sensitivity-collapse plots, and minimal conservation-law model.
\end{enumerate}

Importantly, we see that the main takeaways align with those of the (implicit) regularization community:
\vspace{-0.3em}
\begin{center}
    \textit{Gradient-flow-like training can remember initialization; stochasticity, adaptivity, and norm control erase it on optimizer-dependent time scales.}
\end{center}
\vspace{-0.3em}

\begin{figure}[t]
\centering
\includegraphics[width=0.8\linewidth]{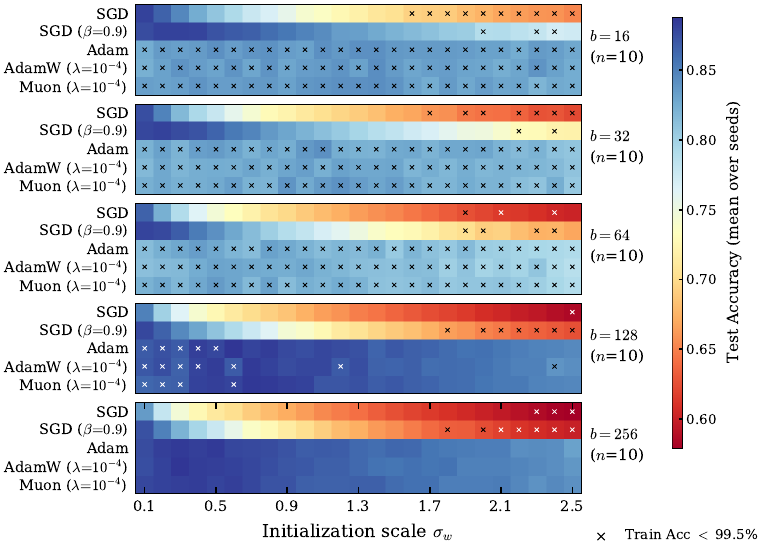}
\vspace{-0.3em}
\caption{
\textbf{Large-batch fixed-epoch regimes forget initialization more slowly.}
ResNet-9 test accuracy at $\taubest$, averaged over $n=10$ seeds. Each panel corresponds to one batch size \(b\in\{16,32,64,128,256\}\);
within each panel, rows are optimizers and columns are initialization scales.  Cells marked $\times$ did not
reach $99.5\%$ mean training accuracy at the best-validation-loss checkpoint $\taubest$.  SGD shows a strong left-to-right
degradation that becomes more pronounced at larger batch sizes; adaptive methods
remain much flatter.
}
\label{fig:heatmap}
\vspace{-1.3em}
\end{figure}
\vspace{-0.3em}
\section{Experimental Design}
\label{sec:design}
\vspace{-0.5em}
We use the initialization scale as a controlled perturbation of the training
pipeline. For each optimizer--batch-size pair, we sweep the global scale
\(\sigw\) of the random kernel initialization while holding all other choices
fixed. If the predictor returned by the same training and checkpoint-selection
rule varies across \(\sigw\), the procedure has retained initialization-scale
memory; if the dependence is small, this component of the initial condition has
been erased.

Formally, let \(\Sigma\) be the finite set of scales in the sweep and let
\(\mathcal R\) denotes the full training procedure: architecture, data split,
optimizer, batch size, schedule, regularization, horizon, and checkpoint rule.
Let \(A_{\mathcal R,K}(\sigma,s)\) be the predictor returned after \(K\) updates
from scale \(\sigma\) and seed \(s\). For metric \(m\), define
\[
    \Mem_m^{\mathrm{ret}}(\mathcal R,K;\Sigma)
    =
    \max_{\sigma\in\Sigma}
    \mathbb E_s[
        m(A_{\mathcal R,K}(\sigma,s))
    ]
    -
    \min_{\sigma\in\Sigma}
    \mathbb E_s[
        m(A_{\mathcal R,K}(\sigma,s))
    ].
\]
In the main experiments, \(m\) is the test accuracy, and the returned predictor is
the best-validation-loss checkpoint.

\begin{summarybox}
\textbf{Diagnostic.}
Initialization-scale memory is the finite-horizon dependence of the returned predictor on one controlled initialization perturbation, \(\sigw\). It is a property of the full training-and-selection procedure.
\end{summarybox}
% \vspace{-1.4em}
\paragraph{Controlled setting.}
We use CIFAR-10 BatchNorm ResNets. The main ResNet-9 grid fixes the
data split, architecture, learning-rate schedule, and training horizon, and varies only initialization scale, optimizer, and batch size. We sweep
$\sigw\in\{0.10,0.20,\ldots,2.50\}$,
optimizers SGD, SGD with momentum, Adam, AdamW, and Muon, and
$b\in\{16,32,64,128,256\}$,
with \(10\) seeds per configuration. All main-grid optimizers use
\(\eta_0=10^{-3}\) with cosine decay for \(300\) epochs.
Ablations over learning rate, $L_2$ regularization, training
length, normalization, augmentation, and depth (ResNet-56, ResNet-110, R9-AvgPool) are reported in
Appendices~\ref{app:ablation}--\ref{app:norm};
full hyperparameters are in Appendix~\ref{app:experiments}.
%\pier{Through the ablations in the appendix, we also ablate through learning rate schedules, xxx. Full hyperparameters are in Appendix~\ref{app:experiments}.}

\begin{figure}[t]
\centering
\includegraphics[width=0.8\linewidth]{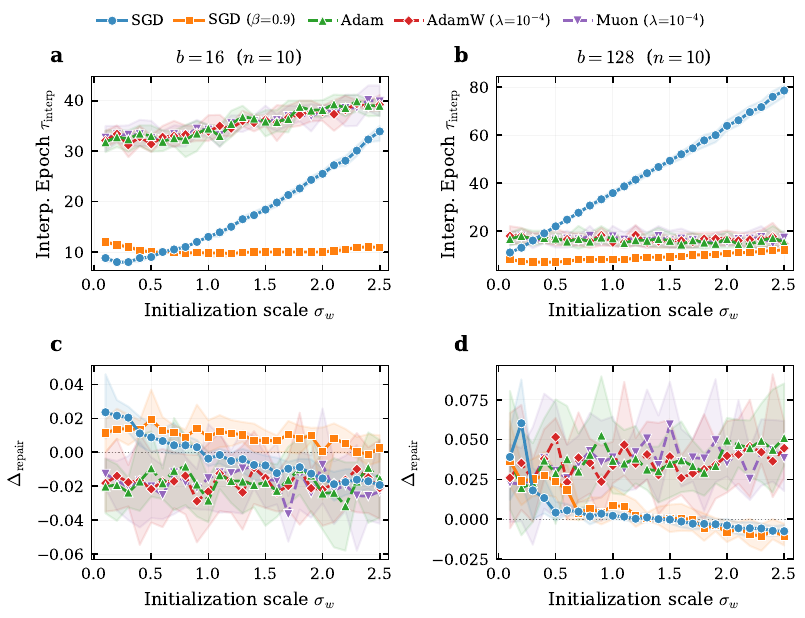}
\vspace{-0.3em}
\caption{\textbf{Interpolation is not forgetting.}
(a,b) Interpolation epoch $\tauinterp$ versus $\sigw$. SGD's interpolation time grows sharply with $\sigw$, especially at large batch sizes; Adam's is nearly flat.
(c,d) Repair gap $\Delta_{\mathrm{repair}}=\mathrm{ValAcc}_{\taubest}-\mathrm{ValAcc}_{\tauinterp}$. At large batch ($b\!\ge\!128$) Adam-family methods continue to gain validation accuracy after interpolation ($\taubest\gg\!\tauinterp$, $\Delta_{\mathrm{repair}}\!>\!0$). At small batch ($b\!\le\!64$) they reach $\taubest$ \emph{before} interpolating, so the formal repair gap can be negative even though validation accuracy keeps improving (see Appendix.~\ref{app:tables}). Vanilla SGD's repair gap stays near zero: validation accuracy is flat between $\tau_{\mathrm{best}}$ and $\tau_{\mathrm{interp}}$, even as both grow with $\sigma_w$.}
\label{fig:dynamics}
\vspace{-1.0em}
\end{figure}
\vspace{-0.5em}
\paragraph{Metrics.}
We report test accuracy at the best-validation-loss checkpoint \(\taubest\). We
also track the interpolation epoch, $\tauinterp=\min\{t:\TrainAcc_t\ge 99.5\%\},$
the trainable-kernel Frobenius norm
\[
    \|W\|_F=\left(\sum_\ell\|W^{(\ell)}\|_F^2\right)^{1/2},
\]
and the checkpoint repair gap, 
$    \Delta_{\mathrm{repair}}
    =
    \ValAcc_{\taubest}-\ValAcc_{\tauinterp}.$

The interpolation epoch measures when the labels are fit. The norm measures
radial memory of the initial scale. The repair gap measures how validation
accuracy changes between interpolation and the selected checkpoint.

\section{Memory of Initialization: Optimizers and Hyperparameters}
\label{sec:memory}
\vspace{-0.5em}
We now compare how much initialization-scale memory different procedures retain.
The central contrast is that, under the shared low-learning-rate diagnostic
procedure, vanilla SGD leaves \(\sigw\) visible at the selected checkpoint, whereas
Adam, AdamW, and Muon largely erase it. Two diagnostics refine this comparison: the kernel norm tests whether the radial scale has been overwritten, and the
checkpoint analysis separates fitting the labels from forgetting the initial
scale.

\vspace{-0.3em}
\subsection{SGD \textit{vs.} Adam and Muon}
Figure~\ref{fig:memory} exemplifies the basic optimizer contrast; Figure~\ref{fig:heatmap} extends it to the full $\sigma_w\!\times\!b$ grid.

\paragraph{Generalization.}
\textit{{Vanilla SGD remembers where it started}}: at $b=128$ it exceeds $99.5\%$ training accuracy across the entire $\sigw$ sweep, yet its test accuracy ranges from $85.0\%$ at $\sigw=0.1$ down to $58.6\%$ at $\sigw=2.5$, a $26.5$ pp spread between two runs that, by training accuracy alone, appear to have learned the same data. \textit{{Momentum is only a partial cure}}: at small batch it reduces the spread substantially (from $22.8$ to $10.9$ pp at $b=16$), but at $b=128$ under the same low-LR procedure it barely helps. 

\textit{{Adam-family methods largely wash out initialization scale}}: with the same data, learning rate, and epoch budget, Adam, AdamW, and Muon have $b{=}128$ spreads of only $4.3$, $4.7$, and $4.0$\,pp, respectively, and Adam's $b{=}16$ spread is $1.6$\,pp against SGD's $22.8$\,pp. The fact that \emph{Adam itself} (no weight decay) matches AdamW and Muon (decoupled $\lambda{=}10^{-4}$) shows the effect is not driven by the optimizer-default weight decay (Appendix~\ref{app:experiments}).

\paragraph{Weight Norm.}
The norm panels in Figure~\ref{fig:memory}~(c,f) show the same distinction in
parameter space. Under low-LR SGD, \(\|W_{\taubest}\|_F\) remains coupled to
\(\|W_{\mathrm{init}}\|_F\): larger initial kernels lead to larger selected
kernel norms. Adam, AdamW, and Muon instead move different initial scales toward
a common norm range. We therefore read \(\|W_{\taubest}\|_F\) as a radial memory
diagnostic, not as a measure of total distance traveled.

% \paragraph{Epoch count is not the right timescale.} In these heatmaps [Figure~\ref{fig:heatmap}], increasing $b$ at fixed epoch count changes two things at once: it reduces gradient noise, and it reduces the total number of parameter updates, since $K \approx E\,n_{\mathrm{train}}/b$. The appropriate interpretation is therefore not simply that ``large batch is worse,'' but that under a fixed-epoch budget, larger-batch regimes erase initialization more slowly. 

% Figure~\ref{fig:dynamics} and Tables~\ref{tab:ckpt_low}--\ref{tab:ckpt_high}
% (Appendix~\ref{app:tables}) separates two events that are easy to conflate. The
% interpolation epoch $\tau_{\mathrm{interp}}$ asks when the network fits the
% labels; the repair gap $\Delta_{\mathrm{repair}}$ asks how much it then continues to improve. The two diverge sharply across optimizers. For vanilla SGD at
% $b=128$, $\tau_{\mathrm{interp}}$ grows from $11$ epochs at $\sigw=0.1$ to $79$ epochs at $\sigw=2.5$: large initial norms substantially slow optimization. But once the model interpolates, validation accuracy improves only marginally: the repair gap is small across the sweep. Adam interpolates in roughly the same number of epochs at every $\sigw$, and at a large batch, it continues to gain several points of validation accuracy \emph{after} interpolation.
\vspace{-0.5em}
\paragraph{The message.}
Forgetting initialization is therefore not the event of fitting the labels. Under low-LR SGD, interpolation is reached before radial memory has been erased, and little subsequent repair occurs. The behavior of the Adam-family is batch-dependent: at large batch ($b\ge128$) the best-validation-loss checkpoint is reached \emph{after} interpolation, and the continued post-interpolation movement closes the $\sigw$ gap; at small batches ($b \le 64$) the best-validation-loss checkpoint is reached \emph{before} interpolation, i.e., gap is already closed in the pre-interpolation phase (Figure~\ref{fig:dynamics}, Appendix.~\ref{app:tables}). 

\begin{summarybox}
\textbf{Message 1.}
\textit{Adam(W) and Muon largely reduce initialization memory. SGD generally retains initialization memory.}
\end{summarybox}

\vspace{-0.3em}

% ============================================================
% Drop-in replacement for Section 4.
% Assumes Section 3 already contains Figures 1--3.
% Section 5 should then be the timescale/mechanism section.
% ============================================================
\vspace{-0.55em}
\subsection{Dependence on Hyperparameters}
\label{sec:erasing}

In Figure \ref{fig:memory}, we use a deliberately diagnostic low-learning-rate SGD baseline, not a tuned SGD procedure.  The natural objection is therefore simple: perhaps SGD only needs more time or better hyperparameters. In this section, we analyze how forgetting the initialization depends on the hyperparameters.

Figure~\ref{fig:repair} tests this directly with a targeted ResNet-9 sweep over training length, learning rate, and explicit $L_2$ regularization. See Appendix~\ref{app:ablation} for more details. Precisely, we see that:

\begin{figure}[t]
\centering
\includegraphics[width=0.8\linewidth]{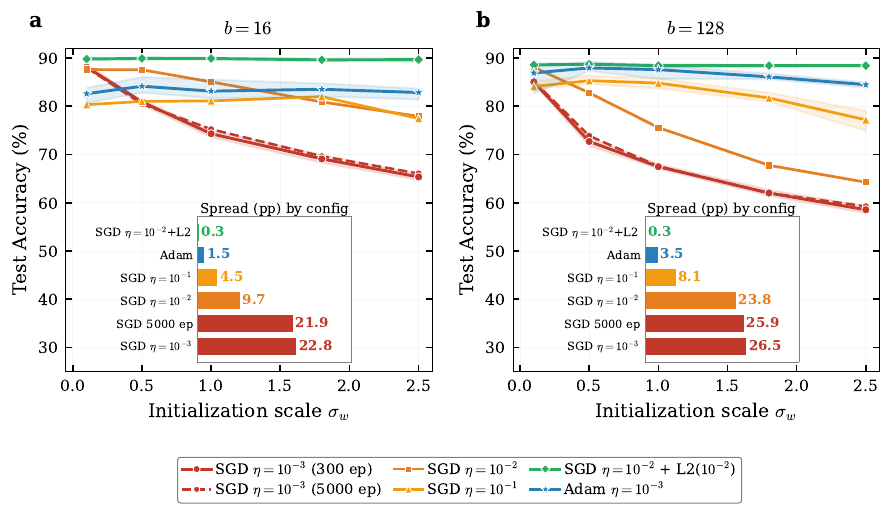}
\caption{%
\textbf{What helps SGD forget initialization?}
Test accuracy vs.\ $\sigma_w$ for (a) $b{=}16$, and (b) $b{=}128$; inset shows the spread (pp) per configuration, sorted ascending. Long training alone leaves the curves nearly unchanged; larger learning rates and especially moderate LR with explicit $L_2$ sharply reduce the spread. Spreads use the $5$-point grid $\sigma_w\!\in\!\{0.1,0.5,1.0,1.8,2.5\}$ [Table~\ref{tab:lr_l2_sweep}]; these agree with the $25$-point spreads of \S\ref{sec:memory} to within $\sim 1$\,pp.}
\label{fig:repair}
\vspace{-1.0em}
\end{figure}
\vspace{-0.3em}

%\pier{for all these claims, you should reference a figure as "see Figure 3" and/or to an appendix.}

\begin{itemize}[leftmargin=0.8em,itemsep=0.25em]
\item \textbf{Weight decay and regularization.} Explicit regularization is most effective once the learning rate is large enough to move the iterate substantially. With $\eta_0=10^{-2}$ and cosine decay, adding $\lambda_{L_2}=10^{-2}$ reduces the spread to $0.3$ pp at both $b=16$ and $b=128$ (Table~\ref{tab:lr_l2_sweep}).
\item \textbf{Learning rate.} Increasing the SGD learning rate helps, especially at a small batch. At $b=16$, raising $\eta_0$ from $10^{-3}$ to $10^{-2}$ reduces the spread from $22.8$ to $9.7$ pp, and $\eta_0=10^{-1}$ reduces it to $4.5$ pp.  At $b=128$, the same intervention is weaker: $\eta_0=10^{-2}$ still leaves $23.8$ pp, and $\eta_0=10^{-1}$ leaves $8.1$ pp.  Large batches need stronger forgetting mechanisms.
\item \textbf{Batch size.} Figure~\ref{fig:heatmap} extends the comparison to the full $\sigw \times b$ grid. We find that \textit{large batch slows forgetting under a fixed-epoch budget.} For SGD, initialization sensitivity remains large across the entire range of batch sizes, with spreads of $22.8$ pp at $b=16$, $26.5$ pp at $b=128$, and $25.7$ pp at $b=256$; as $b$ grows, the failure pattern also becomes more orderly and the absolute test-accuracy degradation more severe. Adaptive methods with the same low-LR training procedure still erase initialization memory more effectively than SGD, but their robustness weakens as the training regime becomes larger-batch and deeper. Adam, indeed, shows the same directional trend, but much more weakly: its spread increases from $1.6$ pp at $b=16$ to $5.3$ pp at $b=256$. 

\item \textbf{Normalization and augmentation.} Replacing BatchNorm with LayerNorm or adding standard data augmentation reduces the SGD spread but does not eliminate it; Adam remains robust under all these configurations (Appendix~\ref{app:norm}).

\end{itemize}

\paragraph{The standard recipe fully erases memory.}
Under a well-tuned SGD recipe (lr$=0.1$, momentum$=0.9$, weight decay$=5\!\times\!10^{-4}$, augmentation, 200 epochs), the $\sigma_w$ spread collapses to zero at $b{=}128$ (\(|\Delta|=0.0\) pp) with $94.2\%$ test accuracy (Table~\ref{tab:best_recipe}, Appendix~\ref{app:best_recipe}). In this setting, the regularizers that produce strong generalization also erase initialization memory.
\vspace{-0.6em}
\paragraph{Depth changes the failure mode.}
Figure~\ref{fig:depth_stress} extends the initialization sweep to ResNet-56, ResNet-110, and the pooling-matched R9-AvgPool control. The main effect is not a monotonic increase in spread with depth. At $b=128$, the SGD spread is $26.5$ pp for ResNet-9, $24.2$ pp for ResNet-56, and $20.4$ pp for ResNet-110.
What changes more decisively is the absolute failure mode: ResNet-9 typically still interpolates and then generalizes poorly, whereas deeper ResNets increasingly fail to train well at large $\sigma_w$. Full numerical results (test and train accuracy, interpolation epochs) and the corresponding train accuracy figure [Figure~\ref{fig:depth_train}] are provided in  Appendix~\ref{app:depth}. The regime, therefore, shifts from memorization without good generalization to degraded trainability itself.

\textbf{SGD can forget, but only when the training procedure makes it forget.} The repair map changes the interpretation of the optimizer comparison.  The message is not that Adam is intrinsically capable of forgetting while SGD is not.  Rather, Adam forgets automatically in this grid, while SGD forgets only when its training pipeline supplies both movement and radial control.  Initialization robustness is therefore a property of the full training pipeline, not of the optimizer name alone.

% \subsection{Depth stress test: the cost of remembering}
% \label{subsec:depth_stress}

% \begin{summarybox}
% {\textbf{Message 2.}
% \textit{Memory of initialization is a property of the full training pipeline. SGD forgets only when the procedure supplies both a sufficient step size and explicit norm control.}}
% \end{summarybox}

% Mine was more descriptive
\begin{summarybox}
\textbf{Message 2.}
\textit{Memory of initialization is, in general, a property of the whole training procedure, not a precise subset. Larger depth and batch sizes, or smaller regularization or learning rate, increase the memory of initialization.}
\end{summarybox}

\begin{figure}[t]
\centering
\includegraphics[width=1.0\linewidth]{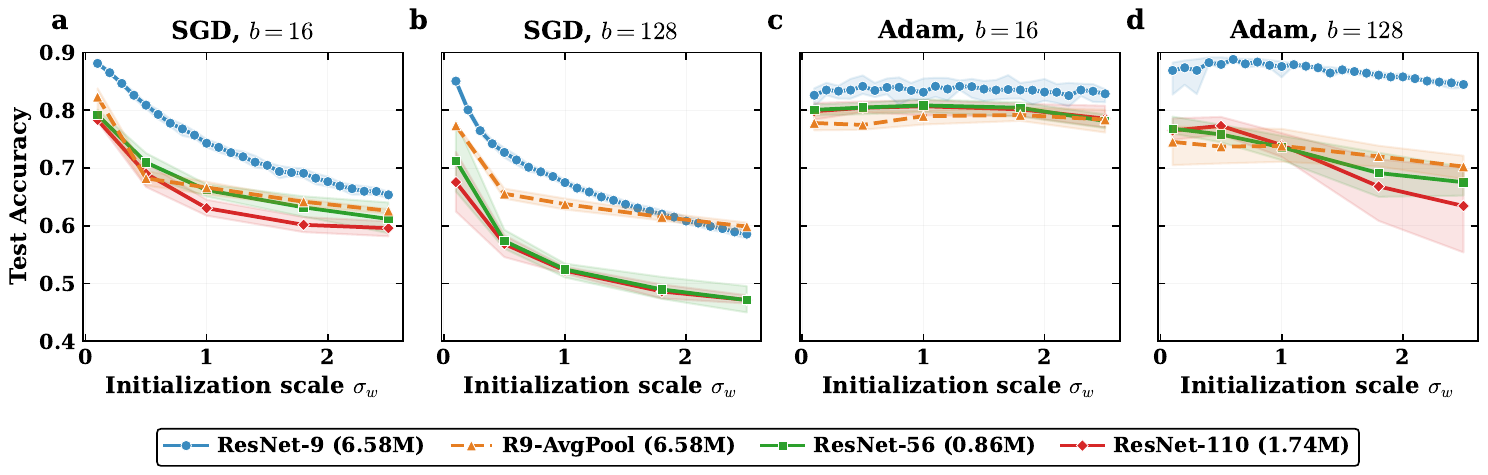}
\caption{\textbf{Depth makes poor forgetting dynamics more damaging; pooling is a partial confound.}
Test accuracy versus $\sigw$ for ResNet-9, R9-AvgPool, ResNet-56, and ResNet-110.  Switching pooling strategy reduces ResNet-9 performance, but the 9-layer R9 remains more robust than ResNet-56, suggesting that depth and optimization difficulty contribute beyond pooling alone.}
\label{fig:depth_stress}
\end{figure}
\vspace{-0.5em}
\section{The Time Scale of Forgetting}
\label{sec:timescales}
Sections~\ref{sec:memory} showed \emph{what} is remembered
and forgotten; this section asks \emph{why}.  The key idea is that epoch count
is not the natural unit of forgetting: the relevant clock is the accumulated
strength of the regularization---implicit or explicit---that acts along the
trajectory.
\vspace{-0.5em}
\subsection{The time scale is regularization}
\paragraph{Backward error view.}
Backward error analysis interprets a discrete optimizer as follows, to
leading order, a nearby continuous dynamics
\citep{griffiths1986scope,hairer2006geometric}.  In ML, this viewpoint writes
the discrete update as a gradient flow on a modified objective
\[
    \mathcal L_{\mathrm{mod}}
    \;=\;
    \mathcal L
    \;+\;
    \rho_{\mathcal A}(\eta,b,\lambda,\ldots)\,
    \mathcal R_{\mathcal A}
    \;+\;
    \text{higher-order terms},
\]
where $\mathcal R_{\mathcal A}$ is the implicit or explicit regularizer
associated with algorithm $\mathcal A$: implicit gradient regularization for
finite-step GD~\citep{barrett2021implicit}, covariance/Fisher-type
regularization for minibatch
SGD~\citep{smith2021origin,beneventano2023trajectories},
momentum-amplified regularization~\citep{ghosh2023implicit}, and
geometry-dependent implicit bias for adaptive
methods~\citep{cattaneo2024implicit}.  What matters is not
$\rho_{\mathcal A}$ itself, but its accumulated strength along the trajectory: the events that erase initialization-dependent quantities are exactly the
perturbation terms of $\mathcal L_{\mathrm{mod}}$ accumulating over many
steps.

\paragraph{Forgetting timescales.}
For the mechanisms our experiments isolate, the leading cumulative scales
along a schedule $(\eta_k)_{k<K}$ are
\[
    \mathcal T_{\mathrm{SGD}}
    =
    \frac{1}{b}\sum_{k<K}\eta_k^2,
    \qquad
    \mathcal T_{L_2}
    =
    \lambda\sum_{k<K}\eta_k,
    \qquad
    \mathcal T_{\mathrm{adapt}}
    =
    \sum_{k<K}\eta_k.
\]
The first scale corresponds to the minibatch finite-step/covariance effect: the $\eta^2/b$ scaling is the cumulative strength of the SGD modified-loss correction~\citep{smith2021origin,beneventano2023trajectories}.  The second
is explicit radial norm decay, which contracts initialization-dependent norm
at rate $\lambda\eta$ per step.  The third is the first-order adaptive
transport scale: adaptive preconditioners do not preserve the symmetries that
make the gradient-flow imbalance conserved, so they can move
initialization-dependent quantities at order $\eta$ per
step~\citep{cattaneo2024implicit}.  For constant learning rate these reduce
to $K\eta^2/b$, $K\eta\lambda$, and $K\eta$.  We use these expressions as
an organizing principle for the experiments, not as a complete theorem for
nonlinear BatchNorm ResNets.

\begin{figure}[t]
\centering
\includegraphics[width=0.95\textwidth]{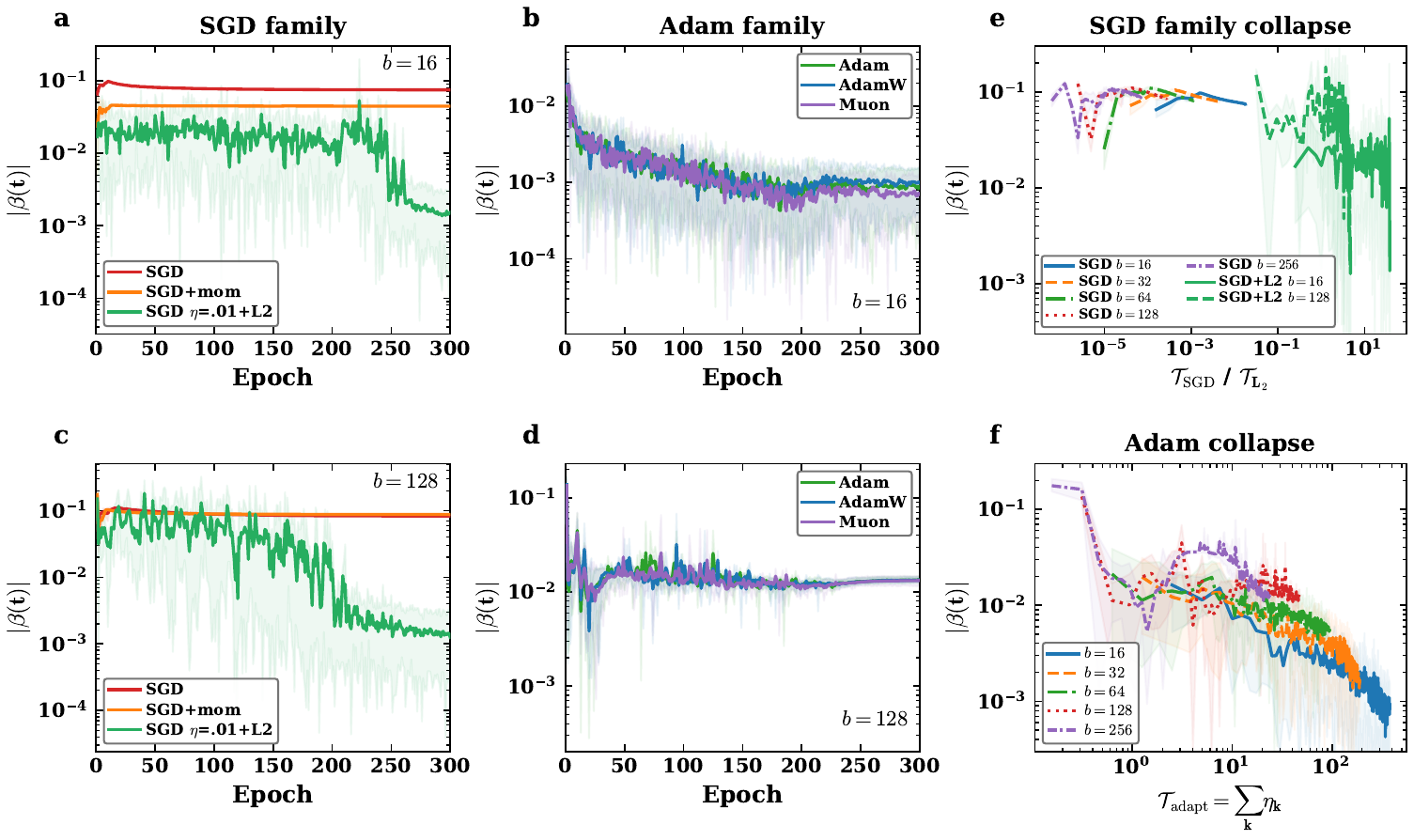}
\caption{\textbf{Initialization sensitivity decays on regularization
timescales, not epoch count.}
(a--d)~$|\beta(t)|$ versus epoch for the SGD family (left) and
adaptive methods (right) at $b{=}16$ (top) and $b{=}128$ (bottom).
Vanilla SGD (red) and SGD+momentum (orange) stay flat; Adam/AdamW/Muon
decay steadily; SGD with $\eta{=}10^{-2}$, $L_2{=}10^{-2}$ (green)
reaches the adaptive noise floor.
(e)~Vanilla SGD $|\beta|$ versus $\mathcal T_{\mathrm{SGD}}{=}(1/b)\sum_k\eta_k^2$
across all five batch sizes (colored lines); SGD+$L_2$ versus
$\mathcal T_{L_2}{=}\lambda\sum_k\eta_k$ at $b{=}16,128$ (green).
(f)~Adam $|\beta|$ versus $\mathcal T_{\mathrm{adapt}}{=}\sum_k\eta_k$
across all five batch sizes.  Curves approximately collapse under the
proposed rescaling.}
\label{fig:forgetting}
\vspace{-0.8em}
\end{figure}
\vspace{-0.5em}
\subsection{From timescales to data}
\paragraph{A linear-network sanity check.}
These timescales are visible in the smallest nontrivial homogeneous network:
the two-parameter scalar problem $(abx-y)^2$ with $a,b\in\R$. (Appendix~\ref{app:toy_model} proves the conservation, finite-step, $L_2$,
and adaptive-preconditioning statements below for two-factor and deep
linear networks).  Gradient flow preserves the imbalance
$D=a^2-b^2$ exactly, so the norm at convergence depends on initialization.
Explicit $L_2$ decay kills $D$ at rate $e^{-2\lambda t}$, giving a
timescale $O(\lambda\eta K)$.  Discrete minibatch SGD breaks the conservation at second order in $\eta$; the batch-dependent part of the leakage scales as  $O(\eta^2 K/b)$
\citep{arora2019implicit, beneventano2025gradient,beneventano2024support}. Adaptive preconditioning, by contrast, generically breaks this first-order
cancellation; the resulting $O(\eta K)$ clock for adaptive methods is proved in our minimal model (Appendix~\ref{app:toy_model}) and is
suggested by the implicit-bias analyses of~\citep{cattaneo2024implicit,
ghosh2023implicit}. In the regimes studied here, the effective clocks are ordered roughly as
\[
    \mathcal T_{\mathrm{SGD}}
    \ll
    \mathcal T_{L_2}
    \lesssim
    \mathcal T_{\mathrm{adapt}},
\]
once explicit \(L_2\) is large enough to matter. This ordering matches the observed forgetting hierarchy: low-LR SGD remembers, SGD with
sufficient \(L_2\) forgets, and adaptive methods forget fastest under the
diagnostic grid. In other words, the formulas explain the empirical hierarchy without requiring a
universal ordering: low-LR SGD has a tiny
\(\mathcal T_{\mathrm{SGD}}\), explicit \(L_2\) becomes effective when
\(\lambda\sum_k\eta_k\) is order one or larger, and adaptive methods have
a first-order movement clock \(\sum_k\eta_k\).
\paragraph{Numerical check.}
The timescale arithmetic explains the repair map of
Section~\ref{sec:erasing}.  Ignoring cosine decay and using
$K\!\approx\!E\,n_{\mathrm{train}}/b$, low-LR SGD at $b{=}128$ has
$\mathcal T_{\mathrm{SGD}}\!\approx\!7.3\!\cdot\!10^{-4}$ after $300$
epochs and only $\approx 1.2\!\cdot\!10^{-2}$ after $5{,}000$ epochs;
both regimes retain ${\sim}26$\,pp of spread.  At $b{=}16$ and
$\eta{=}10^{-2}$, the same timescale reaches ${\approx}\,4.7$ and the
spread drops sharply.  Explicit $L_2$ shows the same threshold: at
$b{=}128$, $\eta{=}10^{-2}$ with $\lambda{=}10^{-3}$ gives
$\mathcal T_{L_2}\!\approx\!0.94$ and still leaves $17.4$\,pp, whereas
$\lambda{=}5\!\times\!10^{-3}$ gives $\mathcal T_{L_2}\!\approx\!4.7$
and collapses the spread below $1$\,pp.  Adaptive methods erase
initialization-scale dependence without the SGD-style sweep, consistent
with the $K\eta$ scaling of $\mathcal T_{\mathrm{adapt}}$.
\paragraph{Empirical $|\beta(t)|$ collapse.}
If these timescales are the right dynamical clocks, then initialization sensitivity at different batch sizes should collapse when plotted
against~$\mathcal T$ rather than epoch count.  We test this with
$|\beta(t)|=|d\,\mathrm{ValAcc}/d\,\sigma_w|$ at each epoch, estimated
by ordinary least squares regression across the $25$-point $\sigma_w$ sweep ($10$~seeds).

Figure~\ref{fig:forgetting}(a--d) shows $|\beta|$ versus epoch: vanilla SGD stays flat near $10^{-1}$; Adam, AdamW, and Muon decay by $1$--$2$ orders of magnitude; the SGD ablation with $\eta{=}10^{-2}$ and
$L_2{=}10^{-2}$ reaches the adaptive noise floor.  Panels~(e--f) replot against the cumulative timescale~$\mathcal T$ for all five batch sizes ($b\in\{16,32,64,128,256\}$): under $\mathcal T_{\mathrm{SGD}}$
the vanilla-SGD curves approximately collapse, and adding SGD with $L_2$ (plotted against $\mathcal T_{L_2}$) shows the same decay trajectory as the adaptive methods.  Panel~(f) confirms that Adam curves at all five batch sizes align under
$\mathcal T_{\mathrm{adapt}}$.  The collapse is not exact---$|\beta|$ is a noisy estimator and the timescale derivations assume linear, scale-invariant dynamics---but the qualitative alignment across a $16{\times}$ range of batch sizes is clear.

\begin{summarybox}
\textbf{Message 3.} \textit{The right unit of forgetting is accumulated regularization, not
epoch count.  Initialization memory gives an empirical readout of
implicit and explicit regularization: the optimizer is not only fitting
the data, but also deciding how quickly the initial condition stops
mattering.}
\end{summarybox}

\section{Conclusion}
\label{sec:conclusion}
\vspace{-0.7em}
\paragraph{An apparent paradox, and its resolution.}
A line of work argues that the parameter--function map of deep networks is
strongly biased toward simple functions, and that this architectural prior
already accounts for
generalization~\citep{valle2019deep,mingard2020neural,mingard2025occam}.
Our results sharpen this picture in a way that initially looks paradoxical:
the training regimes that preserve initialization bias (low-LR SGD) leave generalization fragile to the initialization, while those that erase it (Adam, well-tuned SGD with weight decay) generalize uniformly well across the $\sigma_w$ sweep.  The standard ResNet recipe
($\eta{=}0.1$, momentum, $\lambda{=}5\!\times\!10^{-4}$, augmentation)
collapses initialization memory to $|\Delta|{=}0.0$\,pp at $b{=}128$
while achieving the highest test accuracy in our study, $94.2\%$
(Appendix~\ref{app:best_recipe}).  The Occam's razor that matters in practice is therefore not the one built into the parameter--function map at
initialization; it is the one imposed by the regularizers accumulated along
the trajectory---and the same mechanisms that close the $\sigw$ gap are the
ones that produce good generalization.

\vspace{-0.7em}
\paragraph{Implications, limitations, and future work.}
Our controlled study (6{,}250+ runs across optimizers, batch sizes, depths,
regularizers, normalizations, and training lengths) establishes that
interpolation does not imply forgetting, long, low-LR training does not
erase memory, and that forgetting requires accumulated regularization on
optimizer-dependent timescales ($K\eta^2/b$, $K\eta\lambda$, $K\eta$).
A practical, somewhat counterintuitive consequence is that designing
initialization schemes purely ``for predictor performance'' bring little payoff once the pipeline includes sufficient regularization: any initialization-dependent advantage is erased on the same timescale as the one that delivers
generalization.  The initialization design remains valuable as a
\emph{trainability} device---stabilizing optimization, enabling depth,
supporting shorter training horizons and controlling the signal
propagation---but it should not be expected to leave a persistent fingerprint
under modern, regularization-rich procedures.  Conversely, in gradient-flow-like
regimes (small~$\eta$, large~$b$, no weight decay) initialization-induced
bias does survive, and changing the prior at $t{=}0$ becomes a real lever on
the final predictor.

Our experiments are restricted to BatchNorm ResNets on CIFAR-10, and the
timescale expressions are organizing principles derived from linear and
scale-invariant arguments rather than full nonlinear
theorems---limitations shared with the implicit-bias literature we build
on~\citep{barrett2021implicit,smith2021origin,beneventano2023trajectories,
cattaneo2024implicit}.  Our $\sigma_w$ probe is also one-dimensional,
leaving bias distributions, layer-wise scalings, and $\mu$P
parameterizations unstudied.  The most pressing extensions are to
larger-scale settings (ImageNet, ViTs, LLM pretraining and
fine-tuning), where recent seed-dependence
results---across pretraining~\citep{vanderwal2025polypythias,fehlauer2025convergence,tong2026seedprints,li2026transformersborn}
and fine-tuning~\citep{dodge2020finetuning}---suggest initialization memory
may be both measurable and consequential, and to tighten the backward-error
framework into a quantitative theory of forgetting time for homogeneous and
approximately scale-invariant networks.

%\pier{Implications, Limitations, and Future work.}

%\pier{ideally, we list some limitations clearly, we motivate why they are the case, and we say that also the relevant literature has them (when possible).
%Then, in future work, I would say that we want to find workarounds to the limitations and understand also the effects of the other ingredients of the training pipeline.
%In implication, we say that:
%It makes no sense to find a better initialization procedure, unless their implicit bias can be kept through the training.
%So, new initialization procedures to stabilize or shorten training make sense, but not "for performance" if we run with regularization. This is unintuitive from a dynamical systems perspective. Something like that.
%}

% \newpage

\bibliographystyle{unsrt}
\bibliography{Draft/references_2}

\begin{thebibliography}{100}

\bibitem{zhang2021understanding}
Chiyuan Zhang, Samy Bengio, Moritz Hardt, Benjamin Recht, and Oriol Vinyals.
\newblock Understanding deep learning (still) requires rethinking generalization.
\newblock {\em Communications of the ACM}, 64(3):107--115, 2021.

\bibitem{strogatz2015nonlinear}
Steven~H. Strogatz.
\newblock {\em Nonlinear Dynamics and Chaos: With Applications to Physics, Biology, Chemistry, and Engineering}.
\newblock Westview Press, 2 edition, 2015.

\bibitem{hirsch2013differential}
Morris~W. Hirsch, Stephen Smale, and Robert~L. Devaney.
\newblock {\em Differential Equations, Dynamical Systems, and an Introduction to Chaos}.
\newblock Academic Press, 3 edition, 2013.

\bibitem{min2021explicit}
Hancheng Min, Salma Tarmoun, Ren{\'e} Vidal, and Enrique Mallada.
\newblock On the explicit role of initialization on the convergence and implicit bias of overparametrized linear networks.
\newblock In {\em Proceedings of the 38th International Conference on Machine Learning}, volume 139 of {\em Proceedings of Machine Learning Research}, pages 7760--7768. PMLR, 2021.

\bibitem{gruber2024role}
Oria Gruber and Haim Avron.
\newblock On the role of initialization on the implicit bias in deep linear networks, 2024.

\bibitem{valle2019deep}
Guillermo Valle-P{\'e}rez, Chico~Q. Camargo, and Ard~A. Louis.
\newblock Deep learning generalizes because the parameter--function map is biased towards simple functions.
\newblock In {\em International Conference on Learning Representations}, 2019.

\bibitem{mingard2025occam}
Chris Mingard, Henry Rees, Guillermo Valle-P{\'e}rez, and Ard~A. Louis.
\newblock Deep neural networks have an inbuilt occam's razor.
\newblock {\em Nature Communications}, 16:220, 2025.

\bibitem{fink2026deep}
Thomas Fink.
\newblock Deep-layered machines have a built-in occam's razor.
\newblock {\em arXiv preprint arXiv:2603.01217}, 2026.

\bibitem{glorot2010understanding}
Xavier Glorot and Yoshua Bengio.
\newblock Understanding the difficulty of training deep feedforward neural networks.
\newblock In {\em Proceedings of the Thirteenth International Conference on Artificial Intelligence and Statistics}, volume~9 of {\em Proceedings of Machine Learning Research}, pages 249--256. PMLR, 2010.

\bibitem{he2015delving}
Kaiming He, Xiangyu Zhang, Shaoqing Ren, and Jian Sun.
\newblock Delving deep into rectifiers: Surpassing human-level performance on imagenet classification.
\newblock In {\em Proceedings of the IEEE International Conference on Computer Vision}, pages 1026--1034, 2015.

\bibitem{poole2016exponential}
Ben Poole, Subhaneil Lahiri, Maithra Raghu, Jascha Sohl-Dickstein, and Surya Ganguli.
\newblock Exponential expressivity in deep neural networks through transient chaos.
\newblock In {\em Advances in Neural Information Processing Systems}, volume~29, 2016.

\bibitem{schoenholz2017deep}
Samuel~S. Schoenholz, Justin Gilmer, Surya Ganguli, and Jascha Sohl-Dickstein.
\newblock Deep information propagation.
\newblock In {\em International Conference on Learning Representations}, 2017.

\bibitem{pennington2017resurrecting}
Jeffrey Pennington, Samuel~S. Schoenholz, and Surya Ganguli.
\newblock Resurrecting the sigmoid in deep learning through dynamical isometry: Theory and practice.
\newblock In {\em Advances in Neural Information Processing Systems}, volume~30, 2017.

\bibitem{hanin2018start}
Boris Hanin and David Rolnick.
\newblock How to start training: The effect of initialization and architecture.
\newblock In {\em Advances in Neural Information Processing Systems}, volume~31, 2018.

\bibitem{hanin2018gradients}
Boris Hanin.
\newblock Which neural net architectures give rise to exploding and vanishing gradients?
\newblock In {\em Advances in Neural Information Processing Systems}, volume~31, 2018.

\bibitem{xiao2018dynamical}
Lechao Xiao, Yasaman Bahri, Jascha Sohl-Dickstein, Samuel~S. Schoenholz, and Jeffrey Pennington.
\newblock Dynamical isometry and a mean field theory of {CNN}s: How to train 10,000-layer vanilla convolutional neural networks.
\newblock In {\em Proceedings of the 35th International Conference on Machine Learning}, volume~80 of {\em Proceedings of Machine Learning Research}, pages 5393--5402. PMLR, 2018.

\bibitem{chen2018dynamical}
Minmin Chen, Jeffrey Pennington, and Samuel~S. Schoenholz.
\newblock Dynamical isometry and a mean field theory of {RNN}s: Gating enables signal propagation in recurrent neural networks.
\newblock In {\em Proceedings of the 35th International Conference on Machine Learning}, volume~80 of {\em Proceedings of Machine Learning Research}, pages 873--882. PMLR, 2018.

\bibitem{yang2021feature}
Greg Yang and Edward~J. Hu.
\newblock Tensor programs {IV}: Feature learning in infinite-width neural networks.
\newblock In {\em Proceedings of the 38th International Conference on Machine Learning}, volume 139 of {\em Proceedings of Machine Learning Research}, pages 11727--11737. PMLR, 2021.

\bibitem{yang2021tensorv}
Greg Yang, Edward~J. Hu, Igor Babuschkin, Szymon Sidor, Xiaodong Liu, David Farhi, Nick Ryder, Jakub Pachocki, Weizhu Chen, and Jianfeng Gao.
\newblock Tensor programs {V}: Tuning large neural networks via zero-shot hyperparameter transfer.
\newblock In {\em Advances in Neural Information Processing Systems}, volume~34, 2021.

\bibitem{bordelon2023self}
Blake Bordelon and Cengiz Pehlevan.
\newblock Self-consistent dynamical field theory of kernel evolution in wide neural networks.
\newblock {\em Journal of Statistical Mechanics: Theory and Experiment}, 2023(11):114009, 2023.

\bibitem{bordelon2024finite}
Blake Bordelon and Cengiz Pehlevan.
\newblock Dynamics of finite width kernel and prediction fluctuations in mean field neural networks.
\newblock {\em Journal of Statistical Mechanics: Theory and Experiment}, 2024(10):104021, 2024.

\bibitem{bordelon2025deep}
Blake Bordelon and Cengiz Pehlevan.
\newblock Deep linear network training dynamics from random initialization: Data, width, depth, and hyperparameter transfer.
\newblock In {\em Proceedings of the 42nd International Conference on Machine Learning}, volume 267 of {\em Proceedings of Machine Learning Research}, pages 4968--4997. PMLR, 2025.

\bibitem{lauditi2025adaptive}
Clarissa Lauditi, Blake Bordelon, and Cengiz Pehlevan.
\newblock Adaptive kernel predictors from feature-learning infinite limits of neural networks.
\newblock In {\em Proceedings of the 42nd International Conference on Machine Learning}, volume 267 of {\em Proceedings of Machine Learning Research}, pages 32617--32648. PMLR, 2025.

\bibitem{dodge2020finetuning}
Jesse Dodge, Gabriel Ilharco, Roy Schwartz, Ali Farhadi, Hannaneh Hajishirzi, and Noah~A. Smith.
\newblock Fine-tuning pretrained language models: Weight initializations, data orders, and early stopping.
\newblock {\em CoRR}, abs/2002.06305, 2020.

\bibitem{vanderwal2025polypythias}
Oskar van~der Wal, Pietro Lesci, Max M{\"u}ller-Eberstein, Naomi Saphra, Hailey Schoelkopf, Willem~H. Zuidema, and Stella~R. Biderman.
\newblock {PolyPythias}: Stability and outliers across fifty language model pre-training runs.
\newblock In {\em International Conference on Learning Representations}, 2025.

\bibitem{fehlauer2025convergence}
Finlay Fehlauer, Kyle Mahowald, and Tiago Pimentel.
\newblock Convergence and divergence of language models under different random seeds.
\newblock In {\em Proceedings of the 2025 Conference on Empirical Methods in Natural Language Processing}, pages 32982--32991, Suzhou, China, 2025. Association for Computational Linguistics.

\bibitem{tong2026seedprints}
Yao Tong, Haonan Wang, Siquan Li, Kenji Kawaguchi, and Tianyang Hu.
\newblock {SeedPrints}: Fingerprints can even tell which seed your large language model was trained from.
\newblock In {\em International Conference on Learning Representations}, 2026.
\newblock Poster.

\bibitem{li2026transformersborn}
Siquan Li, Yao Tong, Haonan Wang, and Tianyang Hu.
\newblock Transformers are born biased: Structural inductive biases at random initialization and their practical consequences, 2026.

\bibitem{allenzhu2024physics31}
Zeyuan Allen-Zhu and Yuanzhi Li.
\newblock Physics of language models: Part 3.1, knowledge storage and extraction.
\newblock In {\em Proceedings of the 41st International Conference on Machine Learning}, volume 235 of {\em Proceedings of Machine Learning Research}, pages 1067--1077. PMLR, 2024.

\bibitem{allenzhu2025physics41}
Zeyuan Allen-Zhu.
\newblock Physics of language models: Part 4.1, architecture design and the magic of canon layers.
\newblock In {\em Proceedings of the 39th Conference on Neural Information Processing Systems}, NeurIPS~'25, 2025.
\newblock Full version available at \url{https://ssrn.com/abstract=5240330}.

\bibitem{li2020reconciling}
Zhiyuan Li, Kaifeng Lyu, and Sanjeev Arora.
\newblock Reconciling modern deep learning with traditional optimization analyses: The intrinsic learning rate.
\newblock In {\em Advances in Neural Information Processing Systems}, volume~33, 2020.

\bibitem{barrett2021implicit}
David G.~T. Barrett and Benoit Dherin.
\newblock Implicit gradient regularization.
\newblock In {\em International Conference on Learning Representations}, 2021.

\bibitem{smith2021origin}
Samuel~L. Smith, Benoit Dherin, David G.~T. Barrett, and Soham De.
\newblock On the origin of implicit regularization in stochastic gradient descent.
\newblock In {\em International Conference on Learning Representations}, 2021.

\bibitem{beneventano2023trajectories}
Pierfrancesco Beneventano.
\newblock On the trajectories of sgd without replacement.
\newblock {\em arXiv preprint arXiv:2312.16143}, 2023.

\bibitem{beneventano2024support}
Pierfrancesco Beneventano, Andrea Pinto, and Tomaso Poggio.
\newblock How neural networks learn the support is an implicit regularization effect of {SGD}.
\newblock 2024.

\bibitem{griffiths1986scope}
David~F. Griffiths and J.~M. Sanz-Serna.
\newblock On the scope of the method of modified equations.
\newblock {\em SIAM Journal on Scientific and Statistical Computing}, 7(3):994--1008, 1986.

\bibitem{hairer2006geometric}
Ernst Hairer, Christian Lubich, and Gerhard Wanner.
\newblock {\em Geometric Numerical Integration: Structure-Preserving Algorithms for Ordinary Differential Equations}, volume~31 of {\em Springer Series in Computational Mathematics}.
\newblock Springer, Berlin, Heidelberg, 2 edition, 2006.

\bibitem{ghosh2023implicit}
Avrajit Ghosh, He~Lyu, Xitong Zhang, and Rongrong Wang.
\newblock Implicit regularization in heavy-ball momentum accelerated stochastic gradient descent.
\newblock {\em arXiv preprint arXiv:2302.00849}, 2023.

\bibitem{cattaneo2024implicit}
Matias~D. Cattaneo, Jason~Matthew Klusowski, and Boris Shigida.
\newblock On the implicit bias of {A}dam.
\newblock In {\em Proceedings of the 41st International Conference on Machine Learning}, volume 235 of {\em Proceedings of Machine Learning Research}, pages 5862--5906. PMLR, 2024.

\bibitem{arora2019implicit}
Sanjeev Arora, Nadav Cohen, Wei Hu, and Yuping Luo.
\newblock Implicit regularization in deep matrix factorization.
\newblock In {\em Advances in Neural Information Processing Systems}, volume~32, 2019.

\bibitem{beneventano2025gradient}
Pierfrancesco Beneventano and Blake Woodworth.
\newblock Gradient descent converges linearly to flatter minima than gradient flow in shallow linear networks.
\newblock {\em arXiv preprint arXiv:2501.09137}, 2025.

\bibitem{mingard2020neural}
Chris Mingard, Joar Skalse, Guillermo Valle-P{\'e}rez, David Mart{\'i}nez-Rubio, Vladimir Mikulik, and Ard~A. Louis.
\newblock Neural networks are a priori biased towards boolean functions with low entropy, 2020.

\bibitem{vapnik1971uniform}
Vladimir~N. Vapnik and Alexey~Ya. Chervonenkis.
\newblock On the uniform convergence of relative frequencies of events to their probabilities.
\newblock {\em Theory of Probability and Its Applications}, 16(2):264--280, 1971.

\bibitem{vapnik1998statistical}
Vladimir~N. Vapnik.
\newblock {\em Statistical Learning Theory}.
\newblock Wiley, 1998.

\bibitem{bartlett1998sample}
Peter~L. Bartlett.
\newblock The sample complexity of pattern classification with neural networks: The size of the weights is more important than the size of the network.
\newblock {\em IEEE Transactions on Information Theory}, 44(2):525--536, 1998.

\bibitem{bartlett2002rademacher}
Peter~L. Bartlett and Shahar Mendelson.
\newblock Rademacher and gaussian complexities: Risk bounds and structural results.
\newblock {\em Journal of Machine Learning Research}, 3:463--482, 2002.

\bibitem{neyshabur2015norm}
Behnam Neyshabur, Ryota Tomioka, and Nathan Srebro.
\newblock Norm-based capacity control in neural networks.
\newblock In {\em Proceedings of The 28th Conference on Learning Theory}, volume~40 of {\em Proceedings of Machine Learning Research}, pages 1376--1401. PMLR, 2015.

\bibitem{bartlett2017spectrally}
Peter~L. Bartlett, Dylan~J. Foster, and Matus~J. Telgarsky.
\newblock Spectrally-normalized margin bounds for neural networks.
\newblock In {\em Advances in Neural Information Processing Systems}, volume~30, 2017.

\bibitem{zhang2017understanding}
Chiyuan Zhang, Samy Bengio, Moritz Hardt, Benjamin Recht, and Oriol Vinyals.
\newblock Understanding deep learning requires rethinking generalization.
\newblock In {\em International Conference on Learning Representations}, 2017.

\bibitem{arpit2017closer}
Devansh Arpit, Stanislaw Jastrzebski, Nicolas Ballas, David Krueger, Emmanuel Bengio, Maxinder~S. Kanwal, Tegan Maharaj, Asja Fischer, Aaron Courville, Yoshua Bengio, and Simon Lacoste-Julien.
\newblock A closer look at memorization in deep networks.
\newblock In {\em Proceedings of the 34th International Conference on Machine Learning}, volume~70 of {\em Proceedings of Machine Learning Research}, pages 233--242. PMLR, 2017.

\bibitem{nakkiran2019sgd}
Preetum Nakkiran, Gal Kaplun, Dimitris Kalimeris, Tristan Yang, Benjamin~L. Edelman, Fred Zhang, and Boaz Barak.
\newblock {SGD} on neural networks learns functions of increasing complexity.
\newblock In {\em Advances in Neural Information Processing Systems}, volume~32, 2019.

\bibitem{hardt2016train}
Moritz Hardt, Benjamin Recht, and Yoram Singer.
\newblock Train faster, generalize better: Stability of stochastic gradient descent.
\newblock In {\em Proceedings of the 33rd International Conference on Machine Learning}, volume~48 of {\em Proceedings of Machine Learning Research}, pages 1225--1234. PMLR, 2016.

\bibitem{neyshabur2017exploring}
Behnam Neyshabur, Srinadh Bhojanapalli, David McAllester, and Nathan Srebro.
\newblock Exploring generalization in deep learning.
\newblock In {\em Advances in Neural Information Processing Systems}, volume~30, 2017.

\bibitem{neyshabur2018pacbayes}
Behnam Neyshabur, Srinadh Bhojanapalli, David McAllester, and Nathan Srebro.
\newblock A {PAC}-bayesian approach to spectrally-normalized margin bounds for neural networks.
\newblock In {\em International Conference on Learning Representations}, 2018.

\bibitem{jiang2019predicting}
Yiding Jiang, Dilip Krishnan, Hossein Mobahi, and Samy Bengio.
\newblock Predicting the generalization gap in deep networks with margin distributions.
\newblock In {\em International Conference on Learning Representations}, 2019.

\bibitem{jiang2020fantastic}
Yiding Jiang, Behnam Neyshabur, Hossein Mobahi, Dilip Krishnan, and Samy Bengio.
\newblock Fantastic generalization measures and where to find them.
\newblock In {\em International Conference on Learning Representations}, 2020.

\bibitem{hochreiter1997flat}
Sepp Hochreiter and J{\"u}rgen Schmidhuber.
\newblock Flat minima.
\newblock {\em Neural Computation}, 9(1):1--42, 1997.

\bibitem{keskar2017large}
Nitish~Shirish Keskar, Dheevatsa Mudigere, Jorge Nocedal, Mikhail Smelyanskiy, and Ping Tak~Peter Tang.
\newblock On large-batch training for deep learning: Generalization gap and sharp minima.
\newblock In {\em International Conference on Learning Representations}, 2017.

\bibitem{dinh2017sharp}
Laurent Dinh, Razvan Pascanu, Samy Bengio, and Yoshua Bengio.
\newblock Sharp minima can generalize for deep nets.
\newblock In {\em Proceedings of the 34th International Conference on Machine Learning}, volume~70 of {\em Proceedings of Machine Learning Research}, pages 1019--1028. PMLR, 2017.

\bibitem{foret2021sam}
Pierre Foret, Ariel Kleiner, Hossein Mobahi, and Behnam Neyshabur.
\newblock Sharpness-aware minimization for efficiently improving generalization.
\newblock In {\em International Conference on Learning Representations}, 2021.

\bibitem{mcallester1999pac}
David~A. McAllester.
\newblock {PAC}-bayesian model averaging.
\newblock In {\em Proceedings of the Twelfth Annual Conference on Computational Learning Theory}, pages 164--170, 1999.

\bibitem{dziugaite2017computing}
Gintare~Karolina Dziugaite and Daniel~M. Roy.
\newblock Computing nonvacuous generalization bounds for deep (stochastic) neural networks with many more parameters than training data.
\newblock In {\em Proceedings of the Thirty-Third Conference on Uncertainty in Artificial Intelligence}, 2017.

\bibitem{zhou2019nonvacuous}
Wenda Zhou, Victor Veitch, Morgane Austern, Ryan~P. Adams, and Peter Orbanz.
\newblock Non-vacuous generalization bounds at the {ImageNet} scale: A {PAC}-bayesian compression approach.
\newblock In {\em International Conference on Learning Representations}, 2019.

\bibitem{arora2018stronger}
Sanjeev Arora, Rong Ge, Behnam Neyshabur, and Yi~Zhang.
\newblock Stronger generalization bounds for deep nets via a compression approach.
\newblock In {\em International Conference on Learning Representations}, 2018.

\bibitem{belkin2019reconciling}
Mikhail Belkin, Daniel Hsu, Siyuan Ma, and Soumik Mandal.
\newblock Reconciling modern machine-learning practice and the classical bias--variance trade-off.
\newblock {\em Proceedings of the National Academy of Sciences}, 116(32):15849--15854, 2019.

\bibitem{nakkiran2020deep}
Preetum Nakkiran, Gal Kaplun, Yamini Bansal, Tristan Yang, Boaz Barak, and Ilya Sutskever.
\newblock Deep double descent: Where bigger models and more data hurt.
\newblock In {\em International Conference on Learning Representations}, 2020.

\bibitem{bartlett2020benign}
Peter~L. Bartlett, Philip~M. Long, G{\'a}bor Lugosi, and Alexander Tsigler.
\newblock Benign overfitting in linear regression.
\newblock {\em Proceedings of the National Academy of Sciences}, 117(48):30063--30070, 2020.

\bibitem{neal1996bayesian}
Radford~M. Neal.
\newblock {\em Bayesian Learning for Neural Networks}, volume 118 of {\em Lecture Notes in Statistics}.
\newblock Springer, 1996.

\bibitem{williams1996computing}
Christopher K.~I. Williams.
\newblock Computing with infinite networks.
\newblock In {\em Advances in Neural Information Processing Systems}, volume~9, 1996.

\bibitem{lee2018deep}
Jaehoon Lee, Yasaman Bahri, Roman Novak, Samuel~S. Schoenholz, Jeffrey Pennington, and Jascha Sohl-Dickstein.
\newblock Deep neural networks as gaussian processes.
\newblock In {\em International Conference on Learning Representations}, 2018.

\bibitem{jacot2018neural}
Arthur Jacot, Franck Gabriel, and Cl{\'e}ment Hongler.
\newblock Neural tangent kernel: Convergence and generalization in neural networks.
\newblock In {\em Advances in Neural Information Processing Systems}, volume~31, 2018.

\bibitem{dingle2018input}
Kamaludin Dingle, Chico~Q. Camargo, and Ard~A. Louis.
\newblock Input--output maps are strongly biased towards simple outputs.
\newblock {\em Nature Communications}, 9:761, 2018.

\bibitem{depalma2019random}
Giacomo De~Palma, Bobak Kiani, and Seth Lloyd.
\newblock Random deep neural networks are biased towards simple functions.
\newblock In {\em Advances in Neural Information Processing Systems}, volume~32, 2019.

\bibitem{lempel1976complexity}
Abraham Lempel and Jacob Ziv.
\newblock On the complexity of finite sequences.
\newblock {\em IEEE Transactions on Information Theory}, 22(1):75--81, 1976.

\bibitem{ziv1977universal}
Jacob Ziv and Abraham Lempel.
\newblock A universal algorithm for sequential data compression.
\newblock {\em IEEE Transactions on Information Theory}, 23(3):337--343, 1977.

\bibitem{li2008kolmogorov}
Ming Li and Paul M.~B. Vit{\'a}nyi.
\newblock {\em An Introduction to Kolmogorov Complexity and Its Applications}.
\newblock Springer, 3 edition, 2008.

\bibitem{bhattamishra2023simplicity}
Satwik Bhattamishra, Arkil Patel, Varun Kanade, and Phil Blunsom.
\newblock Simplicity bias in transformers and their ability to learn sparse {B}oolean functions.
\newblock In {\em Proceedings of the 61st Annual Meeting of the Association for Computational Linguistics (Volume 1: Long Papers)}, pages 5767--5791, Toronto, Canada, 2023. Association for Computational Linguistics.

\bibitem{hahn2024sensitive}
Michael Hahn and Mark Rofin.
\newblock Why are sensitive functions hard for transformers?
\newblock In {\em Proceedings of the 62nd Annual Meeting of the Association for Computational Linguistics (Volume 1: Long Papers)}, pages 14973--15008, Bangkok, Thailand, 2024. Association for Computational Linguistics.

\bibitem{vasudeva2025transformers}
Bhavya Vasudeva, Deqing Fu, Tianyi Zhou, Elliott Kau, Youqi Huang, and Vatsal Sharan.
\newblock Transformers learn low sensitivity functions: Investigations and implications.
\newblock In {\em International Conference on Learning Representations}, 2025.

\bibitem{rahaman2019spectral}
Nasim Rahaman, Aristide Baratin, Devansh Arpit, Felix Draxler, Min Lin, Fred~A. Hamprecht, Yoshua Bengio, and Aaron Courville.
\newblock On the spectral bias of neural networks.
\newblock In {\em Proceedings of the 36th International Conference on Machine Learning}, volume~97 of {\em Proceedings of Machine Learning Research}, pages 5301--5310. PMLR, 2019.

\bibitem{novak2018sensitivity}
Roman Novak, Yasaman Bahri, Daniel~A. Abolafia, Jeffrey Pennington, and Jascha Sohl-Dickstein.
\newblock Sensitivity and generalization in neural networks: An empirical study.
\newblock In {\em International Conference on Learning Representations}, 2018.

\bibitem{hanin2019complexity}
Boris Hanin and David Rolnick.
\newblock Complexity of linear regions in deep networks.
\newblock In {\em Proceedings of the 36th International Conference on Machine Learning}, volume~97 of {\em Proceedings of Machine Learning Research}, pages 2596--2604. PMLR, 2019.

\bibitem{hanin2019activation}
Boris Hanin and David Rolnick.
\newblock Deep {ReLU} networks have surprisingly few activation patterns.
\newblock In {\em Advances in Neural Information Processing Systems}, volume~32, pages 359--368, 2019.

\bibitem{dherin2022geometric}
Benoit Dherin, Michael Munn, Mihaela Rosca, and David G.~T. Barrett.
\newblock Why neural networks find simple solutions: The many regularizers of geometric complexity.
\newblock In {\em Advances in Neural Information Processing Systems}, volume~35, 2022.

\bibitem{refinetti2023neural}
Maria Refinetti, Alessandro Ingrosso, and Sebastian Goldt.
\newblock Neural networks trained with {SGD} learn distributions of increasing complexity.
\newblock In {\em Proceedings of the 40th International Conference on Machine Learning}, volume 202 of {\em Proceedings of Machine Learning Research}, pages 28843--28863. PMLR, 2023.

\bibitem{boursier2024simplicity}
Etienne Boursier and Nicolas Flammarion.
\newblock Simplicity bias and optimization threshold in two-layer {ReLU} networks, 2024.

\bibitem{zhang2026saddle}
Yedi Zhang, Andrew~M. Saxe, and Peter~E. Latham.
\newblock Saddle-to-saddle dynamics explains a simplicity bias across neural network architectures.
\newblock In {\em International Conference on Learning Representations}, 2026.
\newblock Poster.

\bibitem{saxe2014exact}
Andrew~M. Saxe, James~L. McClelland, and Surya Ganguli.
\newblock Exact solutions to the nonlinear dynamics of learning in deep linear neural networks.
\newblock In {\em International Conference on Learning Representations}, 2014.

\bibitem{blumenfeld2020beyond}
Yaniv Blumenfeld, Dar Gilboa, and Daniel Soudry.
\newblock Beyond signal propagation: Is feature diversity necessary in deep neural network initialization?
\newblock In {\em Proceedings of the 37th International Conference on Machine Learning}, volume 119 of {\em Proceedings of Machine Learning Research}, pages 960--969. PMLR, 2020.

\bibitem{yang2020tensor}
Greg Yang.
\newblock Tensor programs {I}: Wide feedforward or recurrent neural networks of any architecture are gaussian processes, 2020.

\bibitem{ioffe2015batch}
Sergey Ioffe and Christian Szegedy.
\newblock Batch normalization: Accelerating deep network training by reducing internal covariate shift.
\newblock In {\em Proceedings of the 32nd International Conference on Machine Learning}, volume~37 of {\em Proceedings of Machine Learning Research}, pages 448--456. PMLR, 2015.

\bibitem{arora2019auto}
Sanjeev Arora, Kaifeng Lyu, and Zhiyuan Li.
\newblock Theoretical analysis of auto rate-tuning by batch normalization.
\newblock In {\em International Conference on Learning Representations}, 2019.

\bibitem{li2020intrinsic}
Zhiyuan Li, Kaifeng Lyu, and Sanjeev Arora.
\newblock Reconciling modern deep learning with traditional optimization analyses: The intrinsic learning rate.
\newblock In {\em Advances in Neural Information Processing Systems}, volume~33, 2020.

\bibitem{wilkinson1963rounding}
James~Hardy Wilkinson.
\newblock {\em Rounding Errors in Algebraic Processes}.
\newblock Number~32 in Notes on Applied Science. Her Majesty's Stationery Office, London, 1963.

\bibitem{wilkinson1965algebraic}
James~Hardy Wilkinson.
\newblock {\em The Algebraic Eigenvalue Problem}.
\newblock Monographs on Numerical Analysis. Clarendon Press, Oxford, 1965.

\bibitem{higham2002accuracy}
Nicholas~J. Higham.
\newblock {\em Accuracy and Stability of Numerical Algorithms}.
\newblock Society for Industrial and Applied Mathematics, Philadelphia, PA, 2 edition, 2002.

\bibitem{calvo1994modified}
M.~P. Calvo, Ander Murua, and J.~M. Sanz-Serna.
\newblock Modified equations for {ODEs}.
\newblock In Peter~E. Kloeden and Kenneth~J. Palmer, editors, {\em Chaotic Numerics}, volume 172 of {\em Contemporary Mathematics}, pages 63--74. American Mathematical Society, Providence, RI, 1994.

\bibitem{shardlow2006modified}
Tony Shardlow.
\newblock Modified equations for stochastic differential equations.
\newblock {\em BIT Numerical Mathematics}, 46(1):111--125, 2006.

\bibitem{zygalakis2011existence}
Konstantinos~C. Zygalakis.
\newblock On the existence and the applications of modified equations for stochastic differential equations.
\newblock {\em SIAM Journal on Scientific Computing}, 33(1):102--130, 2011.

\bibitem{debussche2012weak}
Arnaud Debussche and Erwan Faou.
\newblock Weak backward error analysis for {SDEs}.
\newblock {\em SIAM Journal on Numerical Analysis}, 50(3):1735--1752, 2012.

\bibitem{feng2020uniform}
Yuanyuan Feng, Tingran Gao, Lei Li, Jian-Guo Liu, and Yulong Lu.
\newblock Uniform-in-time weak error analysis for stochastic gradient descent algorithms via diffusion approximation.
\newblock {\em Communications in Mathematical Sciences}, 18(1):163--188, 2020.

\bibitem{li2017stochastic}
Qianxiao Li, Cheng Tai, and Weinan {E}.
\newblock Stochastic modified equations and adaptive stochastic gradient algorithms.
\newblock In {\em Proceedings of the 34th International Conference on Machine Learning}, volume~70 of {\em Proceedings of Machine Learning Research}, pages 2101--2110. PMLR, 2017.

\bibitem{li2019stochastic}
Qianxiao Li, Cheng Tai, and Weinan {E}.
\newblock Stochastic modified equations and dynamics of stochastic gradient algorithms {I}: Mathematical foundations.
\newblock {\em Journal of Machine Learning Research}, 20(40):1--47, 2019.

\bibitem{miyagawa2022toward}
Taiki Miyagawa.
\newblock Toward equation of motion for deep neural networks: Continuous-time gradient descent and discretization error analysis.
\newblock In {\em Advances in Neural Information Processing Systems}, volume~35, pages 37778--37791, 2022.

\bibitem{rosca2023continuous}
Mihaela Rosca, Yan Wu, Chongli Qin, and Benoit Dherin.
\newblock On a continuous time model of gradient descent dynamics and instability in deep learning.
\newblock {\em arXiv preprint arXiv:2302.01952}, 2023.

\bibitem{cattaneo2025modified}
Matias~D Cattaneo and Boris Shigida.
\newblock Modified loss of momentum gradient descent: Fine-grained analysis.
\newblock {\em arXiv preprint arXiv:2509.08483}, 2025.

\bibitem{digiovacchino2024backward}
Stefano Di~Giovacchino, Desmond~J. Higham, and Konstantinos~C. Zygalakis.
\newblock Backward error analysis and the qualitative behaviour of stochastic optimization algorithms: Application to stochastic coordinate descent.
\newblock {\em Journal of Computational Dynamics}, 11(4):453--467, 2024.

\bibitem{cattaneo2026memory}
Matias Cattaneo and Boris Shigida.
\newblock How memory in optimization algorithms implicitly modifies the loss.
\newblock {\em Advances in Neural Information Processing Systems}, 38:156059--156096, 2026.

\bibitem{cattaneo2026effect}
Matias~D Cattaneo and Boris Shigida.
\newblock The effect of mini-batch noise on the implicit bias of adam.
\newblock {\em arXiv preprint arXiv:2602.01642}, 2026.

\bibitem{soudry2018implicit}
Daniel Soudry, Elad Hoffer, Mor~Shpigel Nacson, Suriya Gunasekar, and Nathan Srebro.
\newblock The implicit bias of gradient descent on separable data.
\newblock {\em Journal of Machine Learning Research}, 19(70):1--57, 2018.

\bibitem{gunasekar2018implicit}
Suriya Gunasekar, Jason~D. Lee, Daniel Soudry, and Nathan Srebro.
\newblock Implicit bias of gradient descent on linear convolutional networks.
\newblock In {\em Advances in Neural Information Processing Systems}, volume~31, 2018.

\bibitem{lyu2020gradient}
Kaifeng Lyu and Jian Li.
\newblock Gradient descent maximizes the margin of homogeneous neural networks.
\newblock In {\em International Conference on Learning Representations}, 2020.

\bibitem{hoffer2017train}
Elad Hoffer, Itay Hubara, and Daniel Soudry.
\newblock Train longer, generalize better: Closing the generalization gap in large batch training of neural networks.
\newblock In {\em Advances in Neural Information Processing Systems}, volume~30, 2017.

\bibitem{mandt2017sgd}
Stephan Mandt, Matthew~D. Hoffman, and David~M. Blei.
\newblock Stochastic gradient descent as approximate bayesian inference.
\newblock {\em Journal of Machine Learning Research}, 18(134):1--35, 2017.

\bibitem{wilson2017marginal}
Ashia~C. Wilson, Rebecca Roelofs, Mitchell Stern, Nathan Srebro, and Benjamin Recht.
\newblock The marginal value of adaptive gradient methods in machine learning.
\newblock In {\em Advances in Neural Information Processing Systems}, volume~30, 2017.

\end{thebibliography}

\appendix

% ============================================================
% Updated Further Related Work Appendix
% Drop this file after \appendix in the paper.
% It is intended to replace the main-body Related Work section.
% Requires natbib-style citations, e.g. \usepackage{natbib}.
% The paper/template should also load \usepackage{url} or hyperref.
% ============================================================
\clearpage
\section{Experimental Details}
\label{app:experiments}

This appendix records the setups for all experiments. All runs are
implemented in Keras 3 / TensorFlow on a single NVIDIA GPU per run
(H100, H200, or L40S). A summary is given in
Table~\ref{tab:opt_specs}; per-experiment paragraphs follow.

\paragraph{Common ingredients (all experiments).}
\begin{itemize}[leftmargin=1.2em,itemsep=0.05em]
\item \emph{Data.} CIFAR-10 with the standard $50{,}000$ training and
$10{,}000$ test images. We hold out a fixed $10{,}000$-image validation
set from the training partition using permutation seed $42$, giving a
$40{,}000$/$10{,}000$/$10{,}000$ train/validation/test split that is
identical across every experiment. Inputs are scaled to $[0,1]$.
\item \emph{Initialization.} Convolutional and dense kernels are drawn
from $W_{ij}\sim\mathcal N\!\big(0,\sigma_w^{2}/\mathrm{fan}_{\mathrm{in}}\big)$,
so $\sigma_w=1$ corresponds to a fan-in normal baseline. Biases
are drawn from $\mathcal N(0,0.2^{2})$ and are not rescaled by $\sigma_w$.
\item \emph{Loss.} Sparse categorical cross-entropy from logits.
\item \emph{Checkpoint selection.} For each run we save and report
metrics at the validation-loss-minimizing epoch $\taubest$, and record
$\tauinterp=\min\{t:\TrainAcc_t\ge 99.5\%\}$ and the kernel Frobenius
norm $\|W\|_F$.
\item \emph{Reproducibility.} Each run sets Python, NumPy, and
TensorFlow random seeds to the same integer.
\end{itemize}

\paragraph{Main ResNet-9 grid (\S\ref{sec:memory}, App.~\ref{app:tables}).}
ResNet-9 with BatchNorm, no augmentation, $300$ epochs. Learning rate
follows the cosine schedule
$\eta_k = \eta_0\big[\alpha + (1-\alpha)\tfrac{1}{2}(1+\cos(\pi k/K))\big]$
with $\eta_0=10^{-3}$, $\alpha=0.01$, and total updates
$K=300\,\lceil 40{,}000/b\rceil$. The grid varies three axes:
$\sigma_w\in\{0.10,0.20,\ldots,2.50\}$ ($25$ values), optimizer in
\{SGD, SGD-momentum, Adam, AdamW, Muon\} (Table~\ref{tab:opt_specs}),
and batch size $b\in\{16,32,64,128,256\}$. Each of the
$25\times 5\times 5=625$ configurations is run for $10$ seeds
$\{0,1,42,123,456,789,999,2024,2025,2026\}$, for a total of $6{,}250$
trained models.

\paragraph{$5{,}000$-epoch SGD (\S\ref{sec:memory}, \S\ref{sec:timescales}).}
Vanilla SGD ($\eta=10^{-3}$, momentum~$0$) trained for $5{,}000$ epochs
\emph{with constant learning rate} (no cosine decay), no augmentation,
no weight decay or $L_2$. Five scales $\sigma_w\in\{0.1,0.5,1.0,1.8,2.5\}$,
batch sizes $b\in\{16,128\}$, single seed experiment.

\paragraph{LR$\times L_2$ sweep for SGD (\S\ref{sec:erasing},
Table~\ref{tab:lr_l2_sweep}).}
Vanilla SGD (no momentum), $300$ epochs, cosine schedule. Twelve
configurations $\eta\in\{10^{-3},10^{-2},10^{-1}\}$ paired with explicit
kernel $L_2$ regularization
$\lambda\in\{0,10^{-3},5\!\times\!10^{-3},10^{-2}\}$, at five scales
$\sigma_w\in\{0.1,0.5,1.0,1.8,2.5\}$ and batch sizes $b\in\{16,128\}$.

\paragraph{Norm and augmentation (App.~\ref{app:norm},
Fig.~\ref{fig:norm_comparison}).}
Two normalization variants -- BatchNorm, LayerNorm --
under standard data augmentation (random horizontal flip +
$4$-pixel reflective padding + random $32\!\times\!32$ crop). $300$
epochs cosine, $\eta=10^{-3}$, no weight decay. Two optimizers (SGD
without momentum, Adam), $\sigma_w\in\{0.1,0.5,1.0,1.8,2.5\}$,
$b\in\{16,128\}$, $5$ seeds.

\paragraph{Best-procedure comparison (App.~\ref{app:best_recipe},
Table~\ref{tab:best_recipe}).}
\textbf{$200$ epochs}, cosine schedule, with augmentation.
Two procedures: SGD ($\eta=0.1$, momentum $0.9$, weight decay
$5\!\times\!10^{-4}$) and Adam ($\eta=10^{-3}$, no weight decay), at the
two extreme scales $\sigma_w\in\{0.1,2.5\}$ and batch sizes
$b\in\{16,128,256\}$, $10$ seeds.

\paragraph{Depth (App.~\ref{app:depth}, Fig.~\ref{fig:depth_stress}).}
ResNet-56, ResNet-110 (CIFAR-style architectures, $6n+2$ layers with
$n=9$ and $n=18$ respectively), and an R9-AvgPool control that replaces
ResNet-9's global max-pool with global average-pool. All use $300$
epochs cosine, $\eta_0=10^{-3}$, the five main-grid optimizers,
$\sigma_w\in\{0.1,0.5,1.0,1.8,2.5\}$, and $b\in\{16,128\}$, with $10$
seeds. ResNet-56 and ResNet-110 use early stopping with patience $30$
(without altering the cosine schedule); R9-AvgPool does not.

\paragraph{Optimizer settings.}
Table~\ref{tab:opt_specs} lists the per-optimizer settings used in the
main grid. All five optimizers share $\eta_0=10^{-3}$ with cosine decay.
The SGD family is run with zero weight decay; AdamW and Muon retain
their default decoupled weight decay $\lambda=10^{-4}$ as part of the
optimizer definition. Adam, which has \emph{no} weight decay, exhibits
the same small spreads as AdamW and Muon, so the forgetting effect
attributed to adaptive methods is not driven by AdamW's or Muon's
$\lambda=10^{-4}$.

\begin{table}[h]
\centering\small
\caption{\textbf{Main-grid optimizer settings.}
All optimizers use $\eta_0=10^{-3}$ with cosine decay
($\alpha=0.01$) over $300$ epochs. Adam-family $\beta_1=0.9$
(momentum-like) and $\beta_2=0.999$ (variance) are the Keras defaults.}
\label{tab:opt_specs}
\begin{tabular}{lcccl}
\toprule
Optimizer & $\eta_0$ & Momentum & Weight decay $\lambda$ & Notes \\
\midrule
SGD              & $10^{-3}$ & $0$    & $0$       & --- \\
SGD with momentum & $10^{-3}$ & $0.9$  & $0$       & --- \\
Adam             & $10^{-3}$ & $\beta_1{=}0.9$ & $0$ & $\beta_2{=}0.999$ (Keras default) \\
AdamW            & $10^{-3}$ & $\beta_1{=}0.9$ & $10^{-4}$ & Decoupled weight decay (default) \\
Muon             & $10^{-3}$ & ---    & $10^{-4}$ & Output dense and embeddings excluded \\
\bottomrule
\end{tabular}
\end{table}

\paragraph{Reproducibility.}
For each run, Python, NumPy, and TensorFlow random seeds are set to a
single integer, so initialization, minibatch order, and all stochastic
operations are jointly determined by it. 

\section{Supplementary Tables}\label{app:tables}

Tables~\ref{tab:ckpt_low} and~\ref{tab:ckpt_high} report validation accuracy at three training checkpoints for the two extreme initialization
scales, $\sigma_w = 0.1$  and $\sigma_w = 2.5$.
The \emph{repair gap}
$\Delta_{\mathrm{repair}} \coloneqq
\mathrm{ValAcc}_{\tau_{\mathrm{best}}} -
\mathrm{ValAcc}_{\tau_{\mathrm{interp}}}$
quantifies how much validation accuracy changes between the
interpolation epoch and the best-validation-loss epoch.
Its sign, however, is largely determined by the \emph{ordering} of
$\tau_{\mathrm{best}}$ and $\tau_{\mathrm{interp}}$: because validation
accuracy generally increases during training, configurations where
$\tau_{\mathrm{best}} < \tau_{\mathrm{interp}}$ (i.e.\ the lowest
validation loss occurs before the model memorizes the training set)
yield a mechanically negative $\Delta_{\mathrm{repair}}$.
This is the dominant regime for adaptive optimizers at small batch sizes,
where the model achieves its best-calibrated predictions early
($\tau_{\mathrm{best}} \approx 5$--$7$) but does not interpolate until
epoch~30--40; accuracy continues to improve even as cross-entropy loss
rises.
In contrast, at large batch sizes ($b \geq 128$), adaptive optimizers
have $\tau_{\mathrm{best}} \gg \tau_{\mathrm{interp}}$, and the
positive $\Delta_{\mathrm{repair}}$ reflects genuine post-interpolation
restructuring that benefits generalization.
SGD exhibits near-zero $\Delta_{\mathrm{repair}}$ at both extremes:
its validation accuracy is essentially flat between $\taubest$ and
$\tauinterp$, even when both grow with $\sigw$, consistent with the
stopping-time stagnation discussed in \S\ref{sec:memory}.
% ---------------------------------------------------------------------------
% Table A1: sigma_w = 0.1 (near Kaiming)
% ---------------------------------------------------------------------------
\begin{table}[h]
\centering
\caption{Validation accuracy at three checkpoints for $\sigma_w = 0.1$. Values are mean $\pm$ std over $n=10$ seeds.}
\label{tab:ckpt_low}
\resizebox{\textwidth}{!}{
\begin{tabular}{cl|cc|ccc|c}
\toprule
$b$ & Optimizer & $\tau_{\mathrm{interp}}$ & $\tau_{\mathrm{best}}$ & Val Acc$_{\tau_{\mathrm{interp}}}$ & Val Acc$_{\tau_{\mathrm{best}}}$ & Val Acc$_{300}$ & $\Delta_{\mathrm{repair}}$ \\
\midrule
16 & SGD & $8.8 \pm 0.4$ & $10.5 \pm 0.7$ & $0.862 \pm 0.019$ & $0.886 \pm 0.003$ & $0.890 \pm 0.002$ & $+0.024 \pm 0.019$ \\
16 & SGD$~(\beta = 0.9)$ & $12.0 \pm 0.0$ & $14.4 \pm 2.0$ & $0.870 \pm 0.008$ & $0.881 \pm 0.004$ & $0.889 \pm 0.002$ & $+0.012 \pm 0.009$ \\
16 & Adam & $31.8 \pm 2.0$ & $6.0 \pm 1.3$ & $0.851 \pm 0.009$ & $0.831 \pm 0.015$ & $0.879 \pm 0.002$ & $-0.020 \pm 0.017$ \\
16 & AdamW & $32.1 \pm 1.7$ & $5.2 \pm 1.0$ & $0.844 \pm 0.017$ & $0.825 \pm 0.012$ & $0.880 \pm 0.002$ & $-0.018 \pm 0.021$ \\
16 & Muon & $32.6 \pm 1.9$ & $6.6 \pm 1.2$ & $0.849 \pm 0.008$ & $0.836 \pm 0.009$ & $0.879 \pm 0.003$ & $-0.013 \pm 0.011$ \\
\midrule
32 & SGD & $9.0 \pm 0.0$ & $11.4 \pm 1.2$ & $0.868 \pm 0.007$ & $0.882 \pm 0.003$ & $0.883 \pm 0.002$ & $+0.013 \pm 0.008$ \\
32 & SGD$~(\beta = 0.9)$ & $10.0 \pm 0.0$ & $11.4 \pm 1.1$ & $0.874 \pm 0.005$ & $0.882 \pm 0.004$ & $0.889 \pm 0.002$ & $+0.009 \pm 0.008$ \\
32 & Adam & $28.9 \pm 2.7$ & $5.8 \pm 1.0$ & $0.849 \pm 0.008$ & $0.825 \pm 0.012$ & $0.883 \pm 0.001$ & $-0.024 \pm 0.017$ \\
32 & AdamW & $28.5 \pm 2.6$ & $5.4 \pm 1.0$ & $0.849 \pm 0.010$ & $0.819 \pm 0.018$ & $0.882 \pm 0.002$ & $-0.030 \pm 0.023$ \\
32 & Muon & $29.8 \pm 1.6$ & $5.7 \pm 1.3$ & $0.850 \pm 0.012$ & $0.827 \pm 0.011$ & $0.883 \pm 0.003$ & $-0.023 \pm 0.016$ \\
\midrule
64 & SGD & $9.1 \pm 0.3$ & $11.7 \pm 1.3$ & $0.815 \pm 0.024$ & $0.869 \pm 0.003$ & $0.869 \pm 0.003$ & $+0.055 \pm 0.025$ \\
64 & SGD$~(\beta = 0.9)$ & $9.0 \pm 0.0$ & $10.9 \pm 0.9$ & $0.868 \pm 0.005$ & $0.885 \pm 0.002$ & $0.888 \pm 0.001$ & $+0.016 \pm 0.004$ \\
64 & Adam & $23.3 \pm 3.1$ & $5.8 \pm 1.6$ & $0.843 \pm 0.013$ & $0.821 \pm 0.018$ & $0.884 \pm 0.002$ & $-0.022 \pm 0.018$ \\
64 & AdamW & $25.4 \pm 4.7$ & $6.4 \pm 3.7$ & $0.851 \pm 0.010$ & $0.817 \pm 0.025$ & $0.884 \pm 0.002$ & $-0.034 \pm 0.024$ \\
64 & Muon & $24.0 \pm 3.9$ & $7.7 \pm 4.6$ & $0.830 \pm 0.025$ & $0.825 \pm 0.023$ & $0.883 \pm 0.002$ & $-0.005 \pm 0.015$ \\
\midrule
128 & SGD & $11.0 \pm 0.0$ & $14.5 \pm 2.0$ & $0.817 \pm 0.036$ & $0.856 \pm 0.004$ & $0.856 \pm 0.002$ & $+0.039 \pm 0.037$ \\
128 & SGD$~(\beta = 0.9)$ & $8.0 \pm 0.0$ & $10.0 \pm 0.5$ & $0.847 \pm 0.018$ & $0.884 \pm 0.002$ & $0.886 \pm 0.003$ & $+0.037 \pm 0.018$ \\
128 & Adam & $16.7 \pm 2.7$ & $34.6 \pm 22.1$ & $0.833 \pm 0.041$ & $0.872 \pm 0.022$ & $0.885 \pm 0.003$ & $+0.038 \pm 0.035$ \\
128 & AdamW & $17.9 \pm 2.6$ & $34.6 \pm 23.8$ & $0.848 \pm 0.026$ & $0.874 \pm 0.021$ & $0.886 \pm 0.002$ & $+0.026 \pm 0.037$ \\
128 & Muon & $17.1 \pm 3.7$ & $32.6 \pm 23.0$ & $0.844 \pm 0.020$ & $0.867 \pm 0.022$ & $0.886 \pm 0.003$ & $+0.023 \pm 0.021$ \\
\midrule
256 & SGD & $14.1 \pm 0.3$ & $21.1 \pm 2.9$ & $0.744 \pm 0.064$ & $0.841 \pm 0.003$ & $0.841 \pm 0.003$ & $+0.097 \pm 0.063$ \\
256 & SGD$~(\beta = 0.9)$ & $8.4 \pm 0.5$ & $10.7 \pm 0.5$ & $0.844 \pm 0.018$ & $0.878 \pm 0.001$ & $0.880 \pm 0.003$ & $+0.034 \pm 0.018$ \\
256 & Adam & $11.5 \pm 0.7$ & $40.5 \pm 29.5$ & $0.846 \pm 0.014$ & $0.882 \pm 0.009$ & $0.885 \pm 0.002$ & $+0.036 \pm 0.015$ \\
256 & AdamW & $11.5 \pm 1.3$ & $53.5 \pm 30.3$ & $0.843 \pm 0.020$ & $0.878 \pm 0.007$ & $0.884 \pm 0.001$ & $+0.034 \pm 0.023$ \\
256 & Muon & $11.6 \pm 1.0$ & $56.1 \pm 32.9$ & $0.837 \pm 0.023$ & $0.882 \pm 0.005$ & $0.885 \pm 0.003$ & $+0.045 \pm 0.023$ \\
\bottomrule
\end{tabular}}
\end{table}

% ---------------------------------------------------------------------------
% Table A2: sigma_w = 2.5 (large initialization)
% ---------------------------------------------------------------------------
\begin{table}[h]
\centering
\caption{Validation accuracy at three checkpoints for $\sigma_w = 2.5$ (large initialization). At this extreme, SGD's test accuracy collapses to ${\sim}58\%$ at $b=128$ while adaptive optimizers retain ${\sim}85\%$; the contrast with Table~\ref{tab:ckpt_low} directly quantifies initialization bias. Values are mean $\pm$ std over $n=10$ seeds.}
\label{tab:ckpt_high}
\resizebox{\textwidth}{!}{
\begin{tabular}{cl|cc|ccc|c}
\toprule
$b$ & Optimizer & $\tau_{\mathrm{interp}}$ & $\tau_{\mathrm{best}}$ & Val Acc$_{\tau_{\mathrm{interp}}}$ & Val Acc$_{\tau_{\mathrm{best}}}$ & Val Acc$_{300}$ & $\Delta_{\mathrm{repair}}$ \\
\midrule
16 & SGD & $33.9 \pm 1.3$ & $15.0 \pm 2.3$ & $0.679 \pm 0.003$ & $0.660 \pm 0.005$ & $0.699 \pm 0.003$ & $-0.019 \pm 0.005$ \\
16 & SGD$~(\beta = 0.9)$ & $10.9 \pm 0.3$ & $12.0 \pm 2.5$ & $0.779 \pm 0.004$ & $0.781 \pm 0.011$ & $0.794 \pm 0.002$ & $+0.003 \pm 0.010$ \\
16 & Adam & $39.0 \pm 2.1$ & $6.4 \pm 1.9$ & $0.848 \pm 0.012$ & $0.832 \pm 0.011$ & $0.886 \pm 0.001$ & $-0.016 \pm 0.018$ \\
16 & AdamW & $39.2 \pm 1.4$ & $6.5 \pm 1.9$ & $0.853 \pm 0.011$ & $0.832 \pm 0.016$ & $0.886 \pm 0.002$ & $-0.021 \pm 0.019$ \\
16 & Muon & $39.8 \pm 2.5$ & $6.4 \pm 1.3$ & $0.853 \pm 0.010$ & $0.830 \pm 0.015$ & $0.886 \pm 0.002$ & $-0.023 \pm 0.022$ \\
\midrule
32 & SGD & $39.0 \pm 1.2$ & $20.8 \pm 3.0$ & $0.644 \pm 0.005$ & $0.628 \pm 0.007$ & $0.665 \pm 0.005$ & $-0.016 \pm 0.004$ \\
32 & SGD$~(\beta = 0.9)$ & $10.1 \pm 0.3$ & $9.4 \pm 1.6$ & $0.729 \pm 0.007$ & $0.727 \pm 0.008$ & $0.749 \pm 0.003$ & $-0.002 \pm 0.010$ \\
32 & Adam & $40.8 \pm 2.9$ & $7.2 \pm 3.2$ & $0.837 \pm 0.014$ & $0.812 \pm 0.021$ & $0.883 \pm 0.002$ & $-0.025 \pm 0.031$ \\
32 & AdamW & $37.7 \pm 3.5$ & $7.2 \pm 2.3$ & $0.843 \pm 0.011$ & $0.818 \pm 0.018$ & $0.883 \pm 0.003$ & $-0.025 \pm 0.027$ \\
32 & Muon & $40.0 \pm 3.4$ & $6.7 \pm 2.1$ & $0.846 \pm 0.013$ & $0.811 \pm 0.018$ & $0.883 \pm 0.002$ & $-0.035 \pm 0.021$ \\
\midrule
64 & SGD & $51.2 \pm 1.5$ & $35.3 \pm 3.3$ & $0.614 \pm 0.005$ & $0.605 \pm 0.006$ & $0.637 \pm 0.006$ & $-0.009 \pm 0.003$ \\
64 & SGD$~(\beta = 0.9)$ & $10.2 \pm 0.4$ & $8.3 \pm 1.6$ & $0.680 \pm 0.008$ & $0.676 \pm 0.009$ & $0.706 \pm 0.004$ & $-0.004 \pm 0.009$ \\
64 & Adam & $31.5 \pm 4.4$ & $7.9 \pm 3.4$ & $0.833 \pm 0.012$ & $0.799 \pm 0.021$ & $0.876 \pm 0.004$ & $-0.034 \pm 0.025$ \\
64 & AdamW & $32.3 \pm 4.6$ & $6.9 \pm 3.3$ & $0.826 \pm 0.016$ & $0.795 \pm 0.023$ & $0.877 \pm 0.002$ & $-0.031 \pm 0.030$ \\
64 & Muon & $30.9 \pm 5.7$ & $7.4 \pm 3.4$ & $0.828 \pm 0.014$ & $0.795 \pm 0.022$ & $0.876 \pm 0.002$ & $-0.033 \pm 0.030$ \\
\midrule
128 & SGD & $78.6 \pm 3.3$ & $59.6 \pm 4.9$ & $0.595 \pm 0.006$ & $0.587 \pm 0.006$ & $0.612 \pm 0.006$ & $-0.007 \pm 0.003$ \\
128 & SGD$~(\beta = 0.9)$ & $12.1 \pm 0.3$ & $8.3 \pm 1.6$ & $0.640 \pm 0.004$ & $0.630 \pm 0.009$ & $0.671 \pm 0.005$ & $-0.011 \pm 0.005$ \\
128 & Adam & $15.3 \pm 3.1$ & $55.6 \pm 24.2$ & $0.799 \pm 0.024$ & $0.851 \pm 0.004$ & $0.863 \pm 0.004$ & $+0.051 \pm 0.023$ \\
128 & AdamW & $14.6 \pm 1.8$ & $68.6 \pm 30.2$ & $0.808 \pm 0.023$ & $0.852 \pm 0.004$ & $0.860 \pm 0.008$ & $+0.045 \pm 0.025$ \\
128 & Muon & $17.1 \pm 2.8$ & $69.5 \pm 30.8$ & $0.812 \pm 0.028$ & $0.850 \pm 0.003$ & $0.859 \pm 0.005$ & $+0.038 \pm 0.029$ \\
\midrule
256 & SGD & $153.3 \pm 8.7$ & $124.8 \pm 15.0$ & $0.582 \pm 0.004$ & $0.578 \pm 0.006$ & $0.587 \pm 0.004$ & $-0.004 \pm 0.003$ \\
256 & SGD$~(\beta = 0.9)$ & $16.5 \pm 0.5$ & $12.9 \pm 2.1$ & $0.608 \pm 0.005$ & $0.601 \pm 0.007$ & $0.642 \pm 0.005$ & $-0.007 \pm 0.007$ \\
256 & Adam & $8.0 \pm 0.0$ & $80.1 \pm 21.8$ & $0.787 \pm 0.007$ & $0.840 \pm 0.010$ & $0.861 \pm 0.006$ & $+0.053 \pm 0.009$ \\
256 & AdamW & $8.1 \pm 0.3$ & $94.0 \pm 22.4$ & $0.783 \pm 0.011$ & $0.844 \pm 0.008$ & $0.861 \pm 0.007$ & $+0.061 \pm 0.012$ \\
256 & Muon & $8.0 \pm 0.0$ & $89.3 \pm 13.5$ & $0.786 \pm 0.008$ & $0.848 \pm 0.010$ & $0.865 \pm 0.004$ & $+0.062 \pm 0.008$ \\
\bottomrule
\end{tabular}}
\end{table}

\clearpage
\section{Learning-Rate Ablation}
\begin{table}[h]
\centering
\caption{%
  \textbf{Full LR $\times$ L2 sweep for SGD on ResNet-9} (no momentum).
  Test accuracy (\%) at the best-validation-loss epoch for each
  $(\eta, \lambda_{\mathrm{L2}}, \sigma_w)$ configuration.
  The ``Spread'' column reports
  $\max_{\sigma_w}\mathrm{acc} - \min_{\sigma_w}\mathrm{acc}$
  (in percentage points);
  ``Train'' is the mean training accuracy across all $\sigma_w$ values.
  Entries with $\pm$ show mean $\pm$ std over $n=10$ seeds;
  entries marked $^\dagger$ are from a single representative seed (seed~42).
  \textbf{Bottom rows}: Adam baseline and SGD long training
  (5000 ep, constant LR) for comparison.
}
\label{tab:lr_l2_sweep}
\footnotesize
\setlength{\tabcolsep}{2.5pt}

\vspace{4pt}
\resizebox{\textwidth}{!}{
\begin{tabular}{lll rrrrr rr}
\toprule
 & & & \multicolumn{5}{c}{Test Accuracy (\%) at $\sigma_w =$} & & \\
\cmidrule(lr){4-8}
$\eta$ & $\lambda_{\mathrm{L2}}$ & Sched. & $0.1$ & $0.5$ & $1.0$ & $1.8$ & $2.5$ & Spread & Train \\
\midrule
\multicolumn{10}{c}{\textbf{Batch size $b=16$}} \\
\midrule
  $10^{-3}$ & $0$ & cos & $88.1{\pm}0.3$ & $80.9{\pm}0.4$ & $74.3{\pm}0.6$ & $69.1{\pm}0.9$ & $65.4{\pm}0.8$ & 22.8 & 99.6 \\
  $10^{-3}$ & $10^{-3}$ & cos & 88.1 & 82.5 & 77.9 & 74.3 & 72.1 & 16.0 & 100.0$^\dagger$ \\
  $10^{-3}$ & $5{\times}10^{-3}$ & cos & 88.7 & 86.4 & 85.5 & 82.7 & 82.1 & 6.6 & 100.0$^\dagger$ \\
  $10^{-3}$ & $10^{-2}$ & cos & 88.4 & 87.1 & 86.1 & 85.6 & 86.2 & 2.8 & 100.0$^\dagger$ \\
\addlinespace[2pt]
  $10^{-2}$ & $0$ & cos & 87.6 & 87.6 & 85.0 & 80.9 & 77.9 & 9.7 & 100.0$^\dagger$ \\
  $10^{-2}$ & $10^{-3}$ & cos & 90.6 & 90.7 & 90.5 & 90.5 & 90.3 & 0.4 & 100.0$^\dagger$ \\
  $10^{-2}$ & $5{\times}10^{-3}$ & cos & 90.4 & 90.6 & 90.6 & 90.3 & 90.1 & 0.5 & 100.0$^\dagger$ \\
  $10^{-2}$ & $10^{-2}$ & cos & $89.8{\pm}0.2$ & $89.9{\pm}0.2$ & $89.9{\pm}0.2$ & $89.6{\pm}0.2$ & $89.7{\pm}0.4$ & 0.3 & 100.0 \\
\addlinespace[2pt]
  $10^{-1}$ & $0$ & cos & 80.3 & 81.0 & 81.1 & 82.0 & 77.5 & 4.5 & 90.8$^\dagger$ \\
  $10^{-1}$ & $10^{-3}$ & cos & 90.4 & 90.7 & 90.5 & 90.8 & 90.8 & 0.5 & 100.0$^\dagger$ \\
  $10^{-1}$ & $5{\times}10^{-3}$ & cos & 89.9 & 89.6 & 89.4 & 90.1 & 89.9 & 0.7 & 100.0$^\dagger$ \\
  $10^{-1}$ & $10^{-2}$ & cos & 88.1 & 88.9 & 88.2 & 88.1 & 88.5 & 0.8 & 100.0$^\dagger$ \\
\midrule
  \multicolumn{2}{l}{Adam $\eta{=}10^{-3}$} & cos & $82.6{\pm}1.5$ & $84.1{\pm}1.5$ & $83.1{\pm}1.6$ & $83.5{\pm}1.4$ & $82.8{\pm}1.0$ & 1.5 & 95.3 \\
  \multicolumn{2}{l}{SGD $\eta{=}10^{-3}$ (5k ep)} & const & 88.0 & 80.5 & 75.2 & 69.6 & 66.1 & 21.9 & 99.9$^\dagger$ \\
\bottomrule
\end{tabular}}
\vspace{6pt}

\resizebox{\textwidth}{!}{
\begin{tabular}{lll rrrrr rr}
\toprule
 & & & \multicolumn{5}{c}{Test Accuracy (\%) at $\sigma_w =$} & & \\
\cmidrule(lr){4-8}
$\eta$ & $\lambda_{\mathrm{L2}}$ & Sched. & $0.1$ & $0.5$ & $1.0$ & $1.8$ & $2.5$ & Spread & Train \\
\midrule
\multicolumn{10}{c}{\textbf{Batch size $b=128$}} \\
\midrule
  $10^{-3}$ & $0$ & cos & $85.0{\pm}0.3$ & $72.7{\pm}0.7$ & $67.5{\pm}0.5$ & $62.0{\pm}0.5$ & $58.6{\pm}0.7$ & 26.5 & 99.8 \\
  $10^{-3}$ & $10^{-3}$ & cos & 85.8 & 73.7 & 69.9 & 65.0 & 62.0 & 23.8 & 100.0$^\dagger$ \\
  $10^{-3}$ & $5{\times}10^{-3}$ & cos & 85.1 & 74.9 & 71.4 & 66.9 & 64.0 & 21.1 & 100.0$^\dagger$ \\
  $10^{-3}$ & $10^{-2}$ & cos & 85.3 & 75.0 & 72.5 & 69.2 & 66.8 & 18.5 & 100.0$^\dagger$ \\
\addlinespace[2pt]
  $10^{-2}$ & $0$ & cos & 88.1 & 82.8 & 75.6 & 67.8 & 64.3 & 23.8 & 99.6$^\dagger$ \\
  $10^{-2}$ & $10^{-3}$ & cos & 88.4 & 85.0 & 78.5 & 73.5 & 71.1 & 17.4 & 100.0$^\dagger$ \\
  $10^{-2}$ & $5{\times}10^{-3}$ & cos & 89.1 & 88.8 & 88.2 & 88.2 & 88.2 & 0.9 & 100.0$^\dagger$ \\
  $10^{-2}$ & $10^{-2}$ & cos & $88.6{\pm}0.3$ & $88.7{\pm}0.1$ & $88.4{\pm}0.3$ & $88.4{\pm}0.3$ & $88.4{\pm}0.1$ & 0.3 & 100.0 \\
\addlinespace[2pt]
  $10^{-1}$ & $0$ & cos & $84.1{\pm}1.9$ & $85.3{\pm}1.2$ & $84.8{\pm}0.8$ & $81.7{\pm}1.4$ & $77.2{\pm}2.1$ & 8.1 & 98.8 \\
  $10^{-1}$ & $10^{-3}$ & cos & 89.0 & 89.3 & 89.5 & 89.5 & 88.6 & 0.9 & 100.0$^\dagger$ \\
  $10^{-1}$ & $5{\times}10^{-3}$ & cos & 89.1 & 89.0 & 89.1 & 89.3 & 89.1 & 0.3 & 100.0$^\dagger$ \\
  $10^{-1}$ & $10^{-2}$ & cos & 88.5 & 88.6 & 88.5 & 88.8 & 88.3 & 0.6 & 100.0$^\dagger$ \\
\midrule
  \multicolumn{2}{l}{Adam $\eta{=}10^{-3}$} & cos & $86.9{\pm}2.3$ & $87.9{\pm}2.4$ & $87.6{\pm}1.4$ & $86.1{\pm}0.5$ & $84.5{\pm}0.5$ & 3.5 & 99.7 \\
  \multicolumn{2}{l}{SGD $\eta{=}10^{-3}$ (5k ep)} & const & 85.1 & 73.9 & 67.6 & 62.0 & 59.3 & 25.9 & 99.9$^\dagger$ \\
\bottomrule
\end{tabular}}

\end{table}
\label{app:ablation}

\clearpage

\newpage\section{Depth Comparison: Full Results}\label{app:depth}

Figures~\ref{fig:depth_stress} and~\ref{fig:depth_train} display test and
train accuracy as a function of initialization scale~$\sigma_w$ for four architectures: ResNet-9 (6.58M parameters), R9-AvgPool (6.58M), ResNet-56 (0.86M), and ResNet-110 (1.74M).  Each panel fixes one optimizer (SGD or Adam) and one batch size ($b=16$ or $b=128$); lines show the mean over seeds and shaded regions indicate the 10th--90th percentile range.

Table~\ref{tab:depth_comparison} reports the numerical values plotted in Figure~\ref{fig:depth_stress}, together with train accuracy at the final epoch and the spread (difference between the best and worst mean test
accuracy across~$\sigma_w$ values).  The lower part of the table records the interpolation epoch~$\tau_\mathrm{interp}$ (first epoch where
train accuracy $\geq 99.5\%$) and the best-validation-loss epoch for each configuration.

\begin{figure*}[h]
\centering
\includegraphics[width=\textwidth]{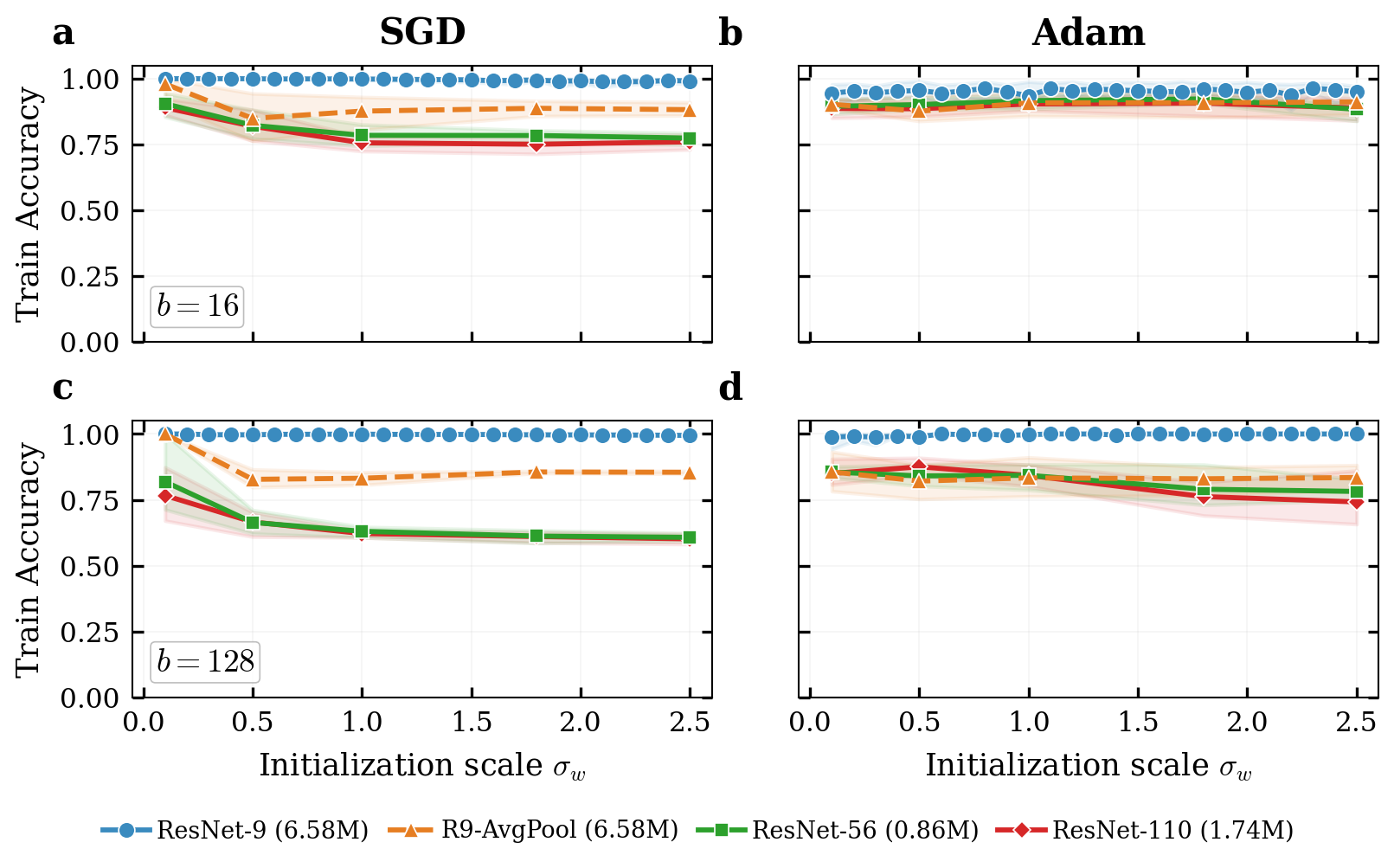}
\caption{Train accuracy vs.\ initialization scale~$\sigma_w$ (same
layout as Figure~\ref{fig:depth_stress}).  Under SGD at $b=128$,
ResNet-56 and ResNet-110 fail to interpolate at large~$\sigma_w$,
while ResNet-9 maintains $>99\%$ train accuracy throughout.}
\label{fig:depth_train}
\end{figure*}

\begin{table*}[t]
\centering
\caption{Depth comparison: training and generalization metrics across architectures. All runs use $\eta{=}10^{-3}$ with cosine decay, no weight decay, 300 epochs. Values are mean $\pm$ std over 10 seeds. $\tau_\mathrm{interp}$: first epoch where train accuracy $\geq 99.5\%$; ``---'' indicates interpolation not reached. Spread $= \max_{\sigma_w} \overline{\mathrm{test\_acc}} - \min_{\sigma_w} \overline{\mathrm{test\_acc}}$.}
\label{tab:depth_comparison}
\footnotesize
\setlength{\tabcolsep}{2.5pt}
\vspace{4pt}
\resizebox{\textwidth}{!}{%
\begin{tabular}{ll rr rr rr rr rr  r}
\toprule
\multicolumn{2}{l}{\textbf{Batch size $b = 16$}} & \multicolumn{2}{c}{$\sigma_w{=}0.1$} & \multicolumn{2}{c}{$\sigma_w{=}0.5$} & \multicolumn{2}{c}{$\sigma_w{=}1.0$} & \multicolumn{2}{c}{$\sigma_w{=}1.8$} & \multicolumn{2}{c}{$\sigma_w{=}2.5$} & \\
Architecture & Opt. & Test & Train & Test & Train & Test & Train & Test & Train & Test & Train & Spread \\
\midrule
  ResNet-9 & SGD & $88.1 \pm 0.3$ & $100.0 \pm 0.0$ & $80.9 \pm 0.4$ & $100.0 \pm 0.0$ & $74.3 \pm 0.6$ & $99.9 \pm 0.2$ & $69.1 \pm 0.9$ & $99.5 \pm 1.2$ & $65.4 \pm 0.8$ & $99.1 \pm 0.8$ & 22.8 \\
   & Adam & $82.6 \pm 1.5$ & $94.5 \pm 3.5$ & $84.1 \pm 1.5$ & $95.6 \pm 3.0$ & $83.1 \pm 1.6$ & $93.5 \pm 3.8$ & $83.5 \pm 1.4$ & $96.2 \pm 2.5$ & $82.8 \pm 1.0$ & $94.9 \pm 2.8$ & 1.5 \\
\addlinespace[2pt]
  R9-AvgPool & SGD & $82.3 \pm 1.9$ & $98.0 \pm 4.2$ & $68.2 \pm 1.1$ & $84.9 \pm 7.5$ & $66.6 \pm 1.0$ & $87.7 \pm 5.4$ & $64.2 \pm 0.6$ & $88.8 \pm 2.3$ & $62.6 \pm 0.3$ & $88.2 \pm 2.1$ & 19.8 \\
   & Adam & $77.8 \pm 1.4$ & $90.3 \pm 2.3$ & $77.5 \pm 1.8$ & $87.7 \pm 4.8$ & $79.0 \pm 1.0$ & $90.8 \pm 3.8$ & $79.2 \pm 1.3$ & $90.9 \pm 4.2$ & $78.4 \pm 1.7$ & $91.2 \pm 4.5$ & 1.7 \\
\addlinespace[2pt]
  ResNet-56 & SGD & $79.2 \pm 1.1$ & $90.5 \pm 3.8$ & $71.0 \pm 1.7$ & $82.3 \pm 4.6$ & $66.2 \pm 1.1$ & $78.5 \pm 3.1$ & $63.2 \pm 1.6$ & $78.4 \pm 2.1$ & $61.2 \pm 2.1$ & $77.4 \pm 1.9$ & 18.0 \\
   & Adam & $80.1 \pm 0.9$ & $89.5 \pm 2.6$ & $80.5 \pm 0.8$ & $90.2 \pm 2.2$ & $80.8 \pm 1.0$ & $91.7 \pm 2.0$ & $80.4 \pm 1.0$ & $92.3 \pm 2.1$ & $78.1 \pm 1.2$ & $88.4 \pm 3.1$ & 2.7 \\
\addlinespace[2pt]
  ResNet-110 & SGD & $78.3 \pm 1.0$ & $89.0 \pm 3.0$ & $69.0 \pm 1.8$ & $81.9 \pm 4.5$ & $63.1 \pm 1.2$ & $75.6 \pm 2.2$ & $60.2 \pm 1.2$ & $75.1 \pm 3.2$ & $59.6 \pm 1.2$ & $76.0 \pm 2.4$ & 18.7 \\
   & Adam & $79.8 \pm 1.4$ & $88.7 \pm 3.4$ & $80.4 \pm 1.0$ & $88.6 \pm 2.5$ & $80.7 \pm 1.1$ & $90.4 \pm 2.5$ & $80.2 \pm 1.2$ & $90.7 \pm 3.4$ & $78.6 \pm 2.2$ & $88.9 \pm 4.0$ & 2.1 \\
\addlinespace[2pt]
\bottomrule
\end{tabular}}
\vspace{4pt}
\resizebox{\textwidth}{!}{%
\begin{tabular}{ll rr rr rr rr rr  r}
\toprule
\multicolumn{2}{l}{\textbf{Batch size $b = 128$}} & \multicolumn{2}{c}{$\sigma_w{=}0.1$} & \multicolumn{2}{c}{$\sigma_w{=}0.5$} & \multicolumn{2}{c}{$\sigma_w{=}1.0$} & \multicolumn{2}{c}{$\sigma_w{=}1.8$} & \multicolumn{2}{c}{$\sigma_w{=}2.5$} & \\
Architecture & Opt. & Test & Train & Test & Train & Test & Train & Test & Train & Test & Train & Spread \\
\midrule
  ResNet-9 & SGD & $85.0 \pm 0.3$ & $100.0 \pm 0.0$ & $72.7 \pm 0.7$ & $99.7 \pm 0.5$ & $67.5 \pm 0.5$ & $99.9 \pm 0.1$ & $62.0 \pm 0.5$ & $99.7 \pm 0.2$ & $58.6 \pm 0.7$ & $99.5 \pm 0.2$ & 26.5 \\
   & Adam & $86.9 \pm 2.3$ & $98.8 \pm 2.6$ & $87.9 \pm 2.4$ & $99.0 \pm 3.3$ & $87.6 \pm 1.4$ & $99.7 \pm 0.7$ & $86.1 \pm 0.5$ & $100.0 \pm 0.0$ & $84.5 \pm 0.5$ & $100.0 \pm 0.0$ & 3.5 \\
\addlinespace[2pt]
  R9-AvgPool & SGD & $77.4 \pm 0.4$ & $100.0 \pm 0.0$ & $65.5 \pm 0.8$ & $82.8 \pm 3.1$ & $63.8 \pm 0.8$ & $83.2 \pm 2.5$ & $61.5 \pm 0.6$ & $85.6 \pm 0.6$ & $59.9 \pm 0.6$ & $85.5 \pm 0.5$ & 17.5 \\
   & Adam & $74.5 \pm 3.2$ & $85.6 \pm 6.6$ & $73.7 \pm 2.6$ & $82.1 \pm 5.9$ & $73.8 \pm 2.5$ & $83.3 \pm 6.2$ & $72.0 \pm 2.5$ & $83.0 \pm 5.7$ & $70.3 \pm 1.7$ & $83.5 \pm 5.5$ & 4.3 \\
\addlinespace[2pt]
  ResNet-56 & SGD & $71.3 \pm 4.9$ & $81.9 \pm 11.0$ & $57.4 \pm 1.5$ & $66.4 \pm 3.9$ & $52.4 \pm 1.0$ & $63.0 \pm 1.8$ & $49.0 \pm 1.6$ & $61.2 \pm 2.3$ & $47.1 \pm 1.9$ & $60.7 \pm 1.3$ & 24.2 \\
   & Adam & $76.8 \pm 1.2$ & $85.5 \pm 2.6$ & $75.8 \pm 1.4$ & $84.1 \pm 3.5$ & $73.6 \pm 2.1$ & $84.3 \pm 4.4$ & $69.1 \pm 3.5$ & $79.0 \pm 6.5$ & $67.5 \pm 2.3$ & $78.2 \pm 3.8$ & 9.3 \\
\addlinespace[2pt]
  ResNet-110 & SGD & $67.5 \pm 4.0$ & $76.6 \pm 8.7$ & $56.8 \pm 1.9$ & $66.6 \pm 4.8$ & $52.2 \pm 0.6$ & $62.2 \pm 1.3$ & $48.6 \pm 1.0$ & $61.0 \pm 1.7$ & $47.2 \pm 0.7$ & $60.2 \pm 1.6$ & 20.4 \\
   & Adam & $76.5 \pm 1.6$ & $85.0 \pm 4.3$ & $77.3 \pm 1.3$ & $87.5 \pm 3.2$ & $74.0 \pm 2.0$ & $84.3 \pm 3.4$ & $66.8 \pm 4.7$ & $76.2 \pm 5.5$ & $63.4 \pm 5.9$ & $74.2 \pm 8.3$ & 13.9 \\
\addlinespace[2pt]
\bottomrule
\end{tabular}}

\vspace{8pt}
\resizebox{\textwidth}{!}{%
\begin{tabular}{ll rr rr rr rr rr}
\toprule
& & \multicolumn{2}{c}{$\sigma_w{=}0.1$} & \multicolumn{2}{c}{$\sigma_w{=}0.5$} & \multicolumn{2}{c}{$\sigma_w{=}1.0$} & \multicolumn{2}{c}{$\sigma_w{=}1.8$} & \multicolumn{2}{c}{$\sigma_w{=}2.5$} \\
Architecture & Opt. & $\tau_\mathrm{interp}$ & Best ep. & $\tau_\mathrm{interp}$ & Best ep. & $\tau_\mathrm{interp}$ & Best ep. & $\tau_\mathrm{interp}$ & Best ep. & $\tau_\mathrm{interp}$ & Best ep. \\
\midrule
\multicolumn{12}{l}{\textbf{$b = 16$}} \\
\addlinespace[1pt]
  ResNet-9 & SGD & $11.0 \pm 0.7$ & $10.5 \pm 0.7$ & $13.8 \pm 1.4$ & $10.5 \pm 1.4$ & $32.2 \pm 3.2$ & $10.7 \pm 1.3$ & $93.5 \pm 10.2$ & $14.2 \pm 2.4$ & $203.6 \pm 30.4$ & $15.0 \pm 2.3$ \\
   & Adam & $155.2 \pm 4.8$ & $6.0 \pm 1.3$ & $154.6 \pm 9.4$ & $6.6 \pm 2.3$ & $153.1 \pm 10.2$ & $5.4 \pm 2.1$ & $157.6 \pm 10.0$ & $7.3 \pm 1.9$ & $161.5 \pm 7.2$ & $6.4 \pm 1.9$ \\
  R9-AvgPool & SGD & $32.1 \pm 3.5$ & $23.8 \pm 11.2$ & $67.2 \pm 6.6$ & $7.2 \pm 2.1$ & $161.4 \pm 11.5$ & $11.6 \pm 2.1$ & --- & $20.2 \pm 1.5$ & --- & $27.3 \pm 1.9$ \\
   & Adam & $173.2 \pm 7.2$ & $5.8 \pm 0.4$ & $168.6 \pm 10.9$ & $5.1 \pm 1.3$ & $166.4 \pm 4.5$ & $5.0 \pm 1.1$ & $167.5 \pm 6.8$ & $4.6 \pm 1.1$ & $166.2 \pm 9.4$ & $4.7 \pm 1.1$ \\
  ResNet-56 & SGD & $48.1 \pm 3.2$ & $13.3 \pm 4.5$ & $67.2 \pm 1.5$ & $12.1 \pm 2.8$ & $115.7 \pm 3.7$ & $20.3 \pm 2.4$ & --- & $41.7 \pm 4.7$ & --- & $66.5 \pm 8.4$ \\
   & Adam & $78.1 \pm 3.6$ & $10.0 \pm 1.8$ & $78.7 \pm 3.4$ & $10.1 \pm 1.9$ & $77.4 \pm 4.2$ & $10.5 \pm 1.7$ & $78.4 \pm 6.3$ & $10.8 \pm 1.5$ & $81.0 \pm 5.2$ & $8.8 \pm 2.3$ \\
  ResNet-110 & SGD & $43.7 \pm 1.8$ & $10.9 \pm 2.2$ & $63.5 \pm 2.8$ & $12.3 \pm 2.1$ & $110.3 \pm 4.3$ & $18.5 \pm 1.6$ & $252.0 \pm 32.5$ & $38.9 \pm 5.0$ & --- & $66.8 \pm 10.5$ \\
   & Adam & $74.2 \pm 3.4$ & $10.1 \pm 2.2$ & $73.6 \pm 2.8$ & $9.3 \pm 1.8$ & $73.5 \pm 4.2$ & $10.2 \pm 2.0$ & $74.7 \pm 3.9$ & $10.2 \pm 2.1$ & $75.9 \pm 3.6$ & $9.6 \pm 1.7$ \\
\addlinespace[3pt]
\multicolumn{12}{l}{\textbf{$b = 128$}} \\
\addlinespace[1pt]
  ResNet-9 & SGD & $12.8 \pm 0.4$ & $14.5 \pm 2.0$ & $32.9 \pm 1.3$ & $22.0 \pm 3.2$ & $61.9 \pm 2.8$ & $39.3 \pm 2.0$ & $117.4 \pm 9.6$ & $51.8 \pm 5.1$ & $222.1 \pm 23.2$ & $59.6 \pm 4.9$ \\
   & Adam & $25.7 \pm 8.4$ & $34.6 \pm 22.1$ & $22.8 \pm 5.4$ & $35.7 \pm 24.0$ & $22.1 \pm 4.3$ & $46.3 \pm 30.2$ & $19.8 \pm 1.8$ & $45.1 \pm 26.0$ & $19.1 \pm 4.2$ & $55.6 \pm 24.2$ \\
  R9-AvgPool & SGD & $18.9 \pm 0.6$ & $20.0 \pm 1.4$ & $62.8 \pm 3.2$ & $26.7 \pm 2.3$ & $169.8 \pm 4.4$ & $54.0 \pm 4.5$ & --- & $118.4 \pm 2.6$ & --- & $224.2 \pm 17.7$ \\
   & Adam & $103.2 \pm 15.2$ & $7.2 \pm 2.1$ & $81.3 \pm 17.8$ & $5.0 \pm 1.6$ & $75.3 \pm 14.0$ & $4.2 \pm 1.5$ & $63.9 \pm 10.4$ & $3.6 \pm 1.0$ & $62.4 \pm 10.2$ & $3.5 \pm 1.4$ \\
  ResNet-56 & SGD & $51.6 \pm 4.8$ & $24.9 \pm 25.7$ & $114.6 \pm 7.2$ & $25.0 \pm 4.3$ & --- & $45.4 \pm 5.0$ & --- & $93.9 \pm 18.7$ & --- & $177.3 \pm 57.8$ \\
   & Adam & $69.7 \pm 3.9$ & $12.1 \pm 1.9$ & $66.0 \pm 3.9$ & $9.9 \pm 2.2$ & $71.2 \pm 7.1$ & $10.9 \pm 3.3$ & $79.0 \pm 6.4$ & $8.2 \pm 2.2$ & $86.5 \pm 8.9$ & $8.1 \pm 2.1$ \\
  ResNet-110 & SGD & $48.2 \pm 2.6$ & $12.4 \pm 5.5$ & $109.0 \pm 7.1$ & $24.3 \pm 3.8$ & $236.7 \pm 9.1$ & $47.4 \pm 6.5$ & --- & $98.5 \pm 9.5$ & --- & $194.0 \pm 46.7$ \\
   & Adam & $70.4 \pm 2.7$ & $13.3 \pm 3.1$ & $69.6 \pm 8.0$ & $13.5 \pm 3.0$ & $74.3 \pm 4.8$ & $10.6 \pm 2.4$ & $83.3 \pm 5.2$ & $8.5 \pm 2.1$ & $94.2 \pm 4.9$ & $9.0 \pm 2.9$ \\
\addlinespace[3pt]
\bottomrule
\end{tabular}}
\end{table*}

\clearpage
\section{Normalization and Data Augmentation}\label{app:norm}

The main experiments use BatchNorm without data augmentation to
isolate the effect of initialization scale from other regularization
mechanisms.  A natural question is whether switching the normalization
layer or adding standard data augmentation eliminates initialization
memory.  To test this, we train ResNet-9 on CIFAR-10 under three
configurations: (i)~BatchNorm with augmentation (random horizontal
flip + random crop with 4-pixel padding), (ii)~LayerNorm with the
same augmentation, and (iii)~the original BatchNorm baseline without
augmentation.  All other hyperparameters are identical to the main
grid ($\eta = 10^{-3}$, cosine decay, no weight decay, 300 epochs,
$\sigma_w \in \{0.1, 0.5, 1.0, 1.8, 2.5\}$).  Results are averaged
over 5~seeds.

Figure~\ref{fig:norm_comparison} and Table~\ref{tab:norm_comparison}
show the results.  Data augmentation substantially reduces initialization
memory under SGD (spread drops from 22.8 to 8.3~pp at $b=16$, and from
26.5 to 18.9~pp at $b=128$ for BatchNorm), but does not eliminate it.
LayerNorm behaves similarly to BatchNorm: the SGD spread is 5.1~pp at
$b=16$ and 15.2~pp at $b=128$.  Under Adam, all three configurations
show small spreads ($\leq 3.5$~pp), confirming that the forgetting property of Adam is robust to the choice of normalization and
augmentation.

These results reinforce the main finding: neither the specific
normalization scheme nor data augmentation is sufficient to erase the initialization memory
under low-learning-rate SGD, though both reduce its magnitude.

\begin{figure*}[b]
\centering
\includegraphics[width=\textwidth]{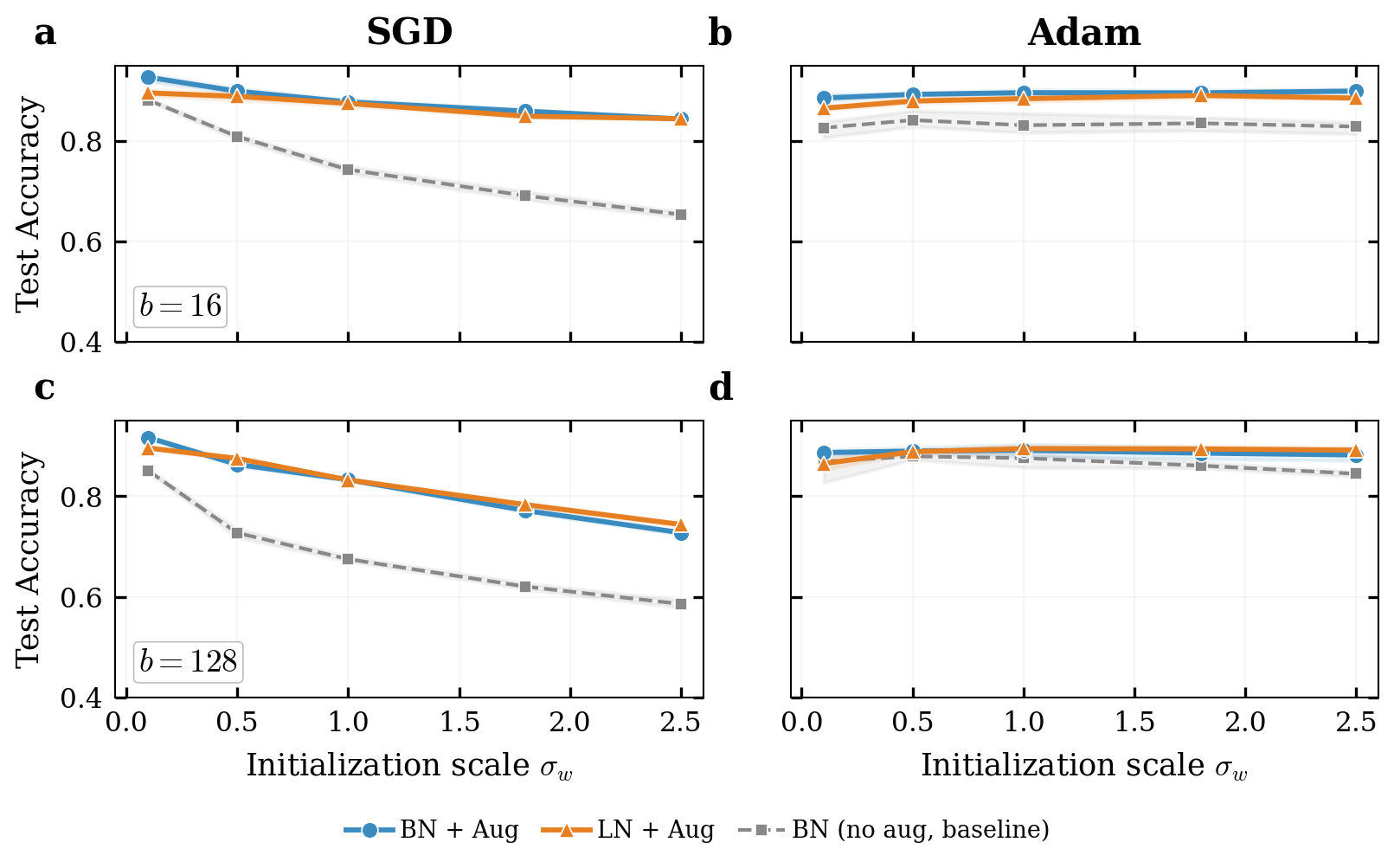}
\caption{Test accuracy vs.\ initialization scale~$\sigma_w$ for
ResNet-9 under BatchNorm + augmentation (blue), LayerNorm +
augmentation (orange), and the BatchNorm baseline without augmentation
(grey dashed).  Lines: mean over seeds; shaded bands: 10th--90th
percentile.}
\label{fig:norm_comparison}
\end{figure*}

\begin{table*}[h]
\centering
\caption{Normalization comparison on ResNet-9 CIFAR-10. ``BN + Aug'' and ``LN + Aug'' use data augmentation (random flip + crop) with BatchNorm and LayerNorm respectively; ``BN (no aug)'' is the main-grid baseline without augmentation. All use $\eta{=}10^{-3}$, cosine decay, no weight decay, 300 epochs. Values: mean $\pm$ std over 5 seeds (aug) or 10 seeds (baseline). Spread $= \max_{\sigma_w} \overline{\mathrm{test\_acc}} - \min_{\sigma_w} \overline{\mathrm{test\_acc}}$ (pp).}
\label{tab:norm_comparison}
\footnotesize
\setlength{\tabcolsep}{2.5pt}
\vspace{4pt}
\resizebox{\textwidth}{!}{%
\begin{tabular}{ll rr rr rr rr rr r}
\toprule
\multicolumn{2}{l}{\textbf{Batch size $b = 16$}} & \multicolumn{2}{c}{$\sigma_w{=}0.1$} & \multicolumn{2}{c}{$\sigma_w{=}0.5$} & \multicolumn{2}{c}{$\sigma_w{=}1.0$} & \multicolumn{2}{c}{$\sigma_w{=}1.8$} & \multicolumn{2}{c}{$\sigma_w{=}2.5$} & \\
Configuration & Opt. & Test & Train & Test & Train & Test & Train & Test & Train & Test & Train & Spread \\
\midrule
  BN + Aug & SGD & $92.7 \pm 0.8$ & $99.7 \pm 0.6$ & $89.9 \pm 0.6$ & $99.2 \pm 0.5$ & $87.8 \pm 0.6$ & $98.3 \pm 0.7$ & $85.9 \pm 0.3$ & $97.6 \pm 0.6$ & $84.4 \pm 0.2$ & $95.5 \pm 0.3$ & 8.3 \\
   & Adam & $88.5 \pm 0.9$ & $94.7 \pm 1.5$ & $89.3 \pm 0.3$ & $94.9 \pm 0.4$ & $89.6 \pm 0.6$ & $95.8 \pm 1.2$ & $89.6 \pm 0.5$ & $95.7 \pm 1.0$ & $89.9 \pm 0.5$ & $96.7 \pm 1.1$ & 1.4 \\
\addlinespace[2pt]
  LN + Aug & SGD & $89.6 \pm 0.7$ & $98.5 \pm 0.8$ & $88.9 \pm 0.7$ & $98.0 \pm 1.2$ & $87.5 \pm 0.5$ & $97.6 \pm 1.3$ & $84.9 \pm 0.4$ & $94.9 \pm 0.9$ & $84.4 \pm 0.1$ & $94.5 \pm 0.5$ & 5.1 \\
   & Adam & $86.5 \pm 0.4$ & $94.6 \pm 0.9$ & $87.9 \pm 0.6$ & $95.2 \pm 1.1$ & $88.4 \pm 0.6$ & $95.5 \pm 1.5$ & $89.0 \pm 0.5$ & $95.9 \pm 1.3$ & $88.5 \pm 0.2$ & $95.0 \pm 1.1$ & 2.5 \\
\addlinespace[2pt]
  BN (no aug) & SGD & $88.1 \pm 0.3$ & $100.0 \pm 0.0$ & $80.9 \pm 0.4$ & $100.0 \pm 0.0$ & $74.3 \pm 0.6$ & $99.9 \pm 0.2$ & $69.1 \pm 0.9$ & $99.5 \pm 1.2$ & $65.4 \pm 0.8$ & $99.1 \pm 0.8$ & 22.8 \\
   & Adam & $82.6 \pm 1.5$ & $94.5 \pm 3.5$ & $84.1 \pm 1.5$ & $95.6 \pm 3.0$ & $83.1 \pm 1.6$ & $93.5 \pm 3.8$ & $83.5 \pm 1.4$ & $96.2 \pm 2.5$ & $82.8 \pm 1.0$ & $94.9 \pm 2.8$ & 1.5 \\
\addlinespace[2pt]
\bottomrule
\end{tabular}}
\vspace{4pt}
\resizebox{\textwidth}{!}{%
\begin{tabular}{ll rr rr rr rr rr r}
\toprule
\multicolumn{2}{l}{\textbf{Batch size $b = 128$}} & \multicolumn{2}{c}{$\sigma_w{=}0.1$} & \multicolumn{2}{c}{$\sigma_w{=}0.5$} & \multicolumn{2}{c}{$\sigma_w{=}1.0$} & \multicolumn{2}{c}{$\sigma_w{=}1.8$} & \multicolumn{2}{c}{$\sigma_w{=}2.5$} & \\
Configuration & Opt. & Test & Train & Test & Train & Test & Train & Test & Train & Test & Train & Spread \\
\midrule
  BN + Aug & SGD & $91.7 \pm 0.2$ & $100.0 \pm 0.0$ & $86.3 \pm 0.4$ & $96.2 \pm 1.0$ & $83.3 \pm 0.3$ & $92.3 \pm 0.1$ & $77.1 \pm 0.4$ & $82.1 \pm 0.3$ & $72.7 \pm 0.4$ & $76.2 \pm 0.4$ & 18.9 \\
   & Adam & $88.7 \pm 0.5$ & $95.2 \pm 0.9$ & $89.0 \pm 0.8$ & $95.2 \pm 1.5$ & $89.1 \pm 1.3$ & $94.9 \pm 2.7$ & $88.6 \pm 1.2$ & $95.1 \pm 2.1$ & $88.2 \pm 1.2$ & $95.2 \pm 1.8$ & 0.9 \\
\addlinespace[2pt]
  LN + Aug & SGD & $89.6 \pm 0.3$ & $99.7 \pm 0.1$ & $87.5 \pm 0.2$ & $97.7 \pm 0.1$ & $83.3 \pm 0.2$ & $89.6 \pm 0.2$ & $78.3 \pm 0.3$ & $80.5 \pm 0.4$ & $74.4 \pm 0.4$ & $74.3 \pm 0.6$ & 15.2 \\
   & Adam & $86.5 \pm 1.1$ & $93.4 \pm 1.6$ & $88.9 \pm 0.3$ & $96.5 \pm 0.8$ & $89.5 \pm 0.6$ & $97.2 \pm 1.1$ & $89.4 \pm 0.6$ & $96.9 \pm 1.3$ & $89.2 \pm 0.8$ & $97.0 \pm 1.4$ & 3.0 \\
\addlinespace[2pt]
  BN (no aug) & SGD & $85.0 \pm 0.3$ & $100.0 \pm 0.0$ & $72.7 \pm 0.7$ & $99.7 \pm 0.5$ & $67.5 \pm 0.5$ & $99.9 \pm 0.1$ & $62.0 \pm 0.5$ & $99.7 \pm 0.2$ & $58.6 \pm 0.7$ & $99.5 \pm 0.2$ & 26.5 \\
   & Adam & $86.9 \pm 2.3$ & $98.8 \pm 2.6$ & $87.9 \pm 2.4$ & $99.0 \pm 3.3$ & $87.6 \pm 1.4$ & $99.7 \pm 0.7$ & $86.1 \pm 0.5$ & $100.0 \pm 0.0$ & $84.5 \pm 0.5$ & $100.0 \pm 0.0$ & 3.5 \\
\addlinespace[2pt]
\bottomrule
\end{tabular}}
\end{table*}
\clearpage

\section{Best-Training Pipeline Comparison}\label{app:best_recipe}

The preceding experiments deliberately use a minimal training
configuration ($\eta{=}10^{-3}$, no momentum, no weight decay, no
augmentation) to isolate the role of initialization scale.  A natural
concern is whether initialization memory persists under a standard,
well-tuned training procedure.  Table~\ref{tab:best_recipe} addresses
this by comparing two configurations at the extreme initialization
scales $\sigma_w \in \{0.1, 2.5\}$: (i)~SGD with lr${=}0.1$,
momentum${=}0.9$, weight decay $=5\times 10^{-4}$, and data
augmentation---the standard ResNet procedure---and (ii)~Adam with
lr${=}10^{-3}$, no weight decay, and the same augmentation.

At $b{=}128$, SGD (best) achieves 94.2\% test accuracy at both
$\sigma_w{=}0.1$ and $2.5$ respectively ($|\Delta|{=}0.0$~pp),
confirming that the standard training procedure erases initialization memory
completely.  Adam also shows a small gap (0.6~pp).  At $b{=}16$, the
SGD procedure exhibits high seed-to-seed variance: with $\eta{=}0.1$ and
only 16 samples per step, the effective per-sample step size is large
enough that some seeds experience a near-divergence in the first epoch
(loss $>50$ vs.\ $\sim$15 for other seeds) and fail to recover fully
within 200~epochs.  This is a training-stability issue---addressable
by learning-rate warmup---not an initialization-memory effect; the gap
between $\sigma_w$ values remains small even for the affected seeds.

\begin{table}[h]
\centering
\caption{\textbf{Best-recipe comparison.}
SGD (standard recipe: $\eta{=}0.1$, momentum${=}0.9$, $L_2$ kernel reg. $\lambda=5\times10^{-4}$)
vs.\ Adam ($\eta{=}10^{-3}$, no WD), both with data augmentation and
cosine decay, 200~epochs on ResNet-9 / CIFAR-10.
Values: mean $\pm$ std over $n{=}10$ seeds.
$\tau_{\mathrm{interp}}$: first epoch with train accuracy $\geq 99.5\%$
(``---'' if not reached within 200 epochs);
format is $\sigma_w{=}0.1$ / $\sigma_w{=}2.5$.
$\tau_{\mathrm{best}}$: best-validation-loss epoch.
$|\Delta|$: absolute difference between $\sigma_w{=}0.1$ and $\sigma_w{=}2.5$ test accuracy.}
\label{tab:best_recipe}
\footnotesize
\setlength{\tabcolsep}{3pt}
\resizebox{\textwidth}{!}{%
\begin{tabular}{ll cc cc cc c}
\toprule
& & \multicolumn{2}{c}{Test Acc.\ (\%)} & \multicolumn{2}{c}{Train Acc.\ (\%)} & \multicolumn{2}{c}{Epochs} & \\
\cmidrule(lr){3-4} \cmidrule(lr){5-6} \cmidrule(lr){7-8}
Optimizer & $b$ & $\sigma_w{=}0.1$ & $\sigma_w{=}2.5$ & $\sigma_w{=}0.1$ & $\sigma_w{=}2.5$ & $\tau_{\mathrm{interp}}$ & $\tau_{\mathrm{best}}$ & $\mathbf{|\Delta|}$ \\
\midrule
SGD (best) & 16
  & $86.5{\pm}6.4$ & $86.7{\pm}3.0$
  & $90.2{\pm}7.7$ & $90.2{\pm}4.2$
  & --- / ---
  & $198{\pm}2$ / $198{\pm}1$
  & $0.2$ \\
 & \textbf{128}
  & $\mathbf{94.2{\pm}0.1}$ & $\mathbf{94.2{\pm}0.2}$
  & $100.0{\pm}0.0$ & $100.0{\pm}0.0$
  & $166{\pm}1$ / $167{\pm}2$
  & $200{\pm}1$ / $200{\pm}1$
  & $\mathbf{0.0}$ \\
 & 256
  & $93.8{\pm}0.5$ & $93.9{\pm}0.3$
  & $100.0{\pm}0.0$ & $100.0{\pm}0.0$
  & $154{\pm}3$ / $153{\pm}2$
  & $198{\pm}2$ / $199{\pm}1$
  & $0.1$ \\
\addlinespace[3pt]
Adam & 16
  & $89.5{\pm}0.7$ & $90.1{\pm}0.5$
  & $100.0{\pm}0.0$ & $100.0{\pm}0.0$
  & $90{\pm}2$ / $90{\pm}3$
  & $17{\pm}3$ / $23{\pm}4$
  & $0.6$ \\
 & 128
  & $88.5{\pm}0.9$ & $89.0{\pm}0.7$
  & $100.0{\pm}0.0$ & $100.0{\pm}0.0$
  & $72{\pm}3$ / $84{\pm}3$
  & $20{\pm}5$ / $32{\pm}8$
  & $0.5$ \\
 & 256
  & $89.1{\pm}1.0$ & $88.6{\pm}0.8$
  & $100.0{\pm}0.0$ & $100.0{\pm}0.0$
  & $68{\pm}4$ / $83{\pm}3$
  & $28{\pm}9$ / $38{\pm}10$
  & $0.5$ \\
\bottomrule
\end{tabular}%
}
\end{table}

\clearpage
\section{Minimal Linear Model for Initialization Memory}
\label{app:toy_model}

This appendix gives the calculations behind the timescale discussion in
Section~\ref{sec:timescales}.  The goal is not a complete theory of
BatchNorm ResNets, but a minimal homogeneous model showing why
gradient-flow-like dynamics can preserve initialization-dependent
quantities and why explicit norm decay, finite steps, stochasticity, and
adaptive preconditioning act on different clocks.

Throughout this appendix, we use \(\phi\) for the loss as a function of
the product represented by the linear network. This avoids overloading
the layer index.  All losses are assumed continuously differentiable, and
all gradients are evaluated at the current product.  In the two-factor
case the product is \(W=UV\).  In the deep-linear case the product is
\[
    F(W_1,\ldots,W_L)=W_LW_{L-1}\cdots W_1 .
\]
All discrete updates below are simultaneous updates of all factors.

For stochastic statements, \(\mathbb E_k[\cdot]\) denotes conditional
expectation given the current iterate.  We write a product-space
minibatch gradient estimate as
\[
    G_k=\bar G_k+\xi_k,
    \qquad
    \mathbb E_k[\xi_k]=0,
    \qquad
    \mathbb E_k\|\xi_k\|_F^2=O(1/b),
\]
where \(\bar G_k\) is the full product-space gradient and \(b\) is the
minibatch size.  In the deep-linear stochastic statements, the layerwise
gradient estimates are induced from the same product-space estimate
\(G_k\) by the chain rule.  This shared \(G_k\) assumption is essential:
the first-order cancellations below need not hold for arbitrary
independent layerwise perturbations.

\paragraph{Gradient flow preserves imbalance.}
Consider a two-layer linear network \(W=UV\), with
\[
    U\in\R^{d_{\mathrm{out}}\times r},
    \qquad
    V\in\R^{r\times d_{\mathrm{in}}},
    \qquad
    \Loss(U,V)=\phi(UV).
\]
Let
\[
    G=\nabla_W\phi(W).
\]
Then
\[
    \nabla_U\Loss=GV^\top,
    \qquad
    \nabla_V\Loss=U^\top G,
\]
so Euclidean gradient flow is
\[
    \dot U=-GV^\top,
    \qquad
    \dot V=-U^\top G.
\]
Define the imbalance
\[
    D(U,V)
    =
    U^\top U - VV^\top
    \in \R^{r\times r}.
\]

\begin{proposition}[Two-factor gradient flow preserves imbalance]
Along Euclidean gradient flow for the two-factor linear network,
\[
    \frac{d}{dt}D(U(t),V(t))=0.
\]
\end{proposition}

\begin{proof}
We compute
\[
\begin{aligned}
    \frac{d}{dt}U^\top U
    &=
    \dot U^\top U+U^\top\dot U  \\
    &=
    (-GV^\top)^\top U+U^\top(-GV^\top) \\
    &=
    -VG^\top U-U^\top G V^\top .
\end{aligned}
\]
Similarly,
\[
\begin{aligned}
    \frac{d}{dt}VV^\top
    &=
    \dot V V^\top+V\dot V^\top  \\
    &=
    (-U^\top G)V^\top+V(-U^\top G)^\top \\
    &=
    -U^\top G V^\top - V G^\top U .
\end{aligned}
\]
The two derivatives are identical, hence their difference is zero:
\[
    \frac{d}{dt}
    \left(
        U^\top U - VV^\top
    \right)
    =
    0.
\]
\end{proof}

Thus gradient flow exactly preserves an initialization-dependent
quantity: \(D(U(t),V(t))=D(U(0),V(0))\).  In this minimal homogeneous
model, initialization memory is a conserved charge.

\paragraph{The conserved imbalance fixes the final norm in the scalar model.}
The preceding invariant is not merely formal: in the scalar case it
directly determines the factor norm of any converged solution.

\begin{proposition}[Scalar two-factor norm is fixed by conserved imbalance]
Consider scalar parameters \(a,b\in\R\) with predictor \(p=ab\).  Suppose
gradient flow converges to a point with product \(p_\star\).  Let
\[
    D_0=a_0^2-b_0^2 .
\]
Since \(D=a^2-b^2\) is conserved, the final squared Euclidean norm
satisfies
\[
    \|(a_\infty,b_\infty)\|_2^2
    =
    a_\infty^2+b_\infty^2
    =
    \sqrt{D_0^2+4p_\star^2}.
\]
Therefore, among solutions with the same product \(p_\star\), the final
factor norm is a monotone function of \(|D_0|\).  The final norm thus
retains initialization memory.
\end{proposition}

\begin{proof}
Let
\[
    x=a_\infty^2,
    \qquad
    y=b_\infty^2.
\]
Conservation of imbalance gives
\[
    x-y=D_0,
\]
and convergence to product \(p_\star\) gives
\[
    xy=(a_\infty b_\infty)^2=p_\star^2.
\]
Therefore
\[
    (x+y)^2
    =
    (x-y)^2+4xy
    =
    D_0^2+4p_\star^2.
\]
Since \(x+y\ge0\), we obtain
\[
    x+y=\sqrt{D_0^2+4p_\star^2}.
\]
Substituting back \(x=a_\infty^2\) and \(y=b_\infty^2\) gives the claim.
\end{proof}

\paragraph{Continuous-time \(L_2\) decay contracts imbalance.}
Now add coupled continuous-time weight decay:
\[
    \dot U=-GV^\top-\lambda U,
    \qquad
    \dot V=-U^\top G-\lambda V,
    \qquad
    \lambda\ge0.
\]

\begin{proposition}[Continuous-time \(L_2\) decay erases imbalance]
Under the dynamics above,
\[
    \frac{d}{dt}D(t)=-2\lambda D(t),
    \qquad
    D(t)=e^{-2\lambda t}D(0).
\]
\end{proposition}

\begin{proof}
The gradient-dependent terms are the same as in the previous
calculation and still cancel in the difference.  The decay terms give
\[
\begin{aligned}
    \frac{d}{dt}U^\top U
    &=
    \text{gradient terms}
    -2\lambda U^\top U,
    \\
    \frac{d}{dt}VV^\top
    &=
    \text{same gradient terms}
    -2\lambda VV^\top .
\end{aligned}
\]
Subtracting yields
\[
    \dot D(t)
    =
    -2\lambda
    \left(
        U^\top U - VV^\top
    \right)
    =
    -2\lambda D(t).
\]
Solving this linear differential equation gives
\[
    D(t)=e^{-2\lambda t}D(0).
\]
\end{proof}

Thus explicit norm control contracts the imbalance on the continuous-time
clock \(\lambda t\).  In discrete training, the gradient-flow time
corresponding to a learning-rate schedule \((\eta_k)\) is
approximately \(\sum_k \eta_k\), so the natural \(L_2\) clock is
\[
    \mathcal T_{L_2}
    =
    \lambda \sum_{k<K}\eta_k
    =
    K\eta\lambda
    \quad\text{for constant \(\eta\)}.
\]

\paragraph{Coupled discrete weight decay.}
The corresponding discrete calculation makes the same clock explicit.
Consider the two-factor update
\[
    U_{k+1}
    =
    (1-\eta_k\lambda)U_k-\eta_kG_kV_k^\top,
    \qquad
    V_{k+1}
    =
    (1-\eta_k\lambda)V_k-\eta_kU_k^\top G_k.
\]
Let
\[
    c_k=1-\eta_k\lambda.
\]

\begin{proposition}[Discrete weight decay contracts imbalance up to finite-step leakage]
For the coupled discrete update above,
\[
    D_{k+1}
    =
    c_k^2D_k
    +
    \eta_k^2
    \left(
        V_kG_k^\top G_kV_k^\top
        -
        U_k^\top G_kG_k^\top U_k
    \right).
\]
\end{proposition}

\begin{proof}
Expand
\[
    U_{k+1}^\top U_{k+1}
    =
    c_k^2U_k^\top U_k
    -
    c_k\eta_k
    \left(
        V_kG_k^\top U_k
        +
        U_k^\top G_kV_k^\top
    \right)
    +
    \eta_k^2V_kG_k^\top G_kV_k^\top .
\]
Similarly,
\[
    V_{k+1}V_{k+1}^\top
    =
    c_k^2V_kV_k^\top
    -
    c_k\eta_k
    \left(
        U_k^\top G_kV_k^\top
        +
        V_kG_k^\top U_k
    \right)
    +
    \eta_k^2U_k^\top G_kG_k^\top U_k .
\]
The first-order terms are identical and cancel after subtraction.  Hence
\[
\begin{aligned}
    D_{k+1}
    &=
    U_{k+1}^\top U_{k+1}
    -
    V_{k+1}V_{k+1}^\top \\
    &=
    c_k^2
    \left(
        U_k^\top U_k-V_kV_k^\top
    \right)
    +
    \eta_k^2
    \left(
        V_kG_k^\top G_kV_k^\top
        -
        U_k^\top G_kG_k^\top U_k
    \right),
\end{aligned}
\]
which is the stated identity.
\end{proof}

Ignoring the second-order finite-step leakage, the multiplicative decay
over \(K\) updates is
\[
    \prod_{k<K}(1-\eta_k\lambda)^2.
\]
When \(\eta_k\lambda\ll1\), this is approximated by
\[
    \prod_{k<K}(1-\eta_k\lambda)^2
    \approx
    \exp\!\left(
        -2\lambda\sum_{k<K}\eta_k
    \right).
\]
This is the discrete counterpart of the continuous-time clock
\(\mathcal T_{L_2}=\lambda\sum_k\eta_k\).

\paragraph{Finite-step Euclidean updates leak imbalance at second order.}
Without explicit \(L_2\) decay, simultaneous Euclidean gradient updates
do not preserve imbalance exactly.  However, the failure of conservation
starts only at second order in the step size.

\begin{proposition}[Two-factor finite-step leakage is second order]
For the simultaneous update
\[
    U_{k+1}=U_k-\eta_k G_kV_k^\top,
    \qquad
    V_{k+1}=V_k-\eta_k U_k^\top G_k,
\]
one has the exact identity
\[
    D_{k+1}-D_k
    =
    \eta_k^2
    \left(
        V_kG_k^\top G_kV_k^\top
        -
        U_k^\top G_kG_k^\top U_k
    \right).
\]
In particular, all first-order terms in \(\eta_k\) cancel exactly.
\end{proposition}

\begin{proof}
This is the previous proposition with \(\lambda=0\), equivalently
\(c_k=1\).  Expanding directly,
\[
\begin{aligned}
    U_{k+1}^\top U_{k+1}
    &=
    U_k^\top U_k
    -
    \eta_k
    \left(
        V_kG_k^\top U_k
        +
        U_k^\top G_kV_k^\top
    \right)
    +
    \eta_k^2V_kG_k^\top G_kV_k^\top ,
\end{aligned}
\]
and
\[
\begin{aligned}
    V_{k+1}V_{k+1}^\top
    &=
    V_kV_k^\top
    -
    \eta_k
    \left(
        U_k^\top G_kV_k^\top
        +
        V_kG_k^\top U_k
    \right)
    +
    \eta_k^2U_k^\top G_kG_k^\top U_k .
\end{aligned}
\]
The first-order terms are identical.  Subtracting leaves exactly the
second-order expression.
\end{proof}

This calculation proves an order statement, not a monotonicity statement:
finite steps can move the conserved quantity at order \(\eta_k^2\), but
the sign of the movement depends on the current factors and gradient.

\paragraph{Deep linear gradient flow preserves all adjacent imbalances.}
The same conservation law holds at every hidden layer of a deep linear
network.  Let
\[
    F(W_1,\ldots,W_L)=W_LW_{L-1}\cdots W_1,
    \qquad
    \Loss(W_1,\ldots,W_L)=\phi(F),
\]
where
\[
    W_j\in\R^{d_j\times d_{j-1}},
    \qquad
    j=1,\ldots,L.
\]
For \(j=1,\ldots,L-1\), define the adjacent-layer imbalance
\[
    D_j
    =
    W_jW_j^\top
    -
    W_{j+1}^\top W_{j+1}
    \in\R^{d_j\times d_j}.
\]

\begin{proposition}[Deep linear gradient flow preserves adjacent imbalances]
Assume \(\phi\) is differentiable.  Along Euclidean gradient flow
\[
    \dot W_j=-\nabla_{W_j}\Loss,
    \qquad
    j=1,\ldots,L,
\]
one has
\[
    \frac{d}{dt}D_j(t)=0,
    \qquad
    j=1,\ldots,L-1.
\]
With coupled continuous-time weight decay,
\[
    \dot W_j=-\nabla_{W_j}\Loss-\lambda W_j,
\]
the imbalances satisfy
\[
    \frac{d}{dt}D_j(t)=-2\lambda D_j(t),
    \qquad
    D_j(t)=e^{-2\lambda t}D_j(0).
\]
\end{proposition}

\begin{proof}
For each layer \(j\), define
\[
    A_j=W_LW_{L-1}\cdots W_{j+1},
    \qquad
    B_j=W_{j-1}\cdots W_1,
\]
with the convention that empty products are identity matrices.  Then
\[
    F=A_jW_jB_j.
\]
Let
\[
    G=\nabla_F\phi(F).
\]
By the chain rule,
\[
    \nabla_{W_j}\Loss
    =
    A_j^\top G B_j^\top.
\]
We use the identities
\[
    A_j=A_{j+1}W_{j+1},
    \qquad
    B_{j+1}=W_jB_j.
\]
Let
\[
    H_j=\nabla_{W_j}\Loss=A_j^\top G B_j^\top.
\]
Then
\[
\begin{aligned}
    H_jW_j^\top
    &=
    A_j^\top G B_j^\top W_j^\top  \\
    &=
    W_{j+1}^\top A_{j+1}^\top G B_{j+1}^\top  \\
    &=
    W_{j+1}^\top H_{j+1},
\end{aligned}
\]
and
\[
\begin{aligned}
    W_jH_j^\top
    &=
    W_jB_jG^\top A_j  \\
    &=
    B_{j+1}G^\top A_{j+1}W_{j+1}  \\
    &=
    H_{j+1}^\top W_{j+1}.
\end{aligned}
\]
Therefore, under gradient flow,
\[
\begin{aligned}
    \frac{d}{dt}(W_jW_j^\top)
    &=
    -H_jW_j^\top-W_jH_j^\top  \\
    &=
    -W_{j+1}^\top H_{j+1}
    -
    H_{j+1}^\top W_{j+1}.
\end{aligned}
\]
On the other hand,
\[
\begin{aligned}
    \frac{d}{dt}(W_{j+1}^\top W_{j+1})
    &=
    -H_{j+1}^\top W_{j+1}
    -
    W_{j+1}^\top H_{j+1}.
\end{aligned}
\]
The two derivatives are identical, so
\[
    \frac{d}{dt}
    \left(
        W_jW_j^\top
        -
        W_{j+1}^\top W_{j+1}
    \right)
    =
    0.
\]
This proves conservation.

With weight decay, the gradient-dependent terms are unchanged and still
cancel.  The additional terms are
\[
    -2\lambda W_jW_j^\top
    \qquad\text{and}\qquad
    -2\lambda W_{j+1}^\top W_{j+1},
\]
so
\[
    \dot D_j
    =
    -2\lambda
    \left(
        W_jW_j^\top
        -
        W_{j+1}^\top W_{j+1}
    \right)
    =
    -2\lambda D_j.
\]
Solving gives \(D_j(t)=e^{-2\lambda t}D_j(0)\).
\end{proof}

\paragraph{Finite-step leakage in deep linear networks.}
The deep-linear finite-step calculation is the discrete analogue of the
previous conservation law.  The key point is that the layerwise gradient
estimates must be induced by the same product-space gradient estimate.

For a product-space gradient estimate \(G_k\), define
\[
    A_{j,k}=W_{L,k}\cdots W_{j+1,k},
    \qquad
    B_{j,k}=W_{j-1,k}\cdots W_{1,k},
\]
and
\[
    H_{j,k}
    =
    A_{j,k}^\top G_k B_{j,k}^\top.
\]
Thus \(H_{j,k}\) is the layer-\(j\) gradient estimate induced by \(G_k\).

\begin{proposition}[Deep finite-step leakage is second order]
Consider simultaneous Euclidean updates
\[
    W_{j,k+1}=W_{j,k}-\eta_k H_{j,k},
    \qquad
    j=1,\ldots,L.
\]
Then, for \(j=1,\ldots,L-1\),
\[
    D_{j,k+1}-D_{j,k}
    =
    \eta_k^2
    \left(
        H_{j,k}H_{j,k}^\top
        -
        H_{j+1,k}^\top H_{j+1,k}
    \right).
\]
In particular, all first-order terms in \(\eta_k\) cancel exactly.
\end{proposition}

\begin{proof}
Expanding the first term in \(D_{j,k+1}\),
\[
\begin{aligned}
    W_{j,k+1}W_{j,k+1}^\top
    &=
    W_{j,k}W_{j,k}^\top
    -
    \eta_k
    \left(
        H_{j,k}W_{j,k}^\top
        +
        W_{j,k}H_{j,k}^\top
    \right)  \\
    &\qquad
    +
    \eta_k^2H_{j,k}H_{j,k}^\top .
\end{aligned}
\]
Expanding the second term,
\[
\begin{aligned}
    W_{j+1,k+1}^\top W_{j+1,k+1}
    &=
    W_{j+1,k}^\top W_{j+1,k} \\
    &\quad
    -
    \eta_k
    \left(
        H_{j+1,k}^\top W_{j+1,k}
        +
        W_{j+1,k}^\top H_{j+1,k}
    \right)  \\
    &\quad
    +
    \eta_k^2H_{j+1,k}^\top H_{j+1,k}.
\end{aligned}
\]
It remains to verify that the first-order terms are identical.  Since
\[
    H_{j,k}
    =
    A_{j,k}^\top G_kB_{j,k}^\top,
\]
and
\[
    A_{j,k}=A_{j+1,k}W_{j+1,k},
    \qquad
    B_{j+1,k}=W_{j,k}B_{j,k},
\]
we have
\[
\begin{aligned}
    H_{j,k}W_{j,k}^\top
    &=
    A_{j,k}^\top G_kB_{j,k}^\top W_{j,k}^\top \\
    &=
    W_{j+1,k}^\top A_{j+1,k}^\top G_kB_{j+1,k}^\top \\
    &=
    W_{j+1,k}^\top H_{j+1,k},
\end{aligned}
\]
and
\[
\begin{aligned}
    W_{j,k}H_{j,k}^\top
    &=
    W_{j,k}B_{j,k}G_k^\top A_{j,k} \\
    &=
    B_{j+1,k}G_k^\top A_{j+1,k}W_{j+1,k} \\
    &=
    H_{j+1,k}^\top W_{j+1,k}.
\end{aligned}
\]
Therefore the entire first-order contribution in the expansion of
\(W_{j,k+1}W_{j,k+1}^\top\) is identical to the first-order contribution
in the expansion of \(W_{j+1,k+1}^\top W_{j+1,k+1}\).  Subtracting the
two expansions cancels all first-order terms and leaves exactly
\[
    D_{j,k+1}-D_{j,k}
    =
    \eta_k^2
    \left(
        H_{j,k}H_{j,k}^\top
        -
        H_{j+1,k}^\top H_{j+1,k}
    \right).
\]
\end{proof}

\paragraph{The minibatch-dependent leakage clock.}
The previous proposition is deterministic: it holds for any product-space
gradient estimate \(G_k\).  To isolate the batch-size-dependent part,
write
\[
    G_k=\bar G_k+\xi_k,
    \qquad
    \mathbb E_k[\xi_k]=0,
    \qquad
    \mathbb E_k\|\xi_k\|_F^2=O(1/b).
\]
Define
\[
    \bar H_{j,k}
    =
    A_{j,k}^\top \bar G_k B_{j,k}^\top,
    \qquad
    \Delta_{j,k}
    =
    A_{j,k}^\top \xi_k B_{j,k}^\top.
\]
Then
\[
    H_{j,k}=\bar H_{j,k}+\Delta_{j,k},
    \qquad
    \mathbb E_k[\Delta_{j,k}]=0.
\]
Taking conditional expectation in the finite-step leakage identity gives
\[
\begin{aligned}
    \mathbb E_k[D_{j,k+1}-D_{j,k}]
    &=
    \eta_k^2
    \left(
        \bar H_{j,k}\bar H_{j,k}^\top
        -
        \bar H_{j+1,k}^\top\bar H_{j+1,k}
    \right) \\
    &\quad
    +
    \eta_k^2
    \mathbb E_k
    \left[
        \Delta_{j,k}\Delta_{j,k}^\top
        -
        \Delta_{j+1,k}^\top\Delta_{j+1,k}
    \right].
\end{aligned}
\]
The first line is the deterministic full-batch finite-step leakage.  It
is second order in \(\eta_k\), but it is not batch-size dependent.  The
second line is the minibatch-dependent stochastic correction.

On any portion of the trajectory where the operator norms of the factors
are bounded, the assumption
\(\mathbb E_k\|\xi_k\|_F^2=O(1/b)\) implies
\[
    \mathbb E_k\|\Delta_{j,k}\|_F^2=O(1/b)
\]
for every fixed layer \(j\).  Therefore the batch-dependent part of the
conditional expected leakage is \(O(\eta_k^2/b)\) per update.  Summing
over updates gives the stochastic finite-step clock
\[
    \mathcal T_{\mathrm{SGD}}
    =
    \frac{1}{b}\sum_{k<K}\eta_k^2
    =
    \frac{K\eta^2}{b}
    \quad\text{for constant \(\eta\)}.
\]
This clock measures the cumulative size of the minibatch-dependent
second-order correction.  It is an order statement; it does not assert
that the correction has a fixed sign.

\paragraph{Generic preconditioning creates a first-order imbalance term.}
Adaptive methods apply a preconditioned update rather than the Euclidean
gradient update.  The following calculation shows exactly where the
first-order conservation cancellation can fail.

Let
\[
    Q_{j,k}=P_{j,k}(H_{j,k}),
\]
where \(P_{j,k}\) is a layerwise or coordinatewise preconditioning map,
possibly depending on the current iterate and optimizer state.  Consider
updates
\[
    W_{j,k+1}
    =
    W_{j,k}
    -
    \eta_k Q_{j,k}.
\]

\begin{proposition}[Preconditioning creates a first-order imbalance term]
For the preconditioned update above,
\[
\begin{aligned}
    D_{j,k+1}-D_{j,k}
    &=
    -\eta_k
    \Big[
        Q_{j,k}W_{j,k}^\top
        +
        W_{j,k}Q_{j,k}^\top
        -
        Q_{j+1,k}^\top W_{j+1,k}
        -
        W_{j+1,k}^\top Q_{j+1,k}
    \Big] \\
    &\quad
    +
    \eta_k^2
    \left[
        Q_{j,k}Q_{j,k}^\top
        -
        Q_{j+1,k}^\top Q_{j+1,k}
    \right].
\end{aligned}
\]
For Euclidean gradient descent, \(Q_{j,k}=H_{j,k}\), and the first-order
bracket vanishes by the chain-rule identity proved above.  For a general
adaptive preconditioner, \(Q_{j,k}\) need not satisfy that identity, so
imbalance can change at order \(\eta_k\) per update.
\end{proposition}

\begin{proof}
Expanding the first adjacent Gram matrix,
\[
\begin{aligned}
    W_{j,k+1}W_{j,k+1}^\top
    &=
    W_{j,k}W_{j,k}^\top
    -
    \eta_k
    \left(
        Q_{j,k}W_{j,k}^\top
        +
        W_{j,k}Q_{j,k}^\top
    \right)  \\
    &\qquad
    +
    \eta_k^2Q_{j,k}Q_{j,k}^\top .
\end{aligned}
\]
Expanding the second adjacent Gram matrix,
\[
\begin{aligned}
    W_{j+1,k+1}^\top W_{j+1,k+1}
    &=
    W_{j+1,k}^\top W_{j+1,k}
    -
    \eta_k
    \left(
        Q_{j+1,k}^\top W_{j+1,k}
        +
        W_{j+1,k}^\top Q_{j+1,k}
    \right)  \\
    &\qquad
    +
    \eta_k^2Q_{j+1,k}^\top Q_{j+1,k}.
\end{aligned}
\]
Subtracting the second expansion from the first gives the stated
identity.  If \(Q_{j,k}=H_{j,k}\), the first-order bracket is zero by
the identities
\[
    H_{j,k}W_{j,k}^\top
    =
    W_{j+1,k}^\top H_{j+1,k},
    \qquad
    W_{j,k}H_{j,k}^\top
    =
    H_{j+1,k}^\top W_{j+1,k}
\]
proved in the finite-step Euclidean case.
\end{proof}

This calculation proves that the order of imbalance movement can change
from second order to first order under preconditioning.  It does not
claim that every adaptive optimizer monotonically contracts imbalance;
rather, it identifies the algebraic cancellation that Euclidean gradient
flow and Euclidean finite-step updates enjoy, and shows that generic
preconditioning need not preserve it.

\begin{figure}[]
\centering
\includegraphics[width=\textwidth]{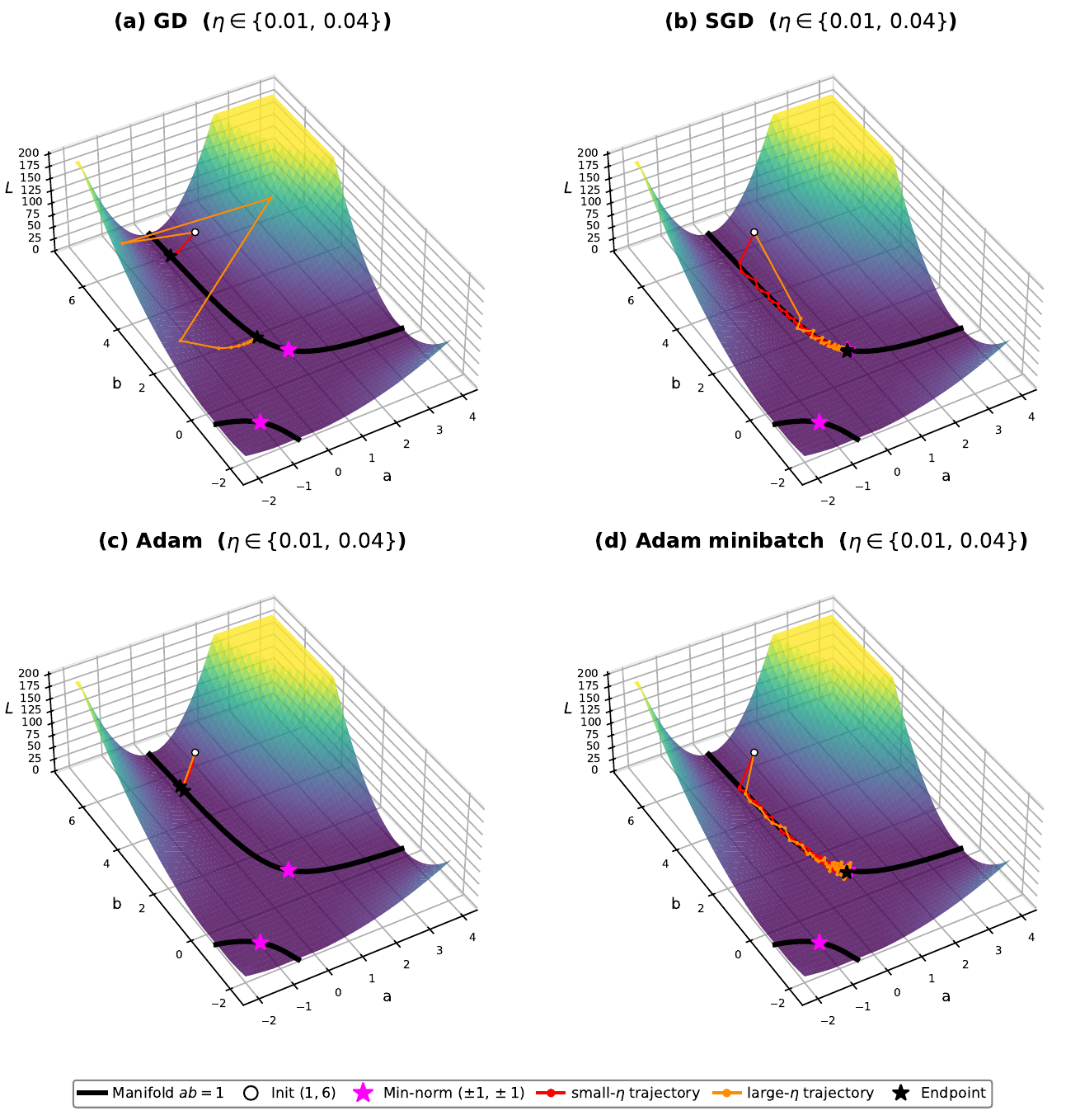}
\caption{\textbf{Optimizer trajectories on the scalar loss
\boldmath$(ab-1)^2$\unboldmath, initialization $(a_0,b_0)=(1,6)$.}
All four panels use the same two learning rates
$\eta\in\{0.01,0.04\}$ (red and orange, respectively); only the
optimizer changes.
The black curve is the manifold of global minima $ab=1$; magenta
stars mark the minimum-norm solutions $(\pm1,\pm1)$.
\textbf{(a)}~Gradient descent (20~steps): the small-$\eta$ trajectory
stays near the initialization ($D\approx-35$, close to gradient-flow
conservation), while the large-$\eta$ trajectory moves toward
lower norm ($D\approx-3$), illustrating the $O(\eta^2K)$ finite-step
leakage.
\textbf{(b)}~SGD with cyclic minibatch targets $\{0,1,2\}$
(up to $50{,}000$~steps): both learning rates converge to the
minimum-norm solution ($D\approx0$), consistent with the
$\mathcal T_{\mathrm{SGD}}=K\eta^2/b$ clock.
\textbf{(c)}~Full-batch Adam ($5{,}000$~steps): the preconditioner
moves $D$ from $-35$ toward $\approx\!-27$, confirming first-order
$O(\eta)$ movement of the imbalance.  However, convergence stalls
because the full-batch gradient $g=2(ab-1)\to0$ at the manifold. This implies that the numerator goes to 0 but the denominator is lowerbounded by $\epsilon>0$.
Thus the $O(\eta)$ coefficient vanishes and $D$ freezes at a nonzero
value.
\textbf{(d)}~Minibatch Adam (up to $200{,}000$~steps): adding
stochasticity restores forgetting---both learning rates reach
$D\approx0$.  Individual minibatch gradients remain nonzero at $ab=1$,
keeping the adaptive clock active.
In real BatchNorm networks, scale invariance prevents the
vanishing-gradient stalling seen in~(c); see the experimental results
in Section~\ref{sec:timescales}.}
\label{fig:toy_landscape}
\end{figure}

\paragraph{Scalar example of first-order breaking.}
The first-order effect is already visible in the scalar two-factor
model. Let \(p=ab\) and \(g=\phi'(ab)\).  The Euclidean gradients are
\[
    h_a=gb,
    \qquad
    h_b=ga.
\]
Consider preconditioned updates
\[
    a_{k+1}=a_k-\eta\alpha_k gb_k,
    \qquad
    b_{k+1}=b_k-\eta\beta_k ga_k,
\]
where \(\alpha_k,\beta_k>0\).  Then
\[
\begin{aligned}
    D_{k+1}-D_k
    &=
    a_{k+1}^2-b_{k+1}^2-(a_k^2-b_k^2) \\
    &=
    -2\eta g a_kb_k(\alpha_k-\beta_k)
    +
    \eta^2g^2
    \left(
        \alpha_k^2 b_k^2-\beta_k^2 a_k^2
    \right).
\end{aligned}
\]
Thus, whenever
\[
    g a_kb_k(\alpha_k-\beta_k)\neq0,
\]
the imbalance changes at first order in \(\eta\).  Euclidean gradient
descent corresponds to \(\alpha_k=\beta_k=1\), in which case the
first-order term vanishes.

This scalar example justifies the adaptive clock
\[
    \mathcal T_{\mathrm{adapt}}
    =
    \sum_{k<K}\eta_k
    =
    K\eta
    \quad\text{for constant \(\eta\)}
\]
as the natural scale on which preconditioned updates can move
initialization-dependent imbalance.

\paragraph{Summary.}
In this minimal homogeneous model, Euclidean gradient flow exactly
conserves initialization-dependent imbalances.  In the scalar two-factor
case, the conserved imbalance directly fixes the final factor norm among
solutions with the same product, giving a precise notion of radial
initialization memory.  Coupled \(L_2\) decay contracts the imbalance on
the clock
\[
    \mathcal T_{L_2}
    =
    \lambda\sum_{k<K}\eta_k.
\]
Euclidean finite-step updates break the conservation law only at second
order; the minibatch-dependent part of the conditional expected leakage
accumulates on the clock
\[
    \mathcal T_{\mathrm{SGD}}
    =
    \frac{1}{b}\sum_{k<K}\eta_k^2.
\]
Generic adaptive preconditioning can break the conservation law at first
order, giving the movement clock
\[
    \mathcal T_{\mathrm{adapt}}
    =
    \sum_{k<K}\eta_k.
\]
These calculations are not a theorem for nonlinear BatchNorm ResNets:
they do not prove monotone forgetting, nor do they determine test
accuracy.  Rather, they identify the conservation law and the
optimizer-dependent orders at which different training mechanisms can
move or contract initialization-dependent quantities.  This mechanism
matches the empirical hierarchy observed in Section~\ref{sec:timescales}:
gradient-flow-like low-LR SGD remembers initialization, increasing the
learning rate or adding explicit \(L_2\) can reduce memory by enlarging
the relevant clocks, and adaptive methods erase initialization-scale
dependence much faster in the diagnostic grid.

\clearpage
\section{Further Related Work}
\label{app:further-related-work}

\paragraph{Purpose of this appendix.}
The main text cites only the papers needed to motivate the central definitions and claims.  This appendix gives the broader historical map.  The question ``why do neural networks not overfit?'' has been approached through several partially overlapping mechanisms: classical capacity control, margins and norms, flatness, PAC-Bayes and compression, interpolation and benign overfitting, function-space priors, spectral and geometric simplicity, stable signal propagation at initialization, and optimizer-dependent implicit bias.  Our paper sits at the intersection of two views.  On one hand, random architectures and initialization schemes define nontrivial priors over functions.  On the other hand, the training pipeline transforms, preserves, or erases those priors.  Initialization memory is our proposed diagnostic for measuring how much of the initial prior survives training.

\paragraph{The common thread.}
The literature can be read as a progression from static explanations to dynamical explanations:
\[
    \text{hypothesis class}
    \quad\to\quad
    \text{trained predictor}
    \quad\to\quad
    \text{training pipeline}
    \quad\to\quad
    \text{memory of initialization}.
\]
Classical theory controls the size of the class; post-hoc complexity measures study the final predictor; implicit-bias work studies the optimizer's selection rule; simplicity-bias work studies the prior at initialization.  Our contribution is to connect the last two: we ask whether the initialization-induced prior remains visible after optimization.

\subsection{From capacity to pipeline-dependent generalization}
\label{app:capacity-pipeline}

\paragraph{Classical complexity control.}
Classical learning theory explains generalization by controlling effective capacity.  VC theory and statistical learning theory formalize uniform convergence in terms of hypothesis-class complexity~\citep{vapnik1971uniform,vapnik1998statistical}.  Later refinements replaced raw parameter counting by data- and norm-dependent quantities, including margins, Rademacher complexities, and weight norms~\citep{bartlett1998sample,bartlett2002rademacher,neyshabur2015norm,bartlett2017spectrally}.  This line is historically essential because it already separates the number of trainable parameters from the actual complexity of the learned predictor.  However, worst-case capacity control alone does not explain why highly overparameterized networks can fit random labels yet generalize well on natural labels.

\paragraph{The random-label challenge.}
The modern crisis point was the observation that standard deep networks can interpolate both natural labels and random labels under essentially the same architecture and optimization machinery~\citep{zhang2017understanding,zhang2021understanding}.  This shifted attention away from the hypothesis class alone and toward the complete pipeline that selects one interpolating solution among many.  Follow-up work showed that deep networks tend to learn simple or structured patterns before memorizing idiosyncratic noise~\citep{arpit2017closer,nakkiran2019sgd}.  The relevant object is therefore not merely the set of functions representable by the architecture, but the algorithmic process by which training selects one of them.

\paragraph{Stability and training time.}
Algorithmic stability gives one route from optimization to generalization.  \citep{hardt2016train} showed that stochastic gradient methods with controlled numbers of steps can be stable, giving a formal sense in which the trajectory itself participates in capacity control.  This perspective is especially relevant for our work because initialization memory is also trajectory-dependent: a fixed architecture may either preserve or erase initialization scale depending on update count, batch size, learning rate, and regularization.

\subsection{Post-hoc explanations: norms, margins, flatness, PAC-Bayes, and compression}
\label{app:posthoc-generalization}

\paragraph{Norm and margin explanations.}
A large body of work attempts to explain generalization by measuring properties of the trained predictor rather than the raw hypothesis class.  Norm-based and margin-based bounds for neural networks include path-norm, spectral-norm, and margin-normalized quantities~\citep{neyshabur2015norm,neyshabur2017exploring,bartlett2017spectrally,neyshabur2018pacbayes,jiang2019predicting}.  Large-scale empirical studies have compared many proposed complexity measures and found that some correlate with generalization better than others, while many fail outside narrow settings~\citep{jiang2020fantastic}.  These works are complementary to ours: they ask how to measure the complexity of the final predictor, whereas we ask how much the final predictor still remembers initialization.

\paragraph{Flatness and sharpness.}
The flat-minima view dates back to \citep{hochreiter1997flat} and became central again in the large-batch generalization debate~\citep{keskar2017large}.  Because sharpness is not invariant to parameter rescaling, later work showed that naive sharpness measures can be misleading in deep networks~\citep{dinh2017sharp}.  Sharpness-aware minimization was subsequently proposed as an explicit algorithmic route to flatter solutions~\citep{foret2021sam}.  In our setting, BatchNorm scale invariance makes this issue particularly relevant: rescaling weights can leave the represented function nearly unchanged while substantially changing the optimization geometry.

\paragraph{PAC-Bayes and compression.}
PAC-Bayes gives another way to turn algorithmic or posterior concentration into generalization statements~\citep{mcallester1999pac,dziugaite2017computing,neyshabur2018pacbayes,zhou2019nonvacuous}.  Compression-based analyses similarly argue that trained networks generalize when they can be compressed without a large loss of performance~\citep{arora2018stronger}.  These approaches support a broad message: overparameterization is not itself fatal if the training procedure selects a small or compressible subset of effective functions.  The initialization-memory viewpoint adds a dynamical question: whether the selected subset is still determined by the initial function prior, or whether training has overwritten it.

\subsection{Interpolation, double descent, and benign overfitting}
\label{app:interpolation}

\paragraph{Interpolation is not necessarily overfitting.}
The double-descent literature recast interpolation as a regime to be understood rather than automatically avoided.  \citep{belkin2019reconciling} proposed that the classical bias--variance curve extends into a second descent beyond the interpolation threshold, and \citep{nakkiran2020deep} demonstrated double descent as a function of model size, data size, and training time in modern deep learning systems.  Benign-overfitting theory shows that, even in simple linear models, interpolating predictors can generalize under appropriate spectral conditions on the data distribution~\citep{bartlett2020benign}.  Our experiments refine this picture for deep networks: interpolation of the training set is not the same as forgetting initialization.  Low-learning-rate SGD can interpolate while retaining large initialization memory.

\subsection{Function-space priors and simplicity bias at initialization}
\label{app:function-priors}

\paragraph{Bayesian and infinite-width priors.}
The idea that random neural networks define a nontrivial function predates the current simplicity-bias literature.  \citep{neal1996bayesian} studied Bayesian neural networks and the connection between random networks and Gaussian processes; \citep{williams1996computing} analyzed computation with infinite neural networks; later work made Gaussian-process and neural-tangent-kernel limits central tools for analyzing wide networks~\citep{lee2018deep,jacot2018neural}.  These function-space limits show that architecture and initialization define a prior over functions before training.

\paragraph{Algorithmic simplicity bias of parameter--function maps.}
A more recent line argues that many parameter--function maps are intrinsically biased toward simple outputs.  \citep{dingle2018input} showed that broad classes of input--output maps can be strongly biased toward low-complexity outputs.  \citep{valle2019deep} applied this idea to neural networks, arguing that the parameter--function map of deep networks is biased toward simple functions and that this can yield PAC-Bayesian explanations of generalization.  \citep{depalma2019random} proved simplicity-bias results for random wide ReLU networks on Boolean inputs.  Related work studies a priori biases toward low-entropy Boolean functions~\citep{mingard2020neural}, and \citep{mingard2025occam} sharpened the claim that deep networks have an inbuilt Occam-like bias.  These works establish initialization and architecture as genuine sources of inductive bias.  They do not, however, determine how much of that bias survives optimization.  That survival question is the focus of our paper.

\paragraph{Algorithmic complexity and Lempel--Ziv measures.}
Several simplicity-bias papers operationalize simplicity through computable proxies for Kolmogorov complexity.  The relevant background includes Lempel--Ziv complexity and universal compression~\citep{lempel1976complexity,ziv1977universal}, as well as the broader theory of Kolmogorov complexity~\citep{li2008kolmogorov}.  In the neural-network setting, these measures are often used not because they are the only possible definition of simplicity, but because they provide a concrete way to compare the output complexity of functions induced by random parameter draws.  Our work is agnostic about which simplicity metric is definitive: the question we isolate is whether any initialization-induced simplicity bias remains visible after training.

\paragraph{Transformer simplicity bias.}
Simplicity bias is not restricted to convolutional or fully connected networks.  Work on Transformers has studied biases toward sparse or low-sensitivity Boolean functions~\citep{bhattamishra2023simplicity,hahn2024sensitive,vasudeva2025transformers}.  Recent work further argues that Transformers already contain structural inductive biases at random initialization~\citep{li2026transformersborn}.  These results are relevant to the broader question of whether large modern architectures prefer simple rules, but they do not by themselves determine whether training preserves the initialization-induced prior.  Initialization memory is intended to measure this survival question directly.

\subsection{Technical notions of simplicity and their training dynamics}
\label{app:notions-of-simplicity}

\paragraph{Frequency, sensitivity, regions, and geometry.}
The literature uses several inequivalent notions of simplicity.  Spectral-bias or frequency-principle results show that neural networks often learn low-frequency components before high-frequency components~\citep{rahaman2019spectral}.  Sensitivity-based work measures input-output Jacobians and local robustness around the data manifold~\citep{novak2018sensitivity}.  Region-counting work studies the number of linear regions or activation patterns in ReLU networks, showing that typical trained or initialized networks use far fewer regions than worst-case expressivity bounds allow~\citep{hanin2019complexity,hanin2019activation}.  Geometric-complexity work measures variation of the learned function through a Dirichlet-energy-like quantity and shows that many regularizers control this geometry~\citep{dherin2022geometric}.  These notions are not identical, but they share a common theme: practical deep networks tend to realize functions much simpler than the worst-case architecture permits.

\paragraph{Simplicity along training.}
Recent work has begun to study not only the simplicity of random networks or final predictors, but also the order in which different complexities are learned.  Neural networks trained with SGD can learn distributions of increasing complexity over time~\citep{refinetti2023neural}.  Two-layer ReLU models exhibit simplicity bias and optimization thresholds that can be analyzed in controlled settings~\citep{boursier2024simplicity}.  Saddle-to-saddle dynamics have also been proposed as a mechanism for simplicity bias across architectures~\citep{zhang2026saddle}.  These works are close in spirit to our paper because they treat simplicity as a dynamical phenomenon rather than a static property of the architecture alone.

\paragraph{Support learning as implicit regularization.}
A particularly relevant recent result is \citep{beneventano2024support}, who study how neural networks learn the support of the target function.  They show that mini-batch SGD can shrink irrelevant input weights in the first layer, while full-batch gradient descent requires explicit regularization to obtain the same effect.  Their mechanism is a second-order implicit regularization effect of SGD that depends on step size and batch size.  This is directly aligned with our perspective: stochastic finite-step effects can erase or restructure parts of the initialization-dependent geometry, whereas gradient-flow-like dynamics may preserve them.

\subsection{Initialization as a dynamical boundary condition}
\label{app:initialization-boundary}

\paragraph{Basins and classical dynamical systems.}
The informal intuition that initialization chooses a basin of attraction comes from classical dynamical systems: the long-time behavior of an autonomous flow can depend strongly on its initial condition and on the basins of attraction of stable invariant sets~\citep{strogatz2015nonlinear,hirsch2013differential}.  In neural-network training, this language is only an analogy: high-dimensional nonconvex losses, stochastic updates, normalization layers, and time-varying learning rates make the dynamics much richer than a simple gradient flow.  Nevertheless, the basin viewpoint motivates the empirical question we study: whether the final trained predictor remains measurably dependent on its initial condition.

\paragraph{Initialization in implicit-bias theory.}
Several theoretical works show that initialization can affect convergence and implicit bias, even in simplified linear or homogeneous models.  \citep{min2021explicit} analyzes the explicit role of initialization in overparameterized linear networks.  \citep{gruber2024role} studies how initialization affects the implicit bias of deep linear networks.  These works complement our minimal linear-network calculation: they show that initialization is not merely a nuisance parameter, but can control the geometry of the solution selected by gradient-based training.

\subsection{Initialization for trainability and signal propagation}
\label{app:init-trainability}

\paragraph{Variance propagation and dynamical isometry.}
A historically distinct literature treats initialization primarily as a trainability device.  Xavier and Kaiming initializations were designed to stabilize forward and backward signal propagation through deep networks~\citep{glorot2010understanding,he2015delving}.  Deep-linear-network analyses showed how initialization affects learning dynamics and training plateaus~\citep{saxe2014exact}.  Mean-field and random-matrix analyses of signal propagation identified ordered and chaotic regimes in random deep networks~\citep{poole2016exponential,schoenholz2017deep}, while dynamical-isometry results emphasized controlling the singular-value distribution of the input-output Jacobian at initialization~\citep{pennington2017resurrecting,chen2018dynamical,xiao2018dynamical}.  Hanin and collaborators studied how architecture and initialization determine trainability, including the onset of exploding and vanishing gradients~\citep{hanin2018start,hanin2018gradients}.  This line explains why initialization can enable training without necessarily claiming that the final predictor should remain close to the initial prior.

\paragraph{Beyond signal propagation.}
Signal propagation is not the only relevant initialization criterion.  \citep{blumenfeld2020beyond} argue that feature diversity at initialization can matter beyond the stability of forward and backward signals.  This is conceptually close to our motivation: initialization can affect not only whether training is numerically stable, but also what information and geometry are available to the optimizer early in training.

\paragraph{Dynamical mean-field and kernel-evolution perspectives.}
Mean-field and dynamical field theories provide another route to analyzing training from random initialization.  Work on kernel evolution in wide neural network models, how predictors move away from their initial random-kernel description during feature learning~\citep{bordelon2023self,bordelon2024finite}.  Recent analyses of deep linear networks from random initialization connect data, width, depth, and hyperparameter transfer~\citep{bordelon2025deep}, while feature-learning infinite limits yield adaptive kernel predictors~\citep{lauditi2025adaptive}.  These works are part of the same broader shift from static initialization priors to the dynamics by which training modifies those priors.

\paragraph{Parameterization and scaling.}
Modern scaling theory further emphasizes parameterization as part of the training pipeline.  Tensor Programs and maximal-update parameterization show that stable feature learning and hyperparameter transfer require carefully chosen width scalings~\citep{yang2020tensor,yang2021feature,yang2021tensorv}.  In this view, initialization is tightly coupled to learning-rate scales and update magnitudes.  Our experiments echo this philosophy in a finite-network setting: initialization scale, learning rate, batch size, and regularization jointly determine whether the initial condition is remembered.

\subsection{Normalization, scale invariance, and effective learning rates}
\label{app:scale-invariance}

\paragraph{BatchNorm and scale-invariant optimization.}
Batch normalization was introduced as a way to accelerate training and reduce sensitivity to initialization~\citep{ioffe2015batch}.  Subsequent theory emphasized that normalized networks contain scale-invariant parameter blocks, for which rescaling weights can leave the represented function nearly unchanged while altering the effective optimization dynamics~\citep{arora2019auto,li2020intrinsic}.  The intrinsic-learning-rate view is particularly close to our mechanism: for scale-invariant parameters, the effective angular step size scales inversely with the squared norm.  This makes initialization scale a useful tracer in BatchNorm networks, because it is partly hidden from the forward map while remaining visible to the optimizer.

\subsection{Backward error analysis and modified equations}
\label{app:bea-modified-equations}

\paragraph{Historical context.}
Backward error analysis is a classical tool in numerical analysis for understanding what a numerical algorithm actually solves.  Its linear-algebra form is closely associated with Wilkinson's work on roundoff error and eigenvalue computations: rather than measuring only the forward error with respect to the original problem, one asks whether the computed output is the exact solution of a nearby perturbed problem~\citep{wilkinson1963rounding,wilkinson1965algebraic,higham2002accuracy}.  For time-stepping methods, the analogous question is whether a discrete integrator is the exact time-$h$ map of a nearby differential equation.  This is the method of modified equations.  Foundational analyses by \citep{griffiths1986scope} and subsequent developments in geometric numerical integration made this viewpoint central for understanding the qualitative behavior of ODE solvers~\citep{calvo1994modified,hairer2006geometric}.  In Hamiltonian and structure-preserving integration, backward error analysis explains why finite-step methods may preserve modified invariants or modified Hamiltonians for long times, even when they do not exactly solve the original continuous system.

\paragraph{Stochastic modified equations.}
A stochastic analogue was developed for numerical SDEs and weak approximation.  Modified-equation methods for SDEs show that a discrete stochastic scheme can be interpreted, in a weak sense, as solving a nearby stochastic dynamics whose drift, diffusion, or invariant measure has been perturbed by the step size~\citep{shardlow2006modified,zygalakis2011existence,debussche2012weak}.  This line is important for modern optimization because constant-step stochastic algorithms are not merely noisy versions of gradient flow; their invariant behavior and local stability can depend on finite-step corrections.  Uniform-in-time weak-error analyses for SGD make this connection explicit by using tools motivated by backward error analysis for stochastic differential equations~\citep{feng2020uniform}.

\paragraph{BEA in optimization and machine learning.}
The ML use of modified-equation ideas began by treating stochastic gradient algorithms as discrete dynamical systems whose continuous-time approximations should include step-size-dependent correction terms.  \citep{li2017stochastic} introduced stochastic modified equations for adaptive stochastic gradient algorithms, and \citep{li2019stochastic} developed the mathematical foundations for SGD, momentum SGD, and stochastic Nesterov methods.  In deep-learning generalization, \citep{barrett2021implicit} used backward error analysis to show that finite learning rates in gradient descent induce an implicit gradient regularization term, while \citep{smith2021origin,beneventano2023trajectories} extended this perspective to SGD and derived minibatch-dependent finite-step regularization.  Subsequent work used related discretization-error or modified-dynamics viewpoints to study deep-network gradient descent~\citep{miyagawa2022toward}, momentum \citep{rosca2023continuous,ghosh2023implicit,cattaneo2025modified}, stochastic coordinate descent~\citep{digiovacchino2024backward}, and the implicit bias of Adam and RMSProp~\citep{cattaneo2024implicit,cattaneo2026memory,cattaneo2026effect}.  These works support the conceptual move made in our paper: raw epoch count is not the natural dynamical unit.  The effects that erase initialization memory are governed by accumulated step-size, stochasticity, preconditioning, and regularization timescales, such as
\[
    \frac{1}{b}\sum_k \eta_k^2,
    \qquad
    \lambda\sum_k \eta_k,
    \qquad
    \sum_k \eta_k .
\]
In this sense, backward error analysis provides the mathematical language for our ``forgetting-time'' interpretation: finite-step optimization follows a nearby dynamics, and the perturbation terms of that nearby dynamics can act as implicit regularizers that overwrite parts of the initialization-induced geometry.

\subsection{Optimization as solution selection and forgetting}
\label{app:optimization-selection}

\paragraph{Implicit bias of gradient descent.}
A large body of literature shows that optimization selects among interpolating solutions.  In separable linear classification, gradient descent on exponential-type losses converges in direction to the max-margin classifier~\citep{soudry2018implicit}.  For linear convolutional networks, the implicit bias depends on architecture and depth~\citep{gunasekar2018implicit}.  Deep matrix factorization exhibits an implicit tendency toward low-rank solutions~\citep{arora2019implicit}.  For homogeneous neural networks, gradient descent can be related to margin maximization~\citep{lyu2020gradient}.  These works support the view that optimization does not merely minimize loss; it chooses a particular geometry among many interpolating predictors.

\paragraph{SGD noise, batch size, and finite-step effects.}
The stochasticity and discretization of SGD also affect the selected solution.  Large-batch training was linked to sharp minima and generalization gaps~\citep{keskar2017large}, while other work argued that the number of updates and high-learning-rate phase can be central to closing this gap~\citep{hoffer2017train}.  Constant-step SGD can be viewed as a stochastic process with an approximate stationary distribution~\citep{mandt2017sgd}.  As reviewed in Appendix~\ref{app:bea-modified-equations}, backward error analysis gives a complementary view: finite-step gradient methods approximately follow the gradient flow of a modified objective~\citep{barrett2021implicit}, and the SGD correction can be minibatch-dependent~\citep{smith2021origin}.  These works motivate our timescale language: raw epoch count is not the natural unit of forgetting; update count, learning rate, batch size, and explicit decay determine how quickly training leaves the initialization-controlled regime.

\paragraph{Adaptive methods.}
Adaptive optimizers use a geometry different from Euclidean gradient descent.  \citep{wilson2017marginal} showed that adaptive methods can select different solutions from SGD and can generalize differently even when they optimize training loss well.  Our results are consistent with this broader message: Adam-family methods do not merely train faster in our grid; they erase initialization-scale dependence more readily.  We interpret this as a forgetting property of the optimizer geometry, not as a universal claim that Adam is always preferable to SGD.

\subsection{Random seeds and initialization effects in language models}
\label{app:llm-seeds}

\paragraph{Fine-tuning and pretraining variability.}
Recent NLP work makes seed dependence concrete in large pretrained models.  \citep{dodge2020finetuning} showed that fine-tuning BERT can vary substantially across random seeds, with both weight initialization and data order contributing to downstream variance.  For decoder-only pretraining, \citep{vanderwal2025polypythias} introduced multiple Pythia-style pretraining runs across seeds and model sizes, finding broadly stable dynamics but also identifiable outlier runs.  \citep{fehlauer2025convergence} studied convergence and divergence of language models across random seeds through token-level distributional comparisons.  \citep{tong2026seedprints} go further, showing that models can retain fingerprints of their training seed.  These papers motivate the same broad question as ours in a different setting: which parts of final performance are due to the initial condition, and which are erased by training?

\paragraph{Controlled language-model training and architecture design.}
The Physics of Language Models series studies controlled mechanisms in language-model training, including knowledge storage and extraction~\citep{allenzhu2024physics31} and architecture design choices in transformer training~\citep{allenzhu2025physics41}.  These works are not direct substitutes for an initialization-scale sweep, but they reinforce the broader point that large-scale model behavior is shaped by a coupled training pipeline rather than by architecture alone.

\paragraph{Why our setting is deliberately smaller.}
Our paper studies CIFAR-10 BatchNorm ResNets rather than large language models because the goal is to sweep initialization scale, optimizer, batch size, learning rate, regularization, depth, and seeds extensively.  The resulting controlled study is not meant to replace LLM-scale evidence.  Instead, it isolates the mechanism: initialization can be a function prior, a trainability device, and a boundary condition for optimization; the observed performance effect depends on whether the training pipeline remembers that boundary condition.

\subsection{Position of this work}
\label{app:related-position}

\paragraph{The missing dynamical link.}
Prior work has studied architecture-induced priors, simplicity bias, trainability at initialization, optimizer implicit bias, scale-invariant optimization, and large-batch effects.  We connect these themes through one measurable quantity,
\[
    \mathsf{Mem}_{\mathrm{acc}}(\mathcal R,K),
\]
The dependence of the returned predictor on the initialization scale under a fixed training procedure $\mathcal R$.  The resulting statement is neither that initialization always matters nor that it disappears universally.  Initialization matters when the training dynamics remember it.

\paragraph{Summary.}
The literature above identifies many reasons why neural networks may avoid overfitting: classical capacity control, margins and norms, flat minima, compression, Bayesian priors, benign overfitting, function-space simplicity, spectral/geometric bias, stable signal propagation, and optimizer-induced implicit bias.  Our contribution is not to replace these explanations.  It is to add a dynamic measurement that connects them:
\[
    \text{initialization prior}
    \quad\xrightarrow{\text{training pipeline}}\quad
    \text{trained predictor}.
\]
Initialization memory measures how much of the left-hand side survives the arrow.  In this sense, initialization is neither universally decisive nor universally irrelevant.  It matters exactly when the training dynamics remember it.

\end{document}